\documentclass[
11pt, 
english, 
singlespacing, 
headsepline, 
]{MastersDoctoralThesis} 

\usepackage[utf8]{inputenc} 
\usepackage[T1]{fontenc} 

\usepackage{palatino} 

\usepackage[backend=bibtex,style=authoryear,natbib=true]{biblatex} 

\usepackage[autostyle=true]{csquotes} 

\usepackage{graphicx}
\usepackage{times}
\usepackage{amsmath}
\usepackage{amssymb}
\usepackage{breqn}
\usepackage{algorithmic}
\usepackage{float}
\usepackage{algorithm}
\usepackage{booktabs}

\newtheorem{theorem}{Theorem}
\newtheorem{definition}[theorem]{Definition}
\newtheorem{lemma}[theorem]{Lemma}

\def\qed{$\blacksquare$}

\newenvironment{proof}{\noindent\textit{Proof:}}{ \qed\smallskip}

\thesistitle{Causal Discovery and Prediction: \\ Methods and Algorithms}

\supervisor{Dr. Marta \textsc{Arias} \\ Dr. Ricard \textsc{Gavaldà}} 
\examiner{} 
\degree{Doctor of Philosophy} 
\author{Gilles \textsc{Blondel}} 
\addresses{} 

\subject{Artificial Intelligence} 
\university{\href{http://www.university.com}{Universitat Politècnica de Catalunya}} 
\department{\href{http://department.university.com}{Department of Computer Science}} 
\group{\href{http://researchgroup.university.com}{Ph.D. program in Artificial Intelligence}} 
\faculty{\href{http://faculty.university.com}{}} 

\AtBeginDocument{
\hypersetup{pdftitle=\ttitle} 
\hypersetup{pdfauthor=\authorname} 
\hypersetup{pdfkeywords=\keywordnames} 
}

\begin{document}

\frontmatter 

\pagestyle{plain} 


\begin{titlepage}
\begin{center}

\vspace*{.06\textheight}
\vspace{1.5cm} 
\textsc{\Large Doctoral Thesis}\\[0.5cm] 

\HRule \\[0.4cm] 
{\huge \bfseries \ttitle\par}\vspace{0.4cm} 
\HRule \\[1.5cm] 
 
\begin{minipage}[t]{0.4\textwidth}
\begin{flushleft} \large
\emph{Author:}\\
\href{http://www.johnsmith.com}{\authorname} 
\end{flushleft}
\end{minipage}
\begin{minipage}[t]{0.4\textwidth}
\begin{flushright} \large
\emph{Supervisor: } \\
\href{http://www.jamessmith.com}{\supname} 
\end{flushright}
\end{minipage}\\[3cm]
 
\vfill

\large \textit{A thesis submitted in fulfillment of the requirements\\ for the degree of \degreename}\\[0.3cm] 
\textit{in the}\\[0.4cm]
\groupname\\\deptname\\\univname\\[1cm] 


\vfill

{\large \today}\\[4cm] 

\vfill
\end{center}
\end{titlepage}


\begin{declaration}
\addchaptertocentry{\authorshipname} 
\noindent I, \authorname, declare that this doctoral thesis titled, \enquote{\ttitle} and the work presented in it are my own. I confirm that:

\begin{itemize} 
\item This work was done wholly or mainly while in candidature for a PhD degree at this University.
\item Where any part of this doctoral thesis has previously been submitted for a degree or any other qualification at this University or any other institution, this has been clearly stated.
\item Where I have consulted the published work of others, this is always clearly attributed.
\item Where I have quoted from the work of others, the source is always given.
\item I have acknowledged all main sources of help.
\item While the doctoral thesis supervisors contributed to defining the problem under study, the methods, algorithms, proofs of their correctness, and the experiments are my work. \\
\end{itemize}
 
\noindent Signed:\\
\rule[0.5em]{25em}{0.5pt} 
 
\noindent Date:\\
\rule[0.5em]{25em}{0.5pt} 
\end{declaration}

\cleardoublepage


\vspace*{0.2\textheight}

{\enquote {\itshape We do not know a truth without its cause... And these things, the most universal, are the hardest for men to know; for they are farthest from the senses.}}

\bigbreak

\indent \indent \indent  \indent \indent \indent \indent \indent \indent \indent \indent \indent \indent \indent Aristotle (Metaphysics; Books I - II)

\bigbreak
\bigbreak
\bigbreak

\indent \indent \indent \indent \indent \enquote {\itshape  There is no dark side 
of the moon, really. 
Matter of fact, it's all dark.}
\break

\indent \indent \indent  \indent \indent \indent \indent \indent \indent \indent \indent \indent \indent \indent \indent \indent \indent \indent \indent Pink Floyd


\begin{abstract}
\addchaptertocentry{\abstractname} 
We are not only observers but also actors of reality, or part of it. Our apparent capability to intervene and alter the course of some events in the space and time surrounding us is an essential component of how we build our model of the world. 

In this doctoral thesis we propose a novel method for measuring how efficient an intervention is to discover the causal relations at play. We introduce a generic a-priori assessment of each possible intervention on the subset of reality we are modelling, in order to select the most cost-effective interventions only, and avoid unnecessary systematic experimentation on the real world. Based on this a-priori assessment, we propose an active learning algorithm that identifies the causal relations in any given causal model, using a least cost sequence of interventions. There are several novel aspects introduced by our algorithm. It is, in most case scenarios, able to discard many causal model candidates using relatively inexpensive interventions that only test one value of the intervened variables. Also, the number of interventions performed by the algorithm can be bounded by the number of causal model candidates. Hence, fewer initial candidates (or equivalently, more prior knowledge) lead to fewer interventions for causal discovery.

Causality is intimately related to time, as causes appear to precede their effects. Cyclical causal processes are a very interesting case of causality in relation to time. In this doctoral thesis we introduce a formal analysis of time cyclical causal settings by defining a causal analog to the purely observational Dynamic Bayesian Networks, and provide a sound and complete algorithm for the identification of causal effects in the cyclic setting. We introduce the existence of two types of hidden confounder variables in this framework, which affect in substantially different ways the identification procedures, a distinction with no analog in either Dynamic Bayesian Networks or standard causal graphs.
\end{abstract}


\begin{acknowledgements}
\addchaptertocentry{\acknowledgementname} 
I would like to thank Qiang Yang for his kindness, support and generosity in helping me find a fundamental topic of research when I visited him at the Hong Kong University of Science and Technology. Also, I would like to thank Isabelle Guyon, Ilya Shpitser, Elias Bareinboim and Marek J. Druzdzel for showing their support. Finally, I thank my advisors, Ricard Gavaldà for sharing his extraordinary wisdom, experience and far-reaching vision, and Marta Arias for her continuous support and for always asking interesting and challenging questions.
\end{acknowledgements}


\tableofcontents 

\listoffigures 

\listoftables 














\dedicatory{ \mbox{To my beautiful wife and daughter, as these are the best years of my life.} } 


\mainmatter 

\pagestyle{thesis} 



\chapter{Introduction} 

\label{ch:Chapter1} 


\newcommand{\keyword}[1]{\textbf{#1}}
\newcommand{\tabhead}[1]{\textbf{#1}}
\newcommand{\code}[1]{\texttt{#1}}
\newcommand{\file}[1]{\texttt{\bfseries#1}}
\newcommand{\option}[1]{\texttt{\itshape#1}}


\section{Causality}
It is our perception, from the moment we start observing and interacting with mother nature, that we live in a time evolving environment. We learn that our actions lead to consequences. We extrapolate towards the outer world these basic frameworks from early personal experiences and we quickly develop a simple and mechanistic explanation of the world: causes produce effects in time. This is our intuitive and practical model of reality.

Imagine a simple universe, fully described using a finite number of variables, and where all these variables are constant across any dimensions, since inception and for ever after. An external observer may make a picture of this universe, but would not be able to describe any dependence relations between its variables. This may not be a problem if nothing ever alters the course of the universe, because there is no expectation of change and no need to model the world to predict change. However, if there is any chance that there will be an external influence altering any of the variables, the observer may ask himself what are the relations (if any) between the variables to attempt predicting the consequences. Causality, as other conceptualizations of the world, is intimately related to the existence of an external observer, and the ability of the observer to interact with the world and alter its course. Simply observing the world without acting on it, would not let us reach the conclusion that causal relations exist, at best we could only describe sequences of events in time and their correlations.

In the causal context, an intervention or experiment is defined as externally forcing a variable, pushing it outside of its natural behaviour, and therefore cutting it from the variables it naturally depends from. By observing the consequences of such action, we may learn what are the causal effects of the intervened variable. An intervention not only isolates a variable from its natural observable causes but also from hidden confounders, unobserved variables that causally influence two or more observed variables such that the association they produce between these observed variables may be taken erroneously as causal.

We use causal graphs to encode the causal relations between variables. Both observable and unobservable variables may be represented as vertex in the graph, and the causal relations may be represented as edges. Causal relations are directional in the sense that acting on one variable causes another to change (smoking causes cancer) while the opposite is untrue. We may reduce the chance of contracting cancer by reducing the amount of smoking. This is not to say that there are not cyclic dependencies among variables, however we may observe this cyclic process over time and, if the observation time scale is small enough, the causal relations may be describable again as directional. A simple way to model causal relations is to use directed acyclic graphs (DAGs).


\section{Discovery of Causal Relations}
\label{sec:discovery_static}

Causal graphs encode the causal relations between the variables in a model. In the smoking and cancer example, we may want to observe a population and attempt to build a causal graph from the characteristics of the observed data. Note that in a deterministic world, once we know the cause variables, we would know without a doubt the effects (cancer would be totally predictable). However, this is almost never the case, as we usually know and observe only part of a complex reality. When we observe samples of a limited set of variables, and assume that unobserved variables exist, the model becomes probabilistic. The observed variables may be represented by their joint probability distribution.

In the probabilistic context, one may analyse the joint probability distribution of the observed variables to extract dependence and independence relations among them. In particular the conditional independence relations may provide causal information. As an example, given four variables $X$,$Y$,$Z$,$T$, if the statistical test $P(Y|X)=P(Y)$ holds true this indicates $Y$ and $X$ are independent. Translated to the causal graph this means there is no "d-connected" path, i.e., a sequence of adjacent edges between $X$ and $Y$ through which $X$ and $Y$ have a causal influence on each other. Also, if $P(T|Z,X)=P(T|Z)$ we conclude that $T$ is independent of $X$ given $Z$ so that $Z$ blocks any such paths between $X$ and $T$. We say that $X$ and $T$ are "d-separated" conditional to $Z$. These statistical tests lead to graphical rules that "d-connect" or "d-separate" variables depending on other variables, and a set of DAGs compatible with these rules can be inferred, in other words, a skeleton structure of a completed partially directed acyclic graph (CPDAG) can be inferred from the observed joint probability distributions.

From a general perspective, the discovery of causal relations from empirical data is the basis of the scientific method. In the last few decades, some focus has been given to the development of algorithms for the discovery of causal relations from observed probabilistic data. The attempts to automate causal discovery from purely observed data, aim at avoiding the requirement of experimentation on the domain, which can be a very costly process. Systematic analysis of the observed data, under various assumptions, provides appropriate hints that there might be causal relationships between the variables under analysis in the model.

In this thesis we will assume that the distributions generated by causal processes satisfy the Markov and Faithfulness conditions. 

\begin{definition}[Markov Condition]
\label{def:markov}
Let $G$ be a DAG, $V$ be the set of vertices of $G$, $X\in V$ be a vertex of $G$, $Pa(X)$ be the set of parents of $X$ and $De(X)$ be the set of descendants of $X$ in $G$. The DAG $G$ and a probability distribution $P$ on the vertices $V$ satisfy the Markov condition iff for every $X\in V$, $X$ and $\{V\setminus(X\cup De(X))\}$ are independent conditional on $Pa(X)$.
\end{definition}

\begin{definition}[Faithfulness Condition]
\label{def:faithfulness}
Let $G$ be a DAG, $V$ the set of vertices of $G$ and $P$ a distribution over $V$. Then $P$ is faithful to $G$ iff $G$ and $P$ satisfy the Markov condition and every conditional independence relation true in $P$ is entailed by the Markov condition for $G$.
\end{definition}

These conditions connect the observed probability distribution with the underlying causal system that generates the distribution.

A number of methods have been proposed for finding causal relations from observational data. Several algorithms have  been developed, IC, IC* \parencite{Ref14} and \parencite{Ref15}, PC \parencite{Ref10}, GES \parencite{Ref13}, FCI \parencite{Ref11} and \parencite{Ref10}, RFCI \parencite{Ref12}. Some of these algorithms assume there are no hidden variables, while others assume that hidden variables exist and may confound, i.e., influence causally, several of the observed variables.

The discovery of causal relations from observational data has its limits. In most cases we will obtain, at best, a set of causal models compatible with the observed data. Examples of this are CPDAGs, output of the PC algorithm and Partial Ancestral Graphs (PAG), output of the FCI algorithm. 

A different approach is to, not only find causal relations from the observational data, but also do so by performing experiments (also called "interventions") on the system of variables we are trying to discover the causal relations from, and gathering the joint probability distributions of the system under various interventions. A number of interventions may be required to build a complete DAG. Worst case scenario bounds for the number of interventions have been identified by \parencite{Ref1}. \parencite{Ref2} expand the GES algorithm to causal discovery with experimental data. \parencite{PetersBuhlmannMeinshausen2015} use the predicted invariance of causal models under several interventions as a method for causal discovery, in comparison with non-causal models where the invariance does not hold, and \parencite{HeinzeDemlPetersMeinshausen2018} expand this setting to non-linear models.

\parencite{Ref6} propose a method for learning causal models from a mixture of observational and interventional data.
\parencite{Ref5} and 
\parencite{Ref7} introduce active learning algorithms for the discovery of causal relations, based on interventional data. However, these methods and algorithms have focused on the discovery of causal relations by using interventions in causal models without hidden confounders. This is very restrictive in real world applications, because hidden confounders are typically present in real world data. 

This thesis is concerned with the discovery of causal models that contain hidden confounders, using sequences of interventions. This is a more realistic scenario. Causal models without hidden confounders are a particular sub-case and as such our proposition includes and unifies previous active learning methods.


\section{Prediction of Causal Effects}

In the previous section we have discussed the discovery of causal relations. Using appropriate methods, the objective of causal discovery is to build a causal model that best represents the causal relations that exist in the system we want to describe. Another central topic in causality is the prediction of causal effects. Once we know what causal relations exist among the variables of a system, we may want to reason about the system's behaviour, and in particular we may want to predict the causal effects of an intervention on the system. This is called the causal identification problem.

Pearl’s causal graphical models and do-calculus \parencite{pearl1994probabilistic} are a leading approach to modelling causal relations and predicting the effect of interventions. The do-calculus is an algebraic framework for reasoning
about such interventions: An expression $P(Y|do(X))$ indicates the probability distribution of a set of variables $Y$ upon performing an intervention on another set $X$. In some cases, the effect of such an experiment can be obtained given a
causal graph and some observational distributions; this is convenient as some experiments may be impossible, expensive, or unethical to perform. When $P(Y|do(X))$, for a given causal graph, can be rewritten as an expression containing
only observational probabilities, without a do operator, we say that it is identifiable. \parencite{shpitser2006identification} and \parencite{huang2006identifiability} showed that a do expression is identifiable if and only if it can be rewritten in
this way with a finite number of applications of the three rules
of do-calculus, and \parencite{shpitser2006identification} proposed the ID algorithm which performs this transformation if at all possible, or else returns {\em fail} indicating non-identifiability. More precisely, the do-calculus framework provides non-parametric identifiability, whereas other identification methods for causal effects need to rely on more restrictive assumptions beyond the causal graph, such
as structural equations models, linearity, or other parametric assumptions involving the variables in the model. In this thesis, by {\em identifiability} we always mean {\em non-parametric identifiability}.

The soundness and completeness of do-calculus opens the way for new research. In particular, under some assumptions, given a causal model we can predict the effects of interventions. This thesis introduces a method for learning causal models by comparing the predicted outcomes of interventions with the actual intervention results.


\section{Timescale and Equilibrium}

Causal discovery and causal identification, as described so far in this thesis, can be time agnostic. In the probabilistic context, indeed the time component may be one more unobserved variable, along with many others. However, we clearly perceive that causes produce effects after some lapse of time, even if microscopical, and that time is crucial to understanding causality. 

One of the problems that appear when trying to discover causal effect relations from data over time is choosing an appropriate timescale. Data sets are usually built from data samples, and are not in the form of continuous data over time. Without any prior knowledge about the rate of change under which the system evolves some causal relations may not be discovered, just because the sampling rate is not adjusted to the dynamics of the variables. Also, there may be various causal effects occurring at different timescales within the same system, which adds to the complexity of the problem. This is a generic challenge in science.  

Another challenge of causal discovery across time is that some causal relations may be hidden if we monitor the system after these causal relations have reached some stable equilibrium. We do not know if and when these causal relations have existed and have caused the system to evolve in the past. Modelling the system only with the observed dynamic causal relations will not guarantee we can correctly predict the effect of interventions. Indeed, in reality some interventions may take the system out of equilibrium and awaken the previously unobserved causal relations.   

Existing research discusses some of the challenges of the dynamic temporal environment. Dynamic causal systems are often modelled with sets of differential
equations. However \parencite{dashfundamental} \parencite{dash2001caveats} \parencite{dash2005restructuring} show the caveats of
the discovery of causal models based on differential equations
which pass through equilibrium states, and how causal
reasoning based on the models discovered in such way may
fail. \parencite{voortman2012learning} propose an algorithm for the discovery of causal relations
based on differential equations while ensuring those
caveats due to system equilibrium states are taken into account.
Timescale and sampling rate at which we observe a
dynamic system play a crucial role in how well the obtained
data may represent the causal relations in the system. \parencite{aalen2014can} discuss
the difficulties of representing a dynamic system with
a DAG built from discrete observations and \parencite{gong2015discovering} argue that
under some conditions the discovery of temporal causal relations
is feasible from data sampled at a lower rate than the
system dynamics. \parencite{HYTTINEN2017208} extend the discussion on the impact of choosing an appropriate timescale and sampling rate, and propose a discovery algorithm based on a general-purpose constraint solver.

\subsection{Causal Discovery in Dynamic Systems}

Regarding the discovery of causal models from observational data in dynamic, time dependent systems \parencite{iwasaki1989causality} and \parencite{dash2008note} propose an algorithm to establish an ordering of the variables corresponding to the temporal order of propagation of causal effects. Methods for the discovery of cyclic causal graphs from data have been proposed using independent component analysis \parencite{lacerda2012discovering} and using local d-separation criteria \parencite{meek2014toward}. Existing algorithms for causal discovery from static data have been extended to the dynamic setting  by \parencite{moneta2006graphical} and \parencite{chicharro2015algorithms}. \parencite{dahlhaus2003causality,white2010granger,white2011linking} discuss the discovery of causal graphs from time series by including Granger causality concepts into their causal models. \parencite{amort_causal_disc} introduce a method to infer causal relations
from samples obtained across various underlying causal models with shared structural dynamics.

\subsection{Causal Prediction in Dynamic Systems}

Regarding reasoning from a given dynamic causal model, one existing line of research is based on time series and Granger causality concepts \parencite{eichler2010granger,eichler2012causal,eichler2012causal2}. The authors in \parencite{queen2009intervention} use multivariate time series for identification of causal effects in traffic flow models. \parencite{lauritzen2002chain} discuss interventions in dynamic systems in equilibrium, for several types of discrete-time and continuous-time processes with feedback. \parencite{didelezcausal} uses local independence graphs to represent time-continuous dynamic systems and identify the effect of interventions by re-weighting the causal processes involved.

Existing work on causality does not thoroughly address causal reasoning in dynamic systems using do-calculus. \parencite{eichler2010granger,eichler2012causal,eichler2012causal2} discuss back-door and front-door criteria in time-series. \parencite{DawidDidelez} study a framework of sequential data-gathering and decision-making through a discrete sequence of stages, but do not extend their work to the full power of do-calculus as a complete logic for general causal effect identification. \parencite{Peters2020CausalMF} discuss causal prediction in dynamic systems, using an extension of structural causal models to models governed by differential
equations that include noise.

One of the advantages of do-calculus is its non-parametric approach so that it leaves the type of functional relation between variables undefined. This thesis extends the use of do-calculus to time series while requiring less restrictions than existing parametric causal analysis. Parametric approaches may require to differentiate the intervention impacts depending on the system state, non-equilibrium or equilibrium, while our non-parametric approach is generic across system states. This thesis shows the generic methods and explicit formulas revealed by the application of do-calculus to the dynamic setting. These methods and formulas simplify the identification of time evolving effects and reduce the complexity of causal identification algorithms.


\section{Thesis Contributions}

This thesis introduces novel contributions in two fields: the discovery of causal relations and the prediction of causal effects. 

Regarding causal discovery, this thesis introduces a novel and generic method to learn causal graphs by performing a sequence of interventions, where each intervention is applied on a single value of the intervened variables, and while minimizing the overall cost of the sequence of intervened and observed variables during the discovery process. Regarding causal effect prediction, this thesis introduces a comprehensive causal reasoning method for models recurrent in time. In this thesis, all causal models are assumed to contain hidden confounders that have an influence on observed variables in the causal model, except when explicitly referring to causal models without hidden confounders as a sub-case. Also, all variables are assumed to be in a finite domain.

\subsection{Contributions to the Discovery of Causal Relations}

Our method for the discovery of causal relations introduces several novelties. Firstly, we use interventions on a single value of the intervened variables. To the best of our knowledge, all previous methods require interventions on several values of the intervened variables in order to measure correlation or conditional independence among variables. By using do-calculus as a tool to predict systematically and numerically the effect of all the interventions that are possible, without having to actually perform them, we move the search space out of the real world, and eliminate the need for systematic correlation and independence testing in the real world. We assume that computational cost is not a concern, when compared with the cost of actually experimenting in the real world. 

Secondly, we accept any set of candidate graphs as input to our method. Previous knowledge may or may not be in the form of an equivalence class of graphs, and the set of candidate graphs may or may not have any particular parametric characteristic. Some candidate graphs may have been discarded previously based on analysing the available observational data, however no algorithm based on observational data only can identify the true graph of a causal model. As such, our method accepts any set of causal graph candidates, with the only assumption that the true graph, the solution to the problem, is included in the set of candidates. 

Thirdly, all causal graphs are assumed to contain hidden confounders, whereas most previous work focused on causal graphs without hidden confounders.

In more detail, our contribution to the discovery of causal relations is as follows:
\begin{itemize}

\item We introduce a mechanism to predict the effect of all possible interventions across a set of candidate graphs, under the hypothesis that any of the candidates can be the true graph, given a-priori knowledge. We do this prediction across all possible values of the intervened variables. This allows to systematically assess and compare the potential effect of all interventions a-priori, choose the most appropriate ones in order to discriminate between the candidates, and avoid the need to apply most of the interventions in the real world. 

\item We avoid the need to use systematic correlation and conditional independence tests in the real world, whereas most previous methods use a systematic interventional approach in the real world.

\item We introduce an algorithm for active learning of causal graphs, based on identifying the set of single value interventions that discriminate between all candidate graphs with the minimal cost of intervening and observing variables. As we assess the effects of all possible interventions a-priori, we can apply a surgical approach and identify with precision the required intervened variables, observed variables and value of the intervened variables that provide the most effective discrimination information between the candidate causal graphs, at the lowest cost.

\item We uncover the graphical conditions under which the true causal graph cannot be fully identified with a sequence of single value interventions, in which case conditional independence testing is required to complete the discovery process. This case scenario occurs if problematic graphical structures called {\em hedges} are present in specific parts of the graph. 

\item We prove that, if we start the discovery process with a set of candidate causal graphs with hidden confounders, but without {\em hedges}, we can always learn the true graph with a sequence of least cost, single value interventions.

\end{itemize}

To provide a simple example of our method, let us consider a causal model $M^\star$, which we call the {\em true} model, with induced graph $G^\star$, which we call the {\em true} graph. Let us consider a set of causal graphs that include
$G^\star$, which we call the {\em candidate} graphs. We do not know which {\em candidate} graph is the {\em true} graph $G^\star$. Figure~\ref{fig:ALCAM_example} shows a set of {\em candidate} graphs. Bi-directed edges represent the presence of a hidden confounder which has an effect on the two variables the bi-directed edge is pointing at. To find which one of the {\em candidates} is $G^\star$ we apply selected interventions on $M^\star$ that provide us enough information to discriminate between the {\em candidates}. Let us consider an intervention on variables $X$ and $Y$, and let us observe the effect of that intervention on $Z$. Assume the {\em true} graph is $G_1$. If that is the case, we predict, using do-calculus rules and standard probability manipulations, that the effect of the intervention should be $P(Z|X,Y)$. Using the same logic, if $G_2$ is the {\em true} graph instead of $G_1$, then we predict that the effect should be $\sum_{X}P(Z|X,Y)P(X)$. Figure~\ref{fig:ALCAM_example1} shows the predicted effect on $Z$ from the intervention for each candidate graph. We see in Figure~\ref{fig:ALCAM_example1} that the predicted effect is different for each {\em candidate} graph. Let us now apply this intervention on $M^\star$, i.e., "in the real world", and find the actual effect of the intervention on $Z$. It suffices to apply the intervention on $M^\star$ for one value of the intervened variables for which the predicted effect on $Z$, i.e., the joint probability distribution of $Z$, differs among the {\em candidate} graphs. We can now eliminate the {\em candidate} graphs for which the predicted effect and the actual effect differ, and conclude that $G^\star$ is the graph for which the prediction and the actual effect are the same.

Another intervention may provide different results. Figure~\ref{fig:ALCAM_example2} shows that an intervention on $Y$ and measuring the effect on $Z$ does not provide enough information to discriminate between the four candidates. We see in Figure~\ref{fig:ALCAM_example2} that the predicted effect is the same for $G_1$ and $G_2$ and is the same for $G_3$ and $G_4$. The intervention is not sufficient to identify $G^\star$, and additional interventions will be required.

The example shows that given a set of candidate causal graphs, some interventions have more discriminative power than others. Also, some interventions may be more costly than others, e.g., observing more variables or intervening more variables, and some variables may be more costly to intervene or observe than others. This thesis proposes a generic method for selecting inexpensive interventions with a high power of discrimination among the candidates, and using these interventions to eliminate candidate models that are incompatible with the interventional effects in $M^\star$. We may do this iteratively until a single graph is found to be compatible with all performed interventions, while minimizing the cost of the sequence of interventions.

\begin{figure}[H]
\begin{center}
\begin{tabular}{cc}
    \includegraphics[width=0.2\textwidth]{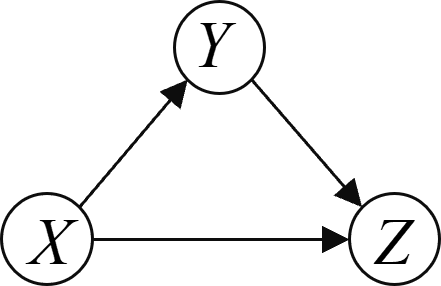} & \includegraphics[width=0.2\textwidth]{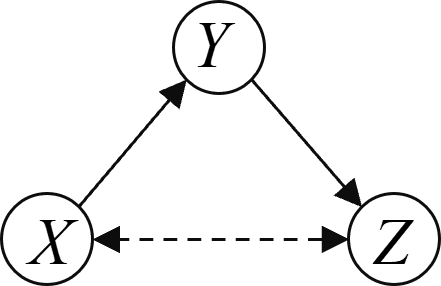}  \\
    
    $G_1$ & $G_2$ \\
    
     &  \\
    \includegraphics[width=0.2\textwidth]{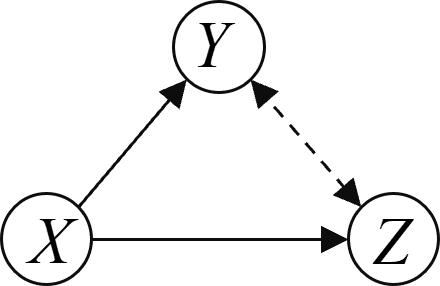} & \includegraphics[width=0.2\textwidth]{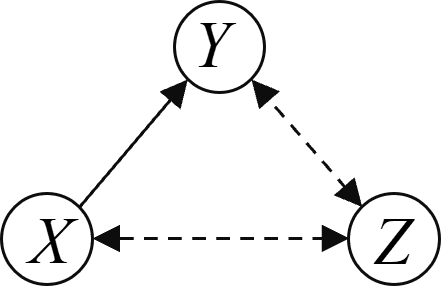}  \\
    $G_3$ & $G_4$ \\

\end{tabular}
\end{center}

\caption{Set of causal graph candidates.}
\label{fig:ALCAM_example}
\end{figure}

\begin{figure}[H]
\begin{center}
\begin{tabular}{cc}
    \includegraphics[width=0.2\textwidth]{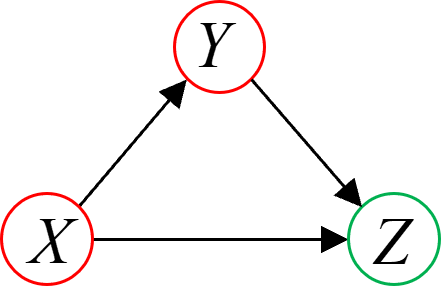} & \includegraphics[width=0.2\textwidth]{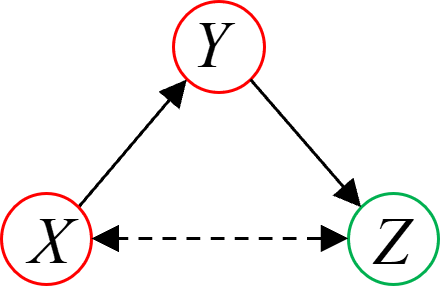}  \\
    
    $P_1=P(Z|X,Y)$ & $P_2=\sum_{X}P(Z|X,Y)P(X)$ \\
    
     &  \\
    \includegraphics[width=0.2\textwidth]{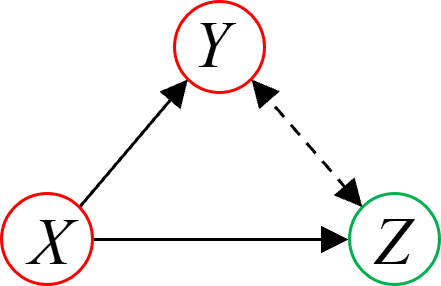} & \includegraphics[width=0.2\textwidth]{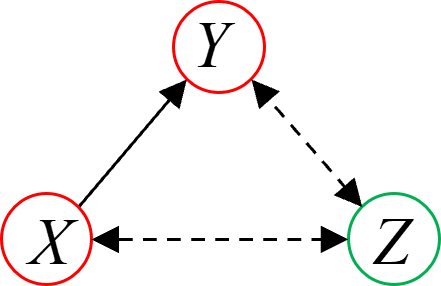}  \\
    $P_3=P(Z|X)$ & $P_4=P(Z)$ \\

\end{tabular}
\end{center}

\caption{Set of candidate graphs for a causal model $M^\star$. Variables $X, Y$ are intervened and variable $Z$ is observed. The intervention $P(Z|do(X,Y))$ can discriminate between the four candidates as follows: assuming $G_1$ is the {\em true} graph $G^\star$, we predict the effect on $Z$ as $P_1$. Assuming $G_2$ is the {\em true} graph $G^\star$, we predict the effect on $Z$ as $P_2$, and we do the same for graphs $G_3$ and $G_4$. We find that the predicted effect for this intervention differs across all candidate graphs. If we apply the intervention on $M^\star$ and find the actual effect on $Z$, we can eliminate the candidate graphs for which the prediction and the actual effect differ. We keep the candidate for which the prediction matches the actual effect of the intervention on $M^\star$, and conclude it is $G^\star$.}
\label{fig:ALCAM_example1}
\end{figure}

\begin{figure}[H]
\begin{center}
\begin{tabular}{cc}
    \includegraphics[width=0.2\textwidth]{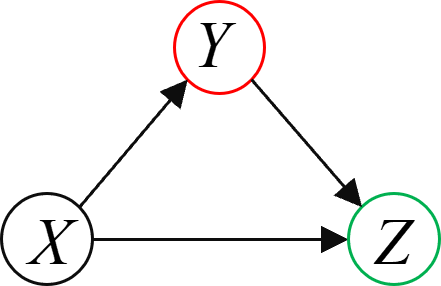} & \includegraphics[width=0.2\textwidth]{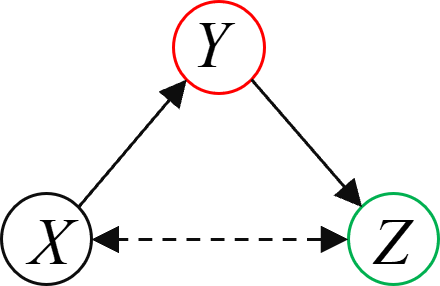}  \\
    
    $P_1=\sum_{X}P(Z|X,Y)P(X)$ & $P_2=\sum_{X}P(Z|X,Y)P(X)$ \\
    
     &  \\
    \includegraphics[width=0.2\textwidth]{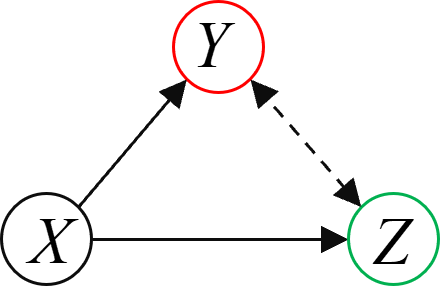} & \includegraphics[width=0.2\textwidth]{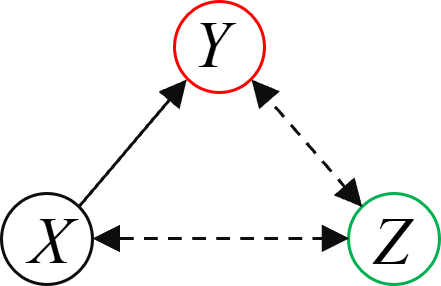}  \\
    $P_3=P(Z)$ & $P_4=P(Z)$ \\

\end{tabular}
\end{center}

\caption{The intervention $P(Z|do(Y))$ cannot discriminate between the four candidates. For instance, the predicted effect on $Z$ for $G_1$ and $G_2$ are the same. Also, the predicted effect is the same for $G_3$ and $G_4$. If we apply the intervention on $M^\star$ and find the actual effect on $Z$, we can eliminate the two candidate graphs for which the prediction and the actual effect differ. We keep, as potentially valid, the two candidate graphs for which the prediction and the actual effect on $M^\star$ are the same.}
\label{fig:ALCAM_example2}
\end{figure}

\subsection{Contributions to the Prediction of Causal Effects}

The second subject of this thesis is causal prediction. The thesis focus is on dynamic causal reasoning: given the formal description of a dynamic causal system and a set of assumptions, we propose methods to evaluate the modified trajectory of the system over
time, after an experiment or intervention. We assume that the
observation timescale is sufficiently small compared to the system dynamics, and that the causal model includes both the non-equilibrium causal relations and those under equilibrium states. We assume that a stable set of causal dependencies exist which generate the system evolution along time. Our proposed algorithms take such models as an input and predict their evolution over time, after an intervention. We also introduce transportability algorithms in the dynamic setting, where causal knowledge in source time-recurrent domains may be used for prediction in target time-recurrent domains.

In more detail, our contribution to the prediction of causal effects is as follows:
\begin{itemize}

\item We introduce a complete reasoning method (Lemmas and Theorems) for the identification of causal effects in causal models recurrent in time, which we call dynamic causal networks. We apply do-calculus to this setting, and show what parts of a bi-infinite causal graph, across time, need to be analysed in order to solve the causal effect identification of the entire graph.

\item We uncover the existence of several scenarios in regards to time sampling and slicing of the bi-infinite temporal graph, and in relation to how hidden confounders affect the dynamic causal network. More specifically, we show how static hidden confounders, affecting variables in the same time step, and dynamic hidden confounders, affecting variables in different time steps, have a very different impact on the complexity of our causal effect identification methods.

\item We introduce several algorithms for the identification of causal effects in dynamic causal networks, when static hidden confounders are present, and when dynamic hidden confounders are present. We identify the causal graph structures that prevent identifiability in the dynamic setting.
 
\item We introduce an algorithm for the transportability of causal effects in dynamic causal networks. That is, we extend the transportability algorithm to the use case where some interventional data is known in a source dynamic domain, and is used to help identify causal effects in a target dynamic domain with the same cyclic causal structure.

\end{itemize}

To provide a simple running example (not as a precise or accurate modelling of reality), let us consider two roads joining the same two cities, where drivers choose every day to use one or the other road. The average travel delay between the two cities on any given day depends on the traffic distribution among the two roads. Drivers choose between the two roads depending on recent experience, in particular how congested a road was last time they used it. Figure~\ref{fig:dynamic_no_confounder_extended} indicates these relations: the weather ($w$) has an effect on traffic conditions on a given day ($tr1$, $tr2$) which affects the travel delay on that same day ($d$). Driver experience has an influence on the road choice next day, therefore impacting $tr1$ and $tr2$. To simplify, we assume that drivers have short memory, being influenced by the conditions on the previous day only. This infinite network can be folded into a finite representation as shown in Figure~\ref{fig:dynamic_no_confounder_compact}, where $+1$ indicates an edge linking two consecutive replicas of the DAG. Additionally, if one assumes the weather to be an unobserved variable, then it becomes a {\em hidden confounder} as it causally affects two observed variables, as shown in Figure~\ref{fig:dcn_confounder_compact_intro}. We call the hidden confounders with causal effect over variables in the same time slice {\em static hidden confounders}, and hidden confounders with causal effect over variables at different time slices {\em dynamic hidden confounders}. Our models allow for causal identification with both types of hidden confounders. 

This setting enables the resolution of causal effect identification problems where causal relations are recurrent over time. These problems are not solvable in the context of classic DBNs, as causal interventions are not defined in such models. For this we use causal networks and do-calculus. However, time dependencies cannot be modelled with static causal networks. As we want to predict the trajectory of the system over time after an intervention, we must use a dynamic causal network. 

Using our example, in order to reduce travel delay traffic controllers could consider actions such as limiting the number of vehicles admitted to one of the two roads. We would like to predict the effect of such action on the travel delay a few days later, e.g., $\Pr(d_{t+\alpha}|do(tr1_t))$.

This thesis solves the causal identification problem (causal prediction) in such settings.

\begin{figure}[H]
\begin{center}
\includegraphics[width=0.60\textwidth]{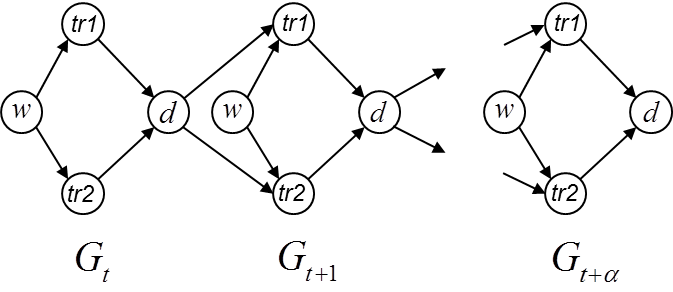}
\end{center}
\caption{A dynamic causal network. The weather $w$ has an effect on traffic flows $tr1$, $tr2$, which in turn have an impact on the average travel delay $d$. Based on the travel delay, car drivers may choose a different road next time, which has a causal effect on the traffic flows.}
\label{fig:dynamic_no_confounder_extended}
\end{figure}

\begin{figure}[H]
\begin{center}
\includegraphics[width=0.18\textwidth]{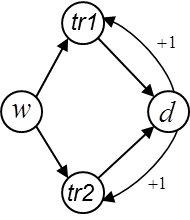}
\end{center}
\caption{Compact representation of the Dynamic Causal Network in Figure~\ref{fig:dynamic_no_confounder_extended} where $+1$ indicates an edge linking a variable in $G_t$ with a variable in $G_{t+1}$.}
\label{fig:dynamic_no_confounder_compact}
\end{figure}

\begin{figure}[H]
\begin{center}
\includegraphics[width=2cm]{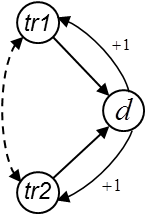}
\end{center}
\caption{Dynamic Causal Network where $tr1$ and $tr2$ have a common unobserved cause, a {\em hidden confounder}. Since both variables are in the same time slice, we call it a {\em static hidden confounder}.}
\label{fig:dcn_confounder_compact_intro}
\end{figure}










\thispagestyle{plain} 
\mbox{}


\chapter{Background and Previous Work} 

\label{ch:Chapter2} 


\section{Causal Models}

The notation used in this thesis is based on causal models and do-calculus \parencite{pearl1994probabilistic,Ref15}.

\begin{definition}[Causal Model]
\label{def:causalnetwork}
A causal model over a set of variables $V$ is a tuple $M=\langle V,U,F,Pa,P(U) \rangle$, where $V=\{ V_1,V_2,...V_n\}$ is a set of n variables that are determined by the model ("endogenous" or "observed" variables), U is a set of random variables that are determined outside the model ("exogenous" or "unobserved" variables) but that can influence the rest of the model, $F$ is a set of n functions such that $V_{i} = f_i(Pa(V_{i}),U_{i})$, $Pa(V_{i})$ is a subset of $V \setminus V_i$ ("observed parents of $V_{i}$"), $U_i$ is a subset of $U$ ("unobserved parents of $V_i$") and $P(U)$ is a joint probability distribution over the variables in $U$. A causal model has an associated graph in which each observed variable $V_i$ corresponds to a vertex, there is one edge pointing to $V_i$ from each of its observed parents ${Pa(V_{i})}$ and there is a doubly-pointed edge between the pairs of vertices influenced by a common unobserved parent in $U$.
\end{definition}

In other words, in a causal model the probability distribution of each variable $V_i$ is assigned by a function $f_i$ which is determined by a subset of $V \setminus V_i$ called the "observed parents" of $V_i$ ($Pa(V_{i}$)) and a subset of $U$ ($U_i$) called the "unobserved parents" of $V_i$. The joint probability distribution of the observed variables in a causal model $M$ is \parencite{TianThesis2002}:

\begin{equation}
\label{eqn:PofV}
P(V)=P(V_1, V_2, ...V_n) = \sum_{U} \prod_{i} P(V_i|Pa(V_i), U_i) \prod_{i} P(U_i)
\end{equation}

The graphical representation of a causal model is also called the "induced graph of the causal model" or "causal graph". It contains vertices  $V_i$, edges from ${Pa(V_{i})}$ to $V_i$ and bidirected edges between the pairs of vertices influenced by a common unobserved variable, that is between $V_i$ and $V_j$ if $U_i \cap U_j \neq \emptyset$ (see Figure~\ref{fig:graph_confounder}). We call the unobserved variables $U$ "hidden confounders".

In this thesis, in chapters~\ref{ch:Chapter3} and \ref{ch:Chapter4} all causal models are assumed to be acyclic. Causal models where there exist $Pa(V_i)$ relations such that the model contains cycles are studied in chapters~\ref{ch:Chapter5} and \ref{ch:Chapter6}, where we introduce a time dependent definition of causal models. 

\begin{figure}
\begin{center}
\includegraphics[width=0.22\textwidth]{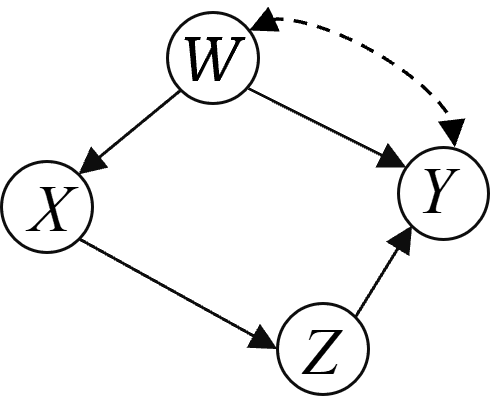}
\end{center}
\caption{Causal graph with vertices representing variables $X,Y,W,Z$, edges representing functions $X = f_1(W); Z = f_2(X); W = f_3(U); Y = f_3(W,Z,U)$. The hidden confounder that has an effect on $W$ and $Y$ is represented by the doubly-pointed edge.}
\label{fig:graph_confounder}
\end{figure}

Causal graphs encode causal relations between variables in a model. The primary purpose of causal graphs is to help estimate the joint probability of some of the variables in the model upon controlling some other variables by forcing them to specific values; this is called an action,  experiment, or intervention. For every model $M$, every set of variables $X \subset V$ and every set of values $X=x$ we define the model $M_{do(X=x)}$ to be the same as $M$ except that every function $f_i$ for variable $X_i \in X$ assigns a probability distribution of 1 to the value $x_i$ and 0 to the rest of values.

Graphically, this is represented by removing all the incoming edges (which represent the causes) of the variables in the graph that we control in the intervention. Mathematically the $do()$ operator represents this intervention on the variables, by transforming $M$ into $M_{do(X=x)}$. Given a causal graph where $X$ and $Y$ are sets of variables, the expression $P(Y|do(X=x))$ is the joint probability of $Y$ upon doing an intervention on the controlled set $X$, rigorously the application of Equation~\ref{eqn:PofV} on $M_{do(X=x)}$ instead of $M$.

A causal relation represented by the expression $P(Y|do(X=x))$ is said to be \textit{identifiable} if 
it can be uniquely determined from the graph $G$ induced by causal model $M$, and from the joint distribution $P$ of its observed variables. A formal definition of \textit{identifiability} is given in Definition~\ref{def:identifiability}.

In many real-world scenarios, it is impossible, impractical, unethical or too expensive to perform an intervention, thus the interest in evaluating its effects from observational data only, i.e., without actually performing the intervention "in the real world".

\subsection{Do-calculus}

The three rules of do-calculus \parencite{pearl1994probabilistic} allow us to transform expressions with $do()$ operators into other equivalent expressions, based on the causal relations present in the causal graph. 

For any disjoint sets of variables $X$, $Y$, $Z$ and $W$:

\begin{enumerate}
\item
$P(Y|Z,W,do(X))=P(Y|W,do(X))$ 
\\if $(Y\perp Z|X,W)_{G_{\overline{X}}}$
\item 
$P(Y|W,do(X),do(Z))=P(Y|Z,W,do(X))$ 
\\if $(Y\perp Z|X,W)_{G_{\overline{X}\underline{Z}}}$
\item 
$P(Y|W,do(X),do(Z))=P(Y|W,do(X))$ 
\\if $(Y\perp Z|X,W)_{G_{\overline{X}\overline{Z(W)}}}$
\end{enumerate}

Where $(X \perp Y|Z)$ means $X$ and $Y$ are independent conditional to $Z$. $Z(W) = Z \setminus An(W)_{G_{\overline{X}}}$. $An(W)_G$ is the set of ancestors of $W$ in $G$.
$G_{\overline{X}}$ is the graph $G$ where all edges incoming to $X$ are removed from the graph. $G_{\underline{Y}}$ is the graph $G$ where all edges outgoing from $Y$ are removed from the graph.

Do-calculus was proven to be sound and complete \parencite{shpitser2006identification,huang2006identifiability}, in the sense that an expression obtained by iterative application of the three rules of calculus is always correct, and if an expression cannot be converted into a do-free one by iterative application of the three do-calculus rules, then it is not identifiable.


\section{Causal Discovery Algorithms}

\subsection{Causal discovery without interventions}

The algorithms PC \parencite{Ref10}, FCI \parencite{Ref11} and \parencite{Ref10}, and RFCI \parencite{Ref12} can, under some assumptions, infer information about the causal structure from observational data. While PC assumes that all variables in the system under consideration are observed, FCI and RFCI consider structure learning in the presence of hidden variables. Other standard algorithms for causal discovery from observational data include IC, IC* \parencite{Ref14} and \parencite{Ref15}, GES \parencite{Ref13}.

\subsubsection{The PC Algorithm}
\label{sec:PC_alg}
Under the assumption of faithfulness, the PC algorithm estimates the set of DAGs compatible with the probability distribution of the observed variables. In general, several DAGs may be compatible with the conditional independence relations among the observed variables. The output of the PC algorithm is a CPDAG, which represent an equivalence class of DAGs, also called Markov equivalence class. 

A CPDAG may contain both directed and undirected edges. There is an edge (directed or undirected) in a CPDAG if the edge's endpoint variables are conditionally dependent given all possible subsets of the remaining variables. A directed edge in a CPDAG represents that all DAGs in the equivalence class contain the directed edge. An undirected edge in a CPDAG represents that some DAGs in the equivalence class contain the edge going in one direction, while the other DAGs in the class contain the edge going in the opposite direction.

The PC algorithm works as follows. The algorithm starts with a complete undirected graph where all variables have an undirected edge to all other variables in the graph. Then based on conditional independence tests (using the observed distribution) some edges are removed, when the variables connected initially by an edge are found to be in fact independent given subsets of the other variables. Then the edges are oriented based on several criteria. First, each set of three variables $i-j-k$ is oriented as a collider $i->j<-k$ if $i$ and $k$ are found to be dependent given $j$. Then some of the remaining undirected edges may be oriented given the rules that no new collider and no directed cycles should be introduced in the graph.
\subsubsection{The FCI Algorithm}
\label{sec:FCI_alg}
The FCI algorithm, as the PC algorithm, assumes the Faithfulness condition, however it is a generalization of the PC algorithm as it considers hidden confounders may be present in the causal model. The output of the FCI algorithm is a PAG. The FCI algorithm finds an equivalence class of graphs which may include hidden confounders. 

A PAG is an equivalence class of Maximal Ancestral Graphs (MAG). A PAG may contain several types of edges: o-o, o-, o->, ->, <->, -. A tail or an arrowhead on an edge of a PAG means all MAGs in the equivalence class contain the tail or arrowhead. An edge ending with o- means some MAGs in the equivalence class contain a tail and the remaining MAGs contain an arrowhead. Bidirected edges represent hidden confounders.

The FCI algorithm works in a similar way than the PC algorithm. It starts with a completed undirected graph, performs conditional independence tests in order to remove edges, and then orients the remaining edges using some rules. However, the conditional independence tests used in FCI are more conservative than for the PC algorithm due to the potential presence of hidden confounders. Indeed, the hidden confounders may not be conditioned upon, and more precise rules for determining independence are required.

However, the discovery of causal relations from observational data only has its limits. In most cases we will obtain, at best, a set of causal models compatible with the observed data, in other words, an equivalence class of DAGs. Examples of this are Completed Partially Directed Acyclic Graphs (CPDAG), which are an output of the PC algorithm, and Partial Ancestral Graphs (PAG), output of the FCI algorithm.


\subsection{Causal discovery with interventions, in the absence of hidden confounders}

The largest part of the literature on causal discovery with interventions assumes the absence of hidden confounders. This is also described by saying that ``observable variables are causally sufficient'': there may be unobserved variables, but none of them causally affects more than one observable variables.

For the active learning process, \parencite{Ref6}, \parencite{Ref5} and \parencite{Ref7} select the next intervention based on maximally reducing some entropy function amongst causal graphs without hidden confounders, and using interventions on single variables only, which limits the graphs that can be learned.
 
The method exposed in \parencite{Ref1} discusses the number of interventions required when there is no prior knowledge of the causal structure, in causal graphs without hidden confounders. For this, it is required to perform adjacency and directional tests on all pairs of variables naively. This is an interesting setup to identify worst case scenarios, starting with a complete graph. However, in an active learning setup, there will be prior knowledge right after the first intervention, since a number of candidate graphs (even if maximal at the start) are discarded at each iteration. So, the worst-case scenarios in \parencite{Ref1} are interesting theoretical bounds but with limited application in real case scenarios, other than setting the bound for worst case. Also, the fact that the setup is assuming the absence of hidden confounders makes it even more theoretical, whereas real world data and causal models include hidden confounders. 

The algorithm proposed in \parencite{Meganck} uses several decision criteria (maximax, maximin, Laplace) to identify the next best intervention for uncovering causal structure. In essence the method evaluates how many edges will be identified by an intervention, in models without hidden confounders and using single variable interventions. This method offers several limitations, as it does not evaluate the effect of interventions, it only evaluates if there is an effect or not. 

\parencite{Eberhardt_Thesis} provides a comprehensive analysis that has influenced much later research, including our own,
and remarks interesting differences among the cases with and without hidden confounders. 
In particular, it shows that, for graphs with $N$ vertices:

\begin{itemize}
\item $\log_2(N) + 1$ interventions suffice for causal identification in the absence of hidden confounders, both for adaptive and non-adaptive strategies.
\item $N$ interventions are necessary and sufficient for causal identification of the observable graph in the presence of hidden confounders, 
both for adaptive and non-adaptive strategies.
\item In particular, for the ``sufficient'' part, $N$ interventions on $N-1$ observable variables suffice. We note that \parencite{Eberhardt_Thesis} does not
explicitly identify in addition the hidden confounders, but this can be done with existing techniques. 
\end{itemize}

The research following \parencite{Eberhardt_Thesis} usually address specific sub-cases, in order to estimate computational costs of selecting the right interventions, or sub-cases with some restrictions on the types of causal models being considered.

Similar to \parencite{Ref1}, the method proposed in \parencite{Eberhardt2008} analyses the theoretical worst-case scenarios but now starting with a Markov equivalence class of causal graphs, instead of starting with no prior knowledge. This is a useful theoretical setup for identifying worst case bounds in scenarios without hidden confounders. However, in real case scenarios, we will have to account for the presence of hidden confounders. The sub-case without hidden confounders discussed in \parencite{Eberhardt2008}, considers the number of interventions as a function of the cliques (a type of graphical structure) present in the graph.

\parencite{He} use minimax and maximum entropy criteria as scoring method for selecting interventions, either sequentially or in a batch (simultaneous), in causal models without hidden confounders. \parencite{Hyttinen} expand and generalize previous combinatorial results given in \parencite{Ref1} and \parencite{Eberhardt2008} for worst case bounds, with no prior knowledge in causal graphs without hidden confounders, and highlight the assumptions and conditions for complete identification of the causal structure.

\parencite{Hauser2014} propose two greedy approaches, one using single interventions, and another with multiple simultaneous interventions. Both are combinatorial approaches, the first aims at maximizing the number of oriented edges after each intervention and the second aims at minimizing the clique number of the graph with simultaneous intervention on several variables. In both cases these methods do not actually predict and measure effects of interventions, and only use a combinatorial approach, in the context of causal graphs without hidden confounders.

\parencite{DBLP:journals/corr/ShanmugamKDV15} and \parencite{https://doi.org/10.48550/arxiv.2011.00641} analyse the bounds in the number of interventions required to learn a causal graph. However, this is without considering hidden confounders, and only considering the number of interventions and not the overall cost of the sequence. 

In general, interventions on more variables can be considered more costly. The number and the size of the interventions should be somehow taken into account in the analysis, rather than only the number of interventions. Our method removes as many candidates with inexpensive interventions as possible. We consider that the general setting of minimizing the cost for discovering the graph is a more realistic approach than just minimizing the number of interventions.

Minimizing the overall cost of the interventions to learn a causal model, instead of only minimizing the number of interventions, has attracted attention recently in the context of causal graphs without hidden confounders \parencite{DBLP:journals/corr/abs-1709-03625},
\parencite{agrawal19}.


\subsection{Causal discovery with interventions, in the presence of hidden confounders}

For causal models which may include hidden confounders, \parencite{Ref50}, \parencite{Ref51}, \parencite{Ref52}, \parencite{Ref53} introduce methods for the discovery of causal models using interventional data, however the approach is limited to linear models only.

\parencite{Kocaoglu2017} propose an algorithm to learn the ancestral relations and the observable graph using strongly separating sets of nodes, which leads to interventions on a large number of variables. The cost of intervening large numbers of variables is not being considered and not being minimized, as only the number of interventions is being evaluated by the method. The main merit of the algorithm is that it uses $O(d \log^2 N))$ interventions to find the observable graph, therefore circumventing (when $d$ is small) the lower bound of $N$ proved in \parencite{Eberhardt_Thesis}. It also identifies the hidden cofounders using $O(d^2 \log_2 N)$ interventions, instead of the more baseline $O(N^2)$.

\parencite{Acharya} propose methods for learning causal graphs with hidden confounders, and prove bounds
on number of samples and interventions required. However, the proposed framework only considers the number of interventions, and does not evaluate the cost of intervening or observing more or less variables, and hence does not minimize the overall cost of the sequence of interventions. \parencite{DBLP:journals/corr/abs-2012-13976} propose algorithms for the discovery of causal graphs, in the presence of hidden confounders, via a minimum cost set of interventions in specific settings: when an undirected graph is already provided as input and when a superset of causal relations is provided. However other costs are disregarded, like the cost of observations, and the cost of conditional independence testing, which requires sampling several values of the intervened variables. Also, in these settings the presence and location of the hidden confounders are not learned, they are only assumed.

\parencite{DBLP:journals/corr/abs-2005-11736} propose two specific settings with hidden confounders. One setting where the cost of interventions is linear with the number of intervened variables, in order to discover the ancestral relations of the underlying graph (not the entire graph) with minimal cost. Another setting is to discover the entire causal graph including hidden confounders, but only considering the number of interventions and not the overall cost as a function of the size of the interventions.


\section{Causal Identification Algorithms}

Several algorithms have been developed using do-calculus offers for the identification of causal effects. In this thesis we will mainly refer to the ID algorithm and the Transportability algorithm.

\subsection{The ID Algorithm}
\label{sec:ID_alg}

The ID algorithm \parencite{shpitser2006identification}, and earlier versions by \parencite{tianpearl2002,tian2004identifying} implement an  iterative application of do-calculus rules to transform a causal expression $P(Y|do(X))$ into an equivalent expression without any $do()$ terms in semi-Markovian causal graphs. This enables the identification of interventional distributions from non-interventional data in such graphs.

The ID algorithm is sound and complete \parencite{shpitser2006identification} in the sense that if a do-free equivalent expression exists it will be found by the algorithm, and if it does not exist the algorithm will exit and provide an error.

The algorithm specifications are as follows. Inputs: a causal graph $G$, variable sets $X$ and $Y$, and a probability distribution $P$ over the observed variables in $G$; Output: an expression for $P(Y|do(X))$ without any $do()$ terms, or {\em fail}.

The ID algorithm is based on the general method for identification of causal effects by {\em C-component} factorization \parencite{tianphd} and \parencite{tian2002general}.

\begin{definition}[C-component]
A set of nodes S is a C-component in a graph G if any two nodes in S are connected by a path consisting entirely of bidirected
edges in G.
\end{definition}

Tian proved that a graph $G$ can be partitioned into a set {\em C-components}, and the joint distribution $P(V)$ in $G$ can be expressed as a product of interventional distributions factors, where each factor corresponds to a {\em C-component}. If all factors from each component are identifiable then $P(V)$ is identifiable. This is known as the {\em C-component} factorization of causal models.

\begin{lemma}[C-component factorization]
\label{lem:CcomponentFactorization}
Let $M$ be a causal model with graph G. Let $X,Y$ be disjoint variables in $G$. Let $C(G\setminus X) = \{S_1, ... S_k\}$. Then:

\begin{align*}
P(Y|do(X))=\sum_{V\setminus (Y\cup X)}\prod_{i} P(S_i|do(V\setminus S_i)) 
\end{align*}

\end{lemma}

However, in some cases, one or more of the factors, corresponding to {\em C-components} in the graph, are not identifiable. This happens when a {\em C-component} contains a graphical structure called {\em hedge} \parencite{shpitser2006identification}. Before defining a {\em hedge}, Shiptser provides the definition of {\em C-forest}.

\begin{definition}[C-forest]
\label{def:Cforest}
A graph G where the set of all its nodes is a C-component,
and where each node has at most one child is called a C-forest.
\end{definition}

A {\em C-forest} with a set $R$ of nodes with no children is called $R$-rooted. This leads to the following definition for {\em hedge}.

\begin{definition}[Hedge]
\label{def:hedge}
Let $X,Y$ be sets of variables in $G$. Let $F, F$' be R-rooted C-forests in $G$ such that $F$' is a subgraph of $F$, $X$ only occur in $F$, and $R \in
An(Y)_{G_{\overline{X}}}$ , where $An(X)_G$ denotes the set of ancestors of $X$ in $G$. Then $F$ and $F$' form a hedge for $P(Y|do(X))$.
\end{definition}

The {\em hedge} criterion \parencite{shpitser2006identification} states that $P(Y|do(X))$ is identifiable
in $G$ if and only if there are no two {\em C-forests}
$F, F$' that form a {\em hedge} for $P(Y|do(X))$ in $G$.

Another algorithm for the identification of causal effects is given in \parencite{shpitser2012efficient}. 

\subsection{The Transportability Algorithm}
\label{sec:Trang_alg}

\parencite{pearl2011transportability} introduced the sID algorithm, based on do-calculus, to identify a transport formula between two domains, where the effect in a target domain can be estimated from experimental results in a source domain and some observations on the target domain, thus avoiding the need to perform an experiment on the target domain.

Let us consider a country with a number of alternative roads linking city pairs in different provinces. Suppose that the alternative roads are all consistent with the same causal model but have different traffic patterns (proportion of cars/trucks, toll prices, traffic light durations...).
Traffic authorities in one of the provinces may have experimented with policies and observed the impact on, say, traffic delay. This information may be usable to predict the average travel delay in another province for a given traffic policy, provided that the source domain (province where the impact of traffic policy has already been monitored) and target domain (new province) share the same causal relations among variables. 

The target domain may have specific distributions of the toll price and traffic signs, which are accounted for in the model by adding a set of selection variables, pointing at variables whose distribution differs among the two domains. Under some assumptions the transportability algorithm provides a transport formula which combines experimental probabilities from the source domain and observed distributions from the target domain. Thus, the traffic authorities in the new province can evaluate the impacts before effectively changing traffic policies. This amounts to relational knowledge transfer learning between the two domains \parencite{pan2010survey}.


\section{Dynamic Bayesian Networks}

Dynamic Bayesian Networks (DBN) are graphical models that generalize Bayesian Networks (BN) in order to model time-evolving phenomena. 
We rephrase them as follows. 

\begin{definition}
A DBN is a directed graph $D$ over a set of nodes that represent time-evolving metavariables. Some of the arcs in the graph have no label, and others are labelled ``$+1$''.
It is required that the sub-graph $G$ formed by the nodes and the unlabelled edges must be acyclic, therefore forming a DAG.

Unlabelled arcs denote dependence relations between metavariables within the same time step, and arcs labelled ``$+1$'' denote dependence between a variable at one time and another variable at the next time step.
\end{definition}

\begin{definition}
A DBN with graph  $G$ represents an infinite Bayesian Network $\hat G$ as follows. Timestamps $t$ are the integer numbers; $\hat G$ will thus be a biinfinite graph. For each metavariable $X$ in $G$ and each time step $t$ there is a variable 
$X_t$ in $\hat G$. The set of variables indexed by the same $t$ is denoted $G_t$ and called ``the slice at time $t$''. There is an edge from $X_t$ to $Y_t$ iff there is an unlabelled edge from $X$ to $Y$ in $G$, and there is an edge from $X_t$ to $Y_{t+1}$ iff 
there is an edge labelled ``$+1$'' from $X$ to $Y$ in $G$. Note that $\hat G$ is acyclic.
\end{definition}

The set of metavariables in $G$ is denoted $V(G)$, or simply $V$ when $G$ is clear from the context.
Similarly, $V_t(G)$ or $V_t$ denote the variables in the $t$-th slice of $G$. 

In this thesis, we will use transition matrices to model the time evolution of probability distributions. Rows and columns are indexed by tuples assigning values to each variable, and the $(v,w)$ entry of the matrix represents the probability $P(V_{t+1} = w|V_t = v)$.
Let $T_t$ denote this transition matrix. Then we have, in matrix notation, $P(V_{t+1})=T_t\,P(V_t)$  and, more in general, $P(V_{t+\alpha}) = (\prod_{i=t}^{t+\alpha-1}T_i) \, P(V_t)$. In the case of time-invariant distributions, all $T_t$ matrices are the same matrix $T$, so  $P(V_{t+\alpha}) = T^{\alpha} P(V_t)$. 

Note that transition matrices model how probability distributions evolve from one time step to the next. They do not model dependencies between variables within a time step. For modelling dependencies within a time step, or without considering time at all, we may use BNs and, in the context of causality, existing non-dynamic causal model theory.


\thispagestyle{plain} 
\mbox{}


\chapter{Distinguishability of Causal Graphs} 

\label{ch:Chapter3} 


This chapter sets the theoretical foundations for the efficient causal discovery algorithm ALCAM that will be presented in Chapter 4. 

The main difference of ALCAM with respect to most existing algorithms is that it does not take a tabula-rasa approach to discovering the causal structure of the phenomenon. Often, we have previous knowledge from our familiarity with reality: we may know that some variable causally influences another; we may know that some pairs of variables are likely to be affected by common unobserved confounders, etc.

One way of providing this a-priori knowledge to an algorithm is with a set ${\cal G}$ of {\em candidate graphs}, 
with the promise that the graph induced by the
true model is included in this candidate set. In other words, the
algorithm can safely assume that all graphs {\em not} in ${\cal G}$
are definitely not the graph induced by the true model. If the algorithm is delicate enough, it can focus on performing only the interventions needed to eliminate all graphs in ${\cal G}$ other than true graph,
rather than eliminating all potential graphs on $n$ variables,
of which there are $2^{O(n^2)}$.

Thus, if $|{\cal G}|$ (the number of candidate graphs in ${\cal G}$) is much smaller than $2^{O(n^2)}$, our algorithm performs fewer, or simpler, interventions than a tabula-rasa algorithm.

In this chapter we provide lemmas that identify the specific interventions needed to progressively rule out graphs in $\cal G$.
In Chapter 4 we present ALCAM and, using these lemmas, we prove its correctness as well as non-trivial bounds on the
number of interventions it performs. It turns out that this number of interventions can always be bounded by $|{\cal G}|$; 
not only that, we provide finer bounds that depend on how similar or different the candidate graphs are to the
graph induced by the true model, hence very often fewer than  $|{\cal G}|$.

\section{Notation and Basic Lemmas} 

In this thesis all causal models are assumed to contain hidden confounders, except when explicitly referring to causal models without hidden confounders as a sub-case. Also, all variables are assumed to be in a finite domain. We will use the following notation: 

\begin{itemize}

\item $V$: a set of observed variables;

\item $U$: a set of hidden confounder variables;

\item $V_i$, $V_j$: single variables in $V$;

\item $X$, $Y$: disjoint sets of variables in $V$;

\item $v_i$, $v_j$, $x$, $y$: a value assignment for $V_i$, $V_j$, $X$, $Y$ respectively;

\item $M^\star$: the reference or ``true'' causal model over the set $V$ of observed variables and the set $U$ of hidden confounder variables. Unless otherwise indicated, all statements in this thesis are implicitly universally quantified over this true model, meaning that they hold for any $M^\star$.

\item $G^\star$: graph induced by $M^\star$;

\item $P^\star$: probability distribution of the observed variables in $M^\star$ without interventions;

\item $P^\star(Y|do(X=x))$: probability distribution of $Y$, given the intervention $X=x$ in $M^\star$;

\item $\cal G$: a set of causal graphs over the set $V$ of observed variables and the set $U$ of hidden confounder variables, which we call the {\em candidates};

\item $G_k$, $G_l$: two causal graphs in $\cal G$;

\item $E=(X,x,Y)$: intervention or experiment where variable $X$ is set to a value $x$, and the causal effect is measured on variable $Y$;


\end{itemize}

In this thesis we use do-calculus \parencite{pearl1994probabilistic}, its completeness and soundness \parencite{huang2006complete}, as well as the notion of causal effect identifiability and the ID algorithm \parencite{shpitser2006identification} as tools for the identification of causal effects.


\begin{definition}[Causal Effect]
\label{def:causal_effect}
The causal effect of an intervention $E=(X,x,Y)$ in a causal model $M$, mathematically $P(Y|do(X=x))$, is the probability distribution of $Y$ given the intervention $X=x$ in $M$.
\end{definition}

\begin{definition}[Causal Effect Identifiability]
\label{def:identifiability}
Let $M$ be a causal model with causal graph $G$ and observational joint probability distribution $P$.
The causal effect of an intervention $E=(X,x,Y)$ in $M$ is said to be identifiable if $P(Y|do(X=x))$ is uniquely determined from $G$ and $P$, that is, if it is the same probability distribution for every model $M$ that induces the same $G$ and $P$.
\end{definition}


\begin{theorem}[Do-calculus is sound and complete]
\label{thm:DoCalculusComplete}
Let $M$ be a causal model with causal graph $G$ and observational joint probability distribution $P$, then
\begin{itemize}
    \item \textbf{\em Completeness:\ } If the causal effect of an intervention $E=(X,x,Y)$ is identifiable then there is a sequence of application of do-calculus rules that finds the probability distribution $P(Y|do(X=x))$ from $G$ and $P$;
    
    \item \textbf{\em Uniqueness:\ } If the causal effect of an intervention $E=(X,x,Y)$ is identifiable then any sequence of application of do-calculus rules finds the same probability distribution of $P(Y|do(X=x))$ from $G$ and $P$;
    
    \item \textbf{\em Soundness:\ } If there exists a sequence of application of do-calculus rules that finds a probability distribution for $P(Y|do(X=x))$ from $G$ and $P$, then the causal effect is identifiable;
\end{itemize}
\end{theorem}

\textbf{Remark}: P(Y|do(X=x)) is an expression, and therefore a syntactic object. The do-calculus rules perform symbolic (i.e., syntactic) transformations and, for all we know, may yield different do-free expressions that represent the same probability distributions. In this thesis we deal with finite domains and finitely many variables, so it is always possible to determine in finite time by brute-force evaluation if two do-free expressions actually represent the same probability distributions. Therefore, we can without loss of generality assume that the application of do-calculus yields a probability distribution whenever it provides a do-free expression. 
We sometimes write "a (do-free) probability distribution" to emphasize that we mean an actual distribution and not a do-free expression.

In the following Definition~\ref{def:identifiability_graph} we relax the assumptions from Definition~\ref{def:identifiability}, as we do not assume the graph $G$ to be induced by the model.


\begin{definition}[Causal Effect from an arbitrary Graph]
\label{def:identifiability_graph}
Given a causal model $M$ with observational joint probability distribution $P$, and given an arbitrary graph $G_k$ (not necessarily induced from $M$), we call $P_k(Y|do(X=x))$ the causal effect of an intervention $E=(X,x,Y)$ from $G_k$ and $P$, and define it as follows:
\begin{enumerate}
    \item let $P^1_k, .., P^Q_k$ be the set of all (do-free) probability distributions that can be obtained by starting from $P$ and $G_k$ and repeatedly applying do-calculus rules, in any order. Then, 
    $P_k(Y|do(X=x)) = \{P_k^1,... ,P_k^Q\}$;
    \item if there is no (do-free) probability distribution, then $P_k(Y|do(X=x)) = \emptyset$;
\end{enumerate}
\end{definition}

Clearly, if the graph $G_k$ is induced by the model $M$ in the definition above, then by Theorem~\ref{thm:DoCalculusComplete} (uniqueness) the causal effect consists of either a singleton joint probability distribution or the empty set, depending on whether $P_k(Y|do(X=x))$ is identifiable or not.




\begin{lemma}
\label{lem:multi_prob}
Let $M$ be a causal model with observational joint probability distribution $P$. If the causal effect from an arbitrary graph $G_k$ consists of more than one probability distribution, i.e., $P_k(Y|do(X=x)) = \{P_k^1,... ,P_k^Q\}$ with $Q>1$ then $G_k$ is not induced from $M$.
\end{lemma}
\begin{proof}
If the causal graph $G$ is induced by $M$ then, by Theorem~\ref{thm:DoCalculusComplete}, if the causal effect is identifiable then any sequence of application of do-calculus rules finds the same probability distribution.
\qed\end{proof}

\begin{definition}[Distinguishability of Causal Graphs from an Intervention]
\label{def:distinguishability_int}
Let $M$ be a causal model with observational joint probability distribution $P$. Let $G_k$ and $G_l$ be two causal graphs. We say that $G_k$ and $G_l$ are distinguishable from an intervention $E=(X,x,Y)$ under $P$, and note this with $G_k \not\approx_{P_E} G_l$
iff:
\begin{itemize}
\item $P_k(Y|do(X=x)) \neq \emptyset$, $P_l(Y|do(X=x)) \neq \emptyset$ and $P_k(Y|do(X=x)) \neq P_l(Y|do(X=x))$; 

or
\item $P_k(Y|do(X=x)) = \emptyset, P_l(Y|do(X=x))=\{P(Y)\}$; 

or
\item $P_k(Y|do(X=x)) =\{P(Y)\}, P_l(Y|do(X=x))= \emptyset$; 
\end{itemize}
\end{definition}

Definition~\ref{def:distinguishability_int} leads to several case scenarios an intervention makes two causal graphs distinguishable. A list of case scenarios is shown in Table ~\ref{tab:intervention_evaluation} and Figure~\ref{fig:Case_scenarios} shows causal graph examples for each case scenario given in Table ~\ref{tab:intervention_evaluation}. 

\begin{table*}[h!]
  \centering
  \caption{Case scenarios an intervention $E=(X,x,Y)$ makes two causal graphs $G_k, G_l$ distinguishable.
  }
  \label{tab:intervention_evaluation}

  \begin{tabular}{cccc}
    \toprule
    Case & $P_k(Y|do(X))$ & $P_l(Y|do(X))$ & $G_k \not\approx_{P_E} G_l$  \\
    \midrule
    1 & $=\{P(Y)\}$ & $=\{P(Y)\}$ & no \\
    2 & $=\{P(Y)\}$ & $\neq \{P(Y)\}$ & yes \\
    3 & $\neq \{P(Y)\}$ & $\neq \{P(Y)\}$ & no \\
    4 & $\neq \{P(Y)\}$ & $\neq \{P(Y)\}$ & yes \\
    5 & $=\{P(Y)\}$ & $\emptyset$ & yes \\
    6 & $\neq \{P(Y)\}$ & $\emptyset$ & no \\
    7 & $\emptyset$ & $\emptyset$ & no \\

    \bottomrule
  \end{tabular}
\end{table*}

\begin{figure*}[h!]
\begin{center}
\includegraphics[width=1\textwidth]{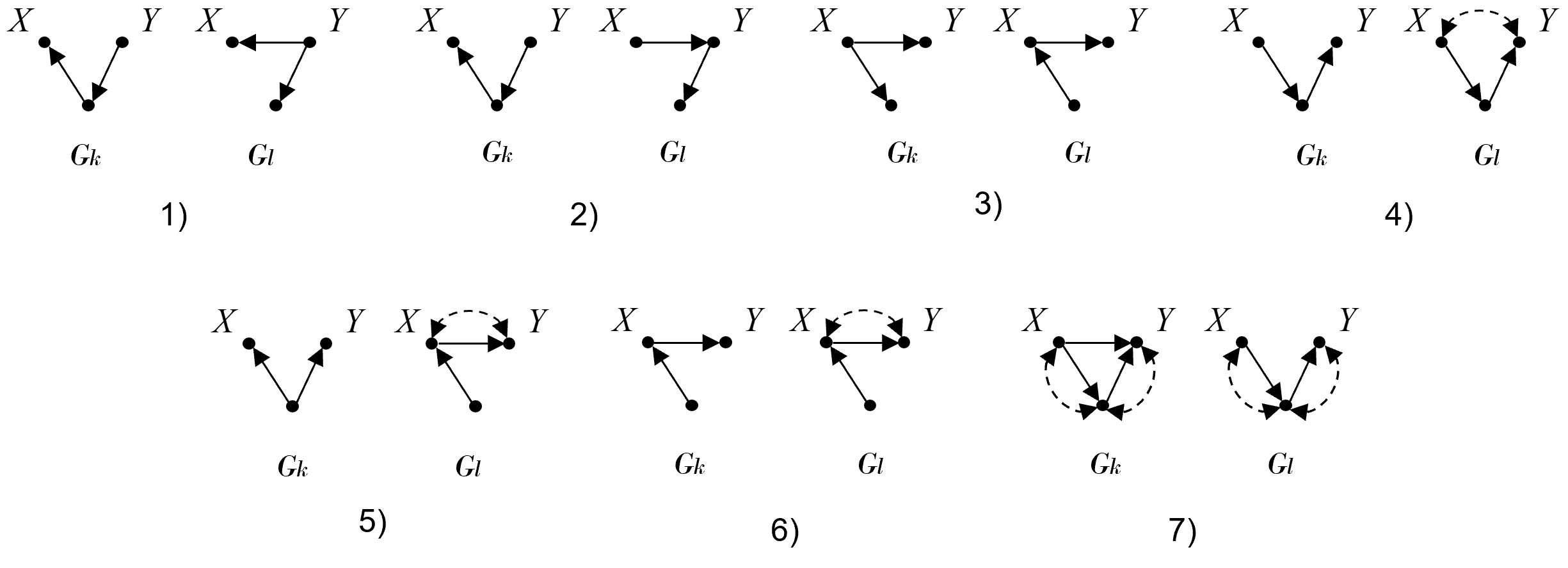}
\end{center}
\caption{Causal graph examples for case scenarios 1-7 from Table ~\ref{tab:intervention_evaluation}.}
\label{fig:Case_scenarios}
\end{figure*}

Note if an intervention does not make the two graphs distinguishable, there may exist other interventions that do.

\begin{definition}[Distinguishability]
\label{def:distinguishability}
Let $P$, $G_k$ and $G_l$ be as in Definition~\ref{def:distinguishability_int}. We say that $G_k$ and $G_l$ are distinguishable iff there exists an intervention $E=(X,x,Y)$ for which $G_k \not\approx_{P_E} G_l$, and note this with $G_k \not\approx_{P} G_l$.

\end{definition}






The intuition behind definitions~\ref{def:distinguishability_int} and \ref{def:distinguishability} is we are able to guarantee that the causal effect from two arbitrary graphs per Definition~\ref{def:identifiability_graph} differs in three case scenarios. Firstly, when there exist causal effect probability distributions from both graphs, and they are different for the two graphs. 

The second and third case occur when the effect is not identifiable in one of the graphs, while there is no effect from $X$ on $Y$ on the other graph. If there is no probability distribution from one of the graphs, there exists a hedge for the intervention \parencite{shpitser2006identification} in that graph, which also implies there exists a direct path from $X$ to $Y$ in the graph. This means there is an effect from $X$ on $Y$ in that graph. If there is no effect from $X$ on $Y$ on the other graph then the two graphs are distinguishable.

Note that for two graphs to be distinguishable it suffices that there exists an intervention $E=(X,x,Y)$ on a single value $X=x$ of the intervened variables, for which one of the conditions of definitions~\ref{def:distinguishability_int} is true.

Figure~\ref{fig:DistinguishThreeGraphs} provides an example of distinguishability among three graphs, across all possible interventions. For the intervention $E=(X,x,Y)=(X_1,x_1,X_4)$ we find $P_1(X_4|do(X_1=x_1)) \neq P_2(X_4|do(X_1=x_1))$ so $G_1 \not\approx_{P_E} G_2$ and therefore $G_1 \not\approx_{P} G_2$. There is a hedge in $G_3$ for the intervention $E$ so $P_3 = \emptyset$ and, since $P_1 \neq P(Y)$ and $P_2 \neq P(Y)$, $G_3$ is not distinguishable from $G_1$ and $G_2$ with $E$: $G_3 \approx_{P_E} G_1$ and $G_3 \approx_{P_E} G_2$. The same logic applies to interventions $P(X_4|do(X_1=x_1,X_2=x_2))$ and $P(X_4|do(X_1=x_1, X_3=x_3))$. And there are no other interventions that make the graphs distinguishable. So $G_3 \approx_{P} G_1$ and $G_3 \approx_{P} G_2$.

\begin{figure}[H]
\begin{center}
\includegraphics[width=0.7\textwidth]{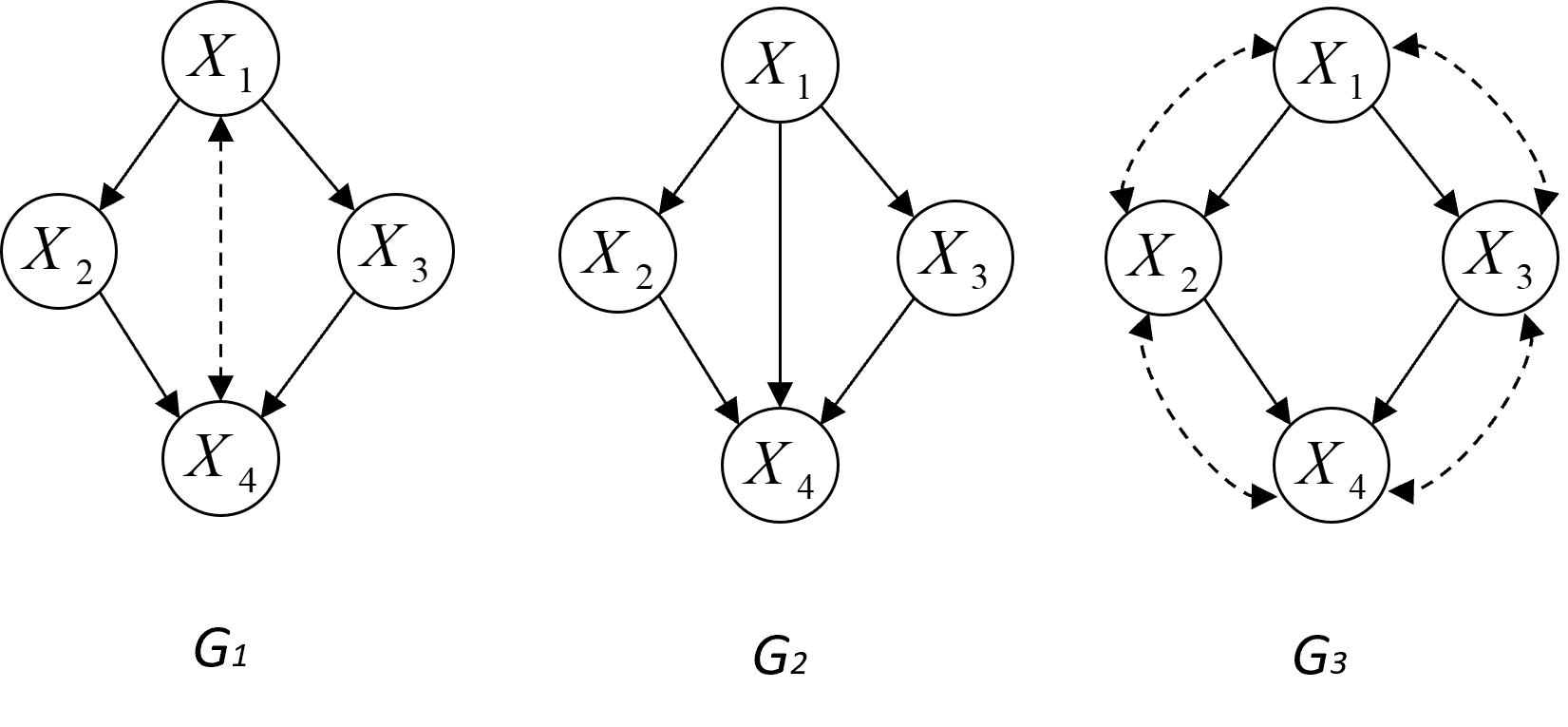}
\end{center}
\caption{Example of distinguishability among three causal graphs. As $P_1(X_4|do(X_1=x_1)) \neq P_2(X_4|do(X_1=x_1))$: $G_1 \not\approx_{P} G_2$. However, $G_3$ is not distinguishable from $G_1$ or $G_2$ due to the hidden confounders forming a {\em hedge} for the interventions $P(X_4|do(X_1=x_1))$, $P(X_4|do(X_1=x_1,X_2=x_2))$ and $P(X_4|do(X_1=x_1, X_3=x_3))$, so
$G_3 \approx_{P} G_1$ and $G_3 \approx_{P} G_2$.
}
\label{fig:DistinguishThreeGraphs}
\end{figure}

\begin{lemma}
\label{lem:DistinguishPair}
Let $M$ be a causal model with induced graph $G$ and observational joint probability distribution $P$. If two causal graphs $G_k$, $G_l$ are distinguishable then either $G_k \neq G$ or $G_l \neq G$.
\end{lemma}
    

\begin{proof}
If $G_k \not\approx_{P} G_l$ then:
\begin{itemize}
    \item $G_k = G \Rightarrow G \not\approx_{P} G_l$
    
    \item $G_l = G \Rightarrow G_k \not\approx_{P} G$
\end{itemize}
If $G \not\approx_{P} G_l$ then $G_l \neq G$, and if $G_k \not\approx_{P} G$ then $G_k \neq G$.
\qed\end{proof}

The key importance of Lemma~\ref{lem:DistinguishPair} is that it is constructive:
if we have some intervention $E$ witnessing $G_k \not\approx_{P} G_l$,
then performing $E$ in the real world lets us pick up one
of $G_k$ and $G_l$ that is guaranteed not to be $G$. 

Let us analyse each one of the three scenarios of distinguishability from Definition~\ref{def:distinguishability_int}. If $G_k$ and $G_l$ are distinguishable due to $P_k(Y|do(X=x)) \neq \emptyset$, $P_l(Y|do(X=x)) \neq \emptyset$ and $P_k(Y|do(X=x)) \neq P_l(Y|do(X=x))$, then $P_k(Y|do(X=x))$ or $P_l(Y|do(X=x))$ will differ from the effect of the intervention in the real world, so we can guarantee that $G_k$ or $G_l$ differ from the true model. 

If $P_k(Y|do(X=x)) = \emptyset, P_l(Y|do(X=x))=\{P(Y)\}$ and the effect of the intervention in the real world is $P(Y)$ then, based on the completeness of do-calculus given in Theorem~\ref{thm:DoCalculusComplete}, $G_k$ is not correct, because otherwise $P_k(Y|do(X=x))$ would be $P(Y)$. Alternatively, if the effect of the intervention in the real world is not $P(Y)$ then, based on the soundness and uniqueness of do-calculus given in Theorem~\ref{thm:DoCalculusComplete}, $G_l$ is not correct.

\section{Structure of What is to Come}

Lemma~\ref{lem:DistinguishPair} is the central tool of our strategy to prune
the set of candidates down to a single element, which must
be the {\em true} graph. Indeed, if we can find an intervention
$E$ that distinguishes any two candidates $G_k$ and $G_l$, 
then performing that intervention ``in the real world'' 
is guaranteed, by Lemma~\ref{lem:DistinguishPair}, to rule out either $G_k$ or $G_l$.

\bigbreak
\bigbreak

Our strategy reduces the candidate set to one graph in phases.

{\bf 1. Pruning candidates that are distinguishable.}

We first prune the candidate graphs that are distinguishable from the {\em true} graph. We prove in Section~\ref{sec_graphical_conditions} that by doing this we can prune many candidate graphs, and leave only the graphs
that have the same ancestor relations, the same edges and the same hidden confounders as the {\em true} graph, except graphs with edges and hidden confounders for which there are {\em hedges} on some specific interventions. We show in Lemma~\ref{lem:DistinguishTopological} that for any two graphs 
that do not have the same ancestor relation one can always find
an intervention $E=(X,x,Y)$ that distinguishes them in the sense
of Definition~\ref{def:distinguishability_int}. Also, we show in Lemma~\ref{lem:DistinguishEdges} that for any two graphs where there exists an edge in one of the graphs that does not exist in the other, we can find an intervention that distinguishes the two graphs, except if the intervention forms a {\em hedge}. And we show in Lemma~\ref{lem:DistinguishHC} that for any two graphs where there exists a hidden confounder in one of the graphs that does not exist in the other, we can find an intervention that distinguishes the two graphs, except if the intervention forms a {\em hedge} or the hidden confounder is across variables not ancestors of each other.

Furthermore, all these interventions can be found
with no intervention "in the real world", by just applying the rules
of do-calculus and computation. And since the graphs are distinguishable,
we know by Lemma~\ref{lem:DistinguishPair} that at least one of them is not the {\em true} graph, and by performing the interventions "in the real world", we can find which candidate graphs are wrong.

Another interesting point is that these interventions are with
a single assignment $x$ to the intervened variables $X$.
In contrast, existing causal discovery algorithms
use conditional independence tests which need
to test all values of $X$. A conditional independence test when the intervention on $X$ is required implies interventions on all values of $X$. Presumably, in almost all real cases
testing for all the values of $X$ is more expensive than testing a single
value of $X$. Therefore, this pruning of the set of candidates
can in many cases lead to enormous savings.

Finally, the pruning of distinguishable graphs leads directly to the {\em true} graph, and nothing else is required to find the {\em true} graph, except when some specific graphical conditions exist in the set of candidate graphs, per Lemma~\ref{lem:DistinguishEdges} and Lemma~\ref{lem:DistinguishHC}. For example, as a sub-case, when there are no hidden confounders in the set of candidate graphs, the pruning of distinguishable graphs leads directly to the {\em true} graph. This in turn allows us to find a sequence of single valued interventions with minimal cost to complete the learning process.

{\bf 2. Pruning candidates with "wrong observable edges" using conditional independence tests. }

In some cases, when specific graphical conditions exist in the set of candidate graphs (given by Lemma~\ref{lem:DistinguishEdges}), it may happen that after the pruning of distinguishable candidate graphs, some of the remaining graphs still differ in the presence or absence of edges among
observable variables.

To determine the presence or absence of such edges, therefore to discard further graphs, one needs
to perform conditional independence tests, which we discuss in Section~\ref{sec_non_dist_edg}.

In Lemma~\ref{lem:CI_Edges} we give a graphical condition that lets us choose
one smallest conditional independence tests that distinguishes
graphs that differ in at least one edge. By performing that
test "in the real world", we can then select the graphs based on if the test confirms the presence or absence of the edge. 
We can then iteratively apply this test to leave only the candidate
graphs whose observable set of edges is the same, which we write
as "have the same observable graph". 

In contrast, \parencite{Eberhardt2008} performs conditional independence
tests via systematic interventions on $n-1$ variables, in causal models without hidden confounders. \parencite{Kocaoglu2017} and \parencite{DBLP:journals/corr/abs-2012-13976} perform conditional independence tests using systematic interventions on strongly separating sets of variables. In practice, interventions tend to be more
expensive the more variables they intervene. Furthermore, there may be
parts of the causal model whose structure is somehow known, and 
therefore there is no real need to perform interventions on those parts.
Therefore, our surgical approach intervenes only "where the uncertainties lie",
which is more efficient in practice than the carpet-bombing approach of \parencite{Eberhardt2008}, \parencite{Kocaoglu2017} and \parencite{DBLP:journals/corr/abs-2012-13976}.

{\bf 3. Pruning candidates with "wrong hidden confounders" using conditional independence tests.}

Finally, in some cases, when specific graphical conditions exist in the set of candidate graphs (given by Lemma~\ref{lem:DistinguishHC}), it may happen that after pruning the distinguishable candidate graphs, and pruning the graphs with "wrong edges" using conditional independence tests, some of the remaining graphs still differ in the presence or absence of some hidden confounders.

To determine the presence or absence of such hidden confounders, therefore to discard further graphs, one needs
to perform conditional independence tests, which we discuss in Section~\ref{sec_non_dist_hc}. 

In Theorems~\ref{thm:confounders_non_adjacent} and \ref{thm:confounders_adjacent} we give graphical conditions that lets us choose conditional independence tests that distinguishes
graphs that differ in at least one hidden confounder. By performing that
test "in the real world", we can then select the graphs based on if the test confirms the presence or absence of the hidden confounder. 
We can then iteratively apply this test to leave only the candidate
graphs whose set of hidden confounders is the same as the {\em true} graph.

Note that previous interventions, for example when pruning the distinguishable graphs, may already have
removed graphs with different hidden confounders, so there may only be few graphs with hidden confounder differences. Again, our surgical approach performs only
interventions where the candidates differ among themselves, which
will be more efficient in many practical cases.

\section{Power of an Intervention}

Given a set of candidate causal graphs, we want to find interventions for which a maximal number of candidates are distinguishable from each other. If we query an oracle for the causal effects from these interventions, we may eliminate a maximal number of candidate graphs, as their causal effect differs from the response from the oracle.

In this section we introduce a measure of the
power an intervention has to make a set of graphs distinguishable. 

\begin{definition}[Power of an Intervention] \label{def:discr_value}
Let $\cal G$ be a set of causal graphs. Let $M$ be a causal model with observational joint probability distribution $P$. We say that the Power of Intervention $E=(X,x,Y)$ over $\cal G$ and under $P$, and note this with $PI(E,\cal G,$$P)$, is the number of pairs of graphs in $\cal G$ that are distinguishable with this intervention under $P$, i.e., the number of pairs $G_k,G_l$ in $\cal G$ that satisfy one of the three distinguishability conditions

\begin{itemize}
\item $P_k(Y|do(X=x)) \neq \emptyset$, $P_l(Y|do(X=x)) \neq \emptyset$ and $P_k(Y|do(X=x)) \neq P_l(Y|do(X=x))$; 

or
\item $P_k(Y|do(X=x)) = \emptyset, P_l(Y|do(X=x))=\{P(Y)\}$; 

or
\item $P_k(Y|do(X=x)) =\{P(Y)\}, P_l(Y|do(X=x))= \emptyset$;

\end{itemize}

\end{definition}

Given a set of causal graphs including the $true$ graph $G{^\star}$, if we find interventions with $PI(E,\cal G,$$P^\star) > $$0$ we are guaranteed to discover invalid graphs. Assuming we have an oracle that provides the causal effect from the {\em true} graph from the intervention, we can eliminate graphs for which the causal effect differs from the oracle response. Interventions with $PI(E,\cal G,$$P^\star) = 0$, are not able to eliminate candidate graphs, while interventions with the highest $PI(E,\cal G,$$P^\star)$ eliminate the maximal possible number of graphs from the candidate set $\cal G$. This process allows us to eliminate all graphs that are distinguishable from the {\em true} graph with a small number of interventions. Also, we can control the cost of the process, choosing interventions with high $PI$ and low cost.

It may happen that more than one graph remain in $\cal G$ and all interventions have $PI(E,\cal G,$$P^\star)=0$. This is the case when the remaining graphs are not distinguishable. This is a worst case which we must consider to complete the discovery process in all situations. Our approach is to first remove as many candidate graphs as possible, using a sequence of interventions with $PI(E,\cal G,$$P^\star)>0$ and minimum total cost of the sequence, until there remains only one graph, or there are no further interventions with $PI(E,\cal G,$$P^\star) > $$0$. Only in the latter case we apply additional steps with conditional independence testing. If the remaining candidates contain edge differences, then we apply conditional independence tests specifically for detecting the edges that are different between the remaining candidate graphs, as discussed in Section~\ref{sec_non_dist_edg}. And if the remaining candidates contain hidden confounder differences, then we apply conditional independence tests specifically for detecting these hidden confounders, as discussed in Section~\ref{sec_non_dist_hc}. This is likely to reduce the overall cost of the discovery process. What is interesting is that we can formulate an algorithm whose improvement over the algorithms that brute-force through conditional independence tests is quantifiable, as a function of the "goodness" of the initial set of candidates.

\section{Graphical Conditions for Distinguishability}
\label{sec_graphical_conditions}

In this section we analyse the graphical conditions that make causal graphs distinguishable. We first define {\em ancestral relation}.

\begin{definition}[Ancestral Relation]
\label{def:ancestral_relation}
We say that two graphs $G_k$ and $G_l$ with the same observable variables have the same ancestral relations iff for every pair of variables $V_i, V_j$ 
\begin{itemize}
\item iff $V_i \in An(V_j)$ in $G_k$ then $V_i \in An(V_j)$ in $G_l$.
\end{itemize}
\end{definition}

\begin{lemma}[Distinguishability of Ancestral Relations]
\label{lem:DistinguishTopological}
Let $G_k, G_l$ be two causal graphs. Let $P$ be an observational joint probability distribution. If the ancestral relations of $G_k$ and $G_l$ are different, then $G_k \not\approx_{P} G_l$.  
\end{lemma}

\begin{proof}
If the ancestral relations of $G_k$ and $G_l$ are different, then there exists an intervention $E=(V_i,v_i,V_j)$ for which $V_i$ is an ancestor of $V_j$ in $G_k$ and $V_i$ is not an ancestor of $V_j$ in $G_l$ (or vice-versa).
Then 

\begin{itemize}

\item $P_k(V_j|do(V_i=v_i)) = \emptyset$, if there is a hedge for $V_i,V_j$ in $G_k$

\item $P_k(V_j|do(V_i=v_i)) \neq \{P(V_j)\}$, if there is no hedge for $V_i,V_j$ in $G_k$

\item $P_l(V_j|do(V_i=v_i)) = \{P(V_j)\}$

\end{itemize}

so $G_k$ and $G_l$ are distinguishable from the intervention $E=(V_i,v_i,V_j)$.

\qed\end{proof}

Graphs with the same ancestral relations may have differences in some of their edges. We are interested in finding interventions that distinguish graphs with the same ancestral relations but different edges.

\begin{lemma}[Distinguishability of Edges]
\label{lem:DistinguishEdges}
Let $G_k, G_l$ be two causal graphs with the same ancestral relations. If

\begin {itemize}
\item there exists an edge from $V_i$ to $V_j$ in $G_k$ that does not exist in $G_l$;

and

\item there are no hedges in $G_k$ and $G_l$ for $P_k(V_j|do(block(V_i,V_j)))$ and $P_l(V_j|do(block(V_i,V_j)))$ respectively, where $block(V_i,V_j)$ is a set of variables that blocks all paths from $V_i$ to $V_j$ in $G_l$, such that $(V_i \perp V_j|block(V_i,V_j))$ in $G_l{}_{\overline{V_i}\overline{block(V_i,V_j)}}$ ;

\end {itemize}
then $G_k \not\approx_{P} G_l$.

\end{lemma}

\begin{proof}
Let $block(V_i,V_j)$ be any set of variables such that $(V_i \perp V_j|block(V_i,V_j))$ in $G_l{}_{\overline{V_i}\overline{block(V_i,V_j)}}$ (for example, but not necessarily, one of minimal cardinality). This condition is the do-calculus rule 3 condition for the intervention $P_l(V_j|do(block(V_i,V_j)))$ in $G_l$. Applying do-calculus rule 3 we obtain:

\begin{align*}
&P_l(V_j|do(block(V_i,V_j)))=P_l(V_j|do(V_i,block(V_i,V_j))) = \\
\end{align*}

by C-component factorization

\begin{align*}
& = \sum_{V \setminus (V_i  \cup V_j \cup block(V_i,V_j))} \prod_{r} P(S_r| do (V \setminus S_r))\\
\end{align*}

where $S_r$ are the C-components  $C( G_l \setminus V_i \setminus block(V_i,V_j) ) = \{S_1, S_2,...\}$.

In $G_k$, rule 3 cannot be applied due to the presence of the edge from $V_i$ to $V_j$, and the C-component factorization gives

\begin{align*}
& P_k(V_j|do(block(V_i,V_j))) = \sum_{V \setminus ( V_j \cup block(V_i,V_j) )} \prod_{r} P(S'_r| do (V \setminus S'_r)) = \\
\end{align*}

where $S'_r$ are the C-components  $C( G_k \setminus block(V_i,V_j) ) = \{S'_1, S'_2,...\}$. Clearly, one of the C-component $S'$ in $G_l$ includes $V_i$, and none of the C-components $S$ in $G_k$ include $V_i$. Assume the two graphs differ only in one C-component $S'_{\phi}$ and $S_{\phi}$, where $S'_{\phi}$ includes $V_i$ and $S_{\phi}$ includes all variables in $S'_{\phi}$ except $V_i$, and $S'_r = S_r$ for all other C-components, then

\begin{align*}
& = \sum_{V_i} \sum_{V \setminus (V_i \cup V_j \cup block(V_i,V_j) )} \prod_{r\neq \phi} P(S'_r| do (V \setminus S'_r)) P(S'_{\phi}| do (V \setminus S'_{\phi})) \frac{P(S_{\phi}| do (V \setminus S_{\phi}))} {P(S_{\phi}| do (V \setminus S_{\phi}))} \\
\end{align*}

and recombining the factors

\begin{align*}
&  P_k(V_j|do(block(V_i,V_j))) = \sum_{V_i}  P_l(V_j|do(block(V_i,V_j)))  \frac{P(S'_{\phi}| do (V \setminus S'_{\phi}))} {P(S_{\phi}| do (V \setminus S_{\phi}))} \\
\end{align*}

By assumption there are no hedges for the intervention in $G_k$ and $G_l$ so $P_k(V_j|do(block(V_i,V_j))) \neq  \emptyset$ and $P_l(V_j|do(block(V_i,V_j))) \neq  \emptyset$. The factor $\frac{P(S'_{\phi}| do (V \setminus S'_{\phi}))} {P(S_{\phi}| do (V \setminus S_{\phi}))} \neq 1$ as $V_i \in S'_{\phi}$ and $V_i \notin S_{\phi}$, and the marginalization by $V_i$ makes $P_k(V_j|do(block(V_i,V_j))) \neq P_l(V_j|do(block(V_i,V_j)))$ so $G_k$ and $G_l$ are distinguishable. 

In the sub-case where $S'_{\phi}=V_i$ and $S_{\phi}=\emptyset$

\begin{align*}
&  P_k(V_j|do(block(V_i,V_j))) = \sum_{V_i}  P_l(V_j|do(block(V_i,V_j)))  P(V_i| do (V \setminus V_i)) \\
\end{align*}

and $P_k(V_j|do(block(V_i,V_j))) \neq P_l(V_j|do(block(V_i,V_j)))$.

If the two graphs differ in additional C-components (due to the presence of different hidden confounders in $G_k$ and $G_l$), we assume that the effects of the other C-components will not cancel out exactly, and across all values of $V_i$,$V_j$, the difference coming from C-components $S'_{\phi}$ and $S_{\phi}$.

\qed\end{proof}

Graphs with the same ancestral relations and the same edges, in other words, graphs with the same observable graph, may have differences in some of their hidden confounders. We are interested in finding interventions that distinguish graphs with the same observable graph but different hidden confounders.

\begin{lemma}[Distinguishability of Hidden Confounders]
\label{lem:DistinguishHC}
Let $G_k, G_l$ be two causal graphs with the same observed graph. If

\begin {itemize}
\item there exists a hidden confounder between $V_i$ and $V_j$ in $G_k$ that does not exist in $G_l$;

and

\item there is no hedge in $G_k$ for $P_k(V_j|do(V_i))$ and in $G_l$ for $P_l(V_j|do(V_i))$;

and 

\item $V_i \in An(V_j)$;

\end {itemize}
then $G_k \not\approx_{P} G_l$.

\end{lemma}

\begin{proof}
By C-component factorization

\begin{align*}
& P_k(V_j|do(V_i)) = \sum_{V \setminus (V_i \cup V_j )} \prod_{r} P(S_r| do (V \setminus S_r))\\
\end{align*}

\begin{align*}
& P_l(V_j|do(V_i)) = \sum_{V \setminus (V_i \cup V_j )} \prod_{r} P(S'_r| do (V \setminus S'_r))\\
\end{align*}

where $S_r$ are the C-components  $C( G_k \setminus V_i ) = \{S_1, S_2,...\}$, and $S'_r$ are the C-components  $C( G_l \setminus V_i ) = \{S'_1, S'_2,...\}$

In $G_k$ there is one C-component $S_{\phi} \subset \hat{S}_{\phi} \in C(G_k)$, where $S_{\phi}$ includes $V_j$ and $\hat{S}_{\phi}$ includes $V_i$ and $V_j$, due to the presence of the hidden confounder between $V_i$ and $V_j$. Also, by assumption $V_i \in An(V_j)$ and there are no hedges, so the conditions of the last recursive call of the ID algorithm are satisfied \parencite{shpitser2006identification}, then in $G_k$:

\begin{align*}
& P(S_{\phi}| do (V \setminus S_{\phi})) = \sum_{\hat{S}_{\phi} \setminus V_j } \prod_{\{t|V_t\in \hat{S}_{\phi}\}} P(V_t|V_{\pi}^{(t-1)})  \\\
\end{align*}

where $V_{\pi}^{(t-1)}$ is the set of nodes preceding $V_t$ in some topological ordering of $\hat{S}_{\phi}$ in $G$.

In $G_l$ there is one C-component $S'_{\phi} \in C(G_l)$, where $S'_{\phi}$ includes $V_j$, as there is no hidden confounder between $V_i$ and $V_j$. Also, by assumption $V_i \in An(V_j)$ and there are no hedges, then in $G_l$:

\begin{align*}
& P(S'_{\phi}| do (V \setminus S'_{\phi})) = \sum_{S'_{\phi} \setminus V_j } \prod_{\{t|V_t \in S'_{\phi}\}} P(V_t|V_{\pi}^{(t-1)})  \\\
\end{align*}

As $\hat{S}_{\phi} \setminus S'_{\phi} = V_i$
then $P(S_{\phi}| do (V \setminus S_{\phi}))$ in $G_k$ includes the additional factor $P(V_i|V_{\pi}^{(i-1)})$, compared with $P(S'_{\phi}| do (V \setminus S'_{\phi}))$ in $G_l$. Assume the two graphs differ only in the C-components $S_{\phi}$ and $S'_{\phi}$, where $S_{\phi}$ includes the additional factor $P(V_i|V_{\pi}^{(i-1)})$, and $S_r = S'_r$ for all other C-components, then

\begin{align*}
& P_k(V_j|do(V_i)) = \sum_{V \setminus (V_i \cup V_j )} \prod_{r} P(S_r| do (V \setminus S_r)) = \\
\end{align*}

\begin{align*}
& = \sum_{V \setminus (V_i \cup V_j )} \prod_{r\neq \phi} P(S_r| do (V \setminus S_r)) P(S_{\phi}| do (V \setminus S_{\phi})) \frac{P(S'_{\phi}| do (V \setminus S'_{\phi}))} {P(S'_{\phi}| do (V \setminus S'_{\phi}))} \\
\end{align*}

and recombining the factors

\begin{align*}
&  P_k(V_j|do(V_i)) = P_l(V_j|do(V_i))  \frac{P(S_{\phi}| do (V \setminus S_{\phi}))} {P(S'_{\phi}| do (V \setminus S'_{\phi}))} \\
\end{align*}

\bigbreak
\bigbreak
\bigbreak
\bigbreak
\bigbreak
\bigbreak

where

\begin{align*}
&  \frac{P(S_{\phi}| do (V \setminus S_{\phi}))} {P(S'_{\phi}| do (V \setminus S'_{\phi}))} = \frac{\sum_{\hat{S}_{\phi} \setminus V_j } \prod_{\{t|V_t\in \hat{S}_{\phi}\}} P(V_t|V_{\pi}^{(t-1)})} {\sum_{S'_{\phi} \setminus V_j } \prod_{\{t|V_t \in S'_{\phi}\}} P(V_t|V_{\pi}^{(t-1)}) } = \\
\end{align*}

\begin{align*}
&  = \frac{\sum_{V_i} \sum_{S'_{\phi} \setminus V_j} \prod_{\{t|V_t\in S'_{\phi}\}} P(V_t|V_{\pi}^{(t-1)}) P(V_i|V_{\pi}^{(i-1)})} {\sum_{S'_{\phi} \setminus V_j } \prod_{\{t|V_t \in S'_{\phi}\}} P(V_t|V_{\pi}^{(t-1)}) } \\
\end{align*}

By assumption there are no hedges for the intervention in $G_k$ and $G_l$ so $P_k(V_j|V_j)) \neq  \emptyset$ and $P_l(V_j|V_i)) \neq  \emptyset$ and the factor $\frac{P(S_{\phi}| do (V \setminus S_{\phi}))} {P(S'_{\phi}| do (V \setminus S'_{\phi}))} \neq 1$, so $G_k$ and $G_l$ are distinguishable. 

If the two graphs differ in additional C-components (due to the presence of different hidden confounders in $G_k$ and $G_l$), we assume that the effects of the other C-components will not cancel out exactly, and across all values of $V_i$,$V_j$, the difference coming from C-components $S'_{\phi}$ and $S_{\phi}$.

\qed\end{proof}

\begin{lemma}
\label{lem:Dist_Graphs}
Two graphs $G_k, G_l$ are distinguishable if :
\begin{enumerate}

\item they have different ancestral relations;

or

\item they have the same ancestral relations, there exists an edge from $V_i$ to $V_j$ in $G_k$ that does not exist in $G_l$, and there are no hedges for $P(V_j|do(block(V_i,V_j)))$ in either graph, where $block(V_i,V_j)$ is a set of variables that blocks all paths from $V_i$ to $V_j$ in $G_l$, such that $(V_i \perp V_j|block(V_i,V_j))$ in $G_l{}_{\overline{V_i}\overline{block(V_i,V_j)}}$ ;

or

\item they have the same observable graph, there exists a hidden confounder between $V_i$ and $V_j$ in $G_k$ that does not exist in $G_l$, there are no hedges for $P(V_j|do(V_i))$ in either graph, and $V_i \in An(V_j)$;

\end{enumerate}
\end{lemma}

\begin{proof}
By lemmas~\ref{lem:DistinguishTopological},~\ref{lem:DistinguishEdges} and ~\ref{lem:DistinguishHC}.

\qed\end{proof}

Lemma~\ref{lem:Dist_Graphs} provides conditions under which two graphs are distinguishable. This means we can distinguish them with interventions using a single value of the intervened variables. However, two graphs may contain edges or hidden confounder differences that we cannot distinguish with such interventions, particularly in the presence of hedges, in which case we need to use conditional independence testing, as described in sections~\ref{sec_non_dist_edg} and ~\ref{sec_non_dist_hc}.

\section{Testing Non-Distinguishable Edges}
\label{sec_non_dist_edg}

In this section we discuss the criteria for detecting edges using conditional independence tests. As discussed in previous sections, this is avoidable if interventions with $PI(E,\cal G,$$P^\star) > $$0$ have eliminated enough candidate graphs so that the remaining candidates do not contain edge differences. In other words, given a set of candidate causal graphs, conditional independence tests for detecting edges are only required if $PI(E,\cal G,$$P^\star)= $$0$ for all interventions across the remaining candidate graphs (the remaining candidate graphs are non-distinguishable) and some of the remaining candidate graphs contain edge differences.

\begin{lemma}[Conditional Independence Testing of Edges]
\label{lem:CI_Edges}
Let $M$ be a causal model with observational joint probability distribution $P$. Let ${\cal G}$ be a set of causal graphs that includes the graph $G$ induced by $M$. Let ${\cal G}_{ed} \subset {\cal G}$ be the subset of all graphs in ${\cal G}$ that contain an edge from $V_i$ to $V_j$: $ed(V_i,V_j)$, and ${\cal G} \setminus {\cal G}_{ed}$ be the subset of graphs that do not contain $ed(V_i,V_j)$. Let ${\cal G}_{ed} \neq \emptyset$ and ${\cal G} \setminus {\cal G}_{ed} \neq \emptyset$. Let $dsep(V_i,V_j)$ be a minimal set of variables whose intervention d-separates $V_i$ and $V_j$ in all graphs in ${\cal G} \setminus {\cal G}_{ed}$. Then

\begin {itemize}
\item iff $G \in {\cal G}_{ed}$ then $V_i$ and $V_j$ are dependent in $M$ under the intervention on $dsep(V_i,V_j)$;

\item iff $G \in {\cal G} \setminus {\cal G}_{ed}$ then $V_i$ and $V_j$ are independent in $M$ under the intervention on $dsep(V_i,V_j)$;

\end {itemize}

\end{lemma}

\begin{proof}
If $G \in {\cal G} \setminus {\cal G}_{ed}$ an intervention on $dsep(V_i,V_j)$ in $M$ d-separates $V_i$ and $V_j$, as $G$ does not have the edge $ed(V_i,V_j)$, so $V_i$ and $V_j$ are independent, and vice-versa. If $G \in {\cal G}_{ed}$ an intervention on $dsep(V_i,V_j)$ in $M$ does not d-separate $V_i$ and $V_j$, due to the presence of the edge $ed(V_i,V_j)$ in $G$, so $V_i$ and $V_j$ are dependent, and vice-versa.

\qed\end{proof}

Lemma~\ref{lem:CI_Edges} provides a strategy for further reducing the set of candidates in ${\cal G}$. Each edge $ed(V_i,V_j)$ such that ${\cal G}_{ed} \neq \emptyset$ and ${\cal G} \setminus {\cal G}_{ed} \neq \emptyset$
leads to a conditional independence test between $V_i$ and $V_j$, under the intervention on $dsep(V_i,V_j)$. If $V_i$ and $V_j$ are independent then we know none of the graphs in ${\cal G}_{ed}$ is the {\em true} graph. Else, we know that none of the graphs in ${\cal G} \setminus {\cal G}_{ed}$ is the {\em true} graph. Note that it may be required to intervene on $V_i$ to d-separate $V_i$ and $V_j$, due to the presence of hidden confounders in ${\cal G} \setminus {\cal G}_{ed}$, in which case the conditional independence test requires doing interventions across all values of $V_i$.

\section{Testing Non-Distinguishable Hidden Confounders}
\label{sec_non_dist_hc}

In this section we discuss the criteria for detecting hidden confounders using conditional independence tests. Given a set of candidate causal graphs, this is only required if $PI(E,\cal G,$$P^\star)= $$0$ for all interventions, as discussed in previous sections. If $PI(E,\cal G,$$P^\star)= $$0$ across all interventions, we cannot eliminate further candidate graphs with interventions, however in some cases the remaining candidate graphs  may still contain some hidden confounder differences among them. At this stage, all remaining candidate graphs have the same edges between observable variables than the {\em true} graph. We say that we have learned the {\em observable} graph. We use the {\em observable} graph in the criteria to test the remaining hidden confounders.

We formally define {\em confounding} in definition~\ref{def:confounding}. 

\begin{definition}[Confounding]
\label{def:confounding}
We say that two variables $V_i$ and $V_j$ are not confounded iff:
\begin{align*}
P(V_j|do(V_i))=P(V_j|V_i)
\end{align*}
\end{definition}


Two variables $V_i, V_j$ are confounded if there exists another variable with causal effect on both $V_i$ and $V_j$, in which case $P(V_j|do(V_i)) \neq P(V_j|V_i)$. Note that two observed variables may be confounded by another observed variable or may be confounded by an unobserved variable, a hidden confounder.

Given the {\em observable} graph, an intervention on all observed parents of $V_i$ and all observed parents of $V_j$ makes $V_i$ and $V_j$ independent, except if there is an edge between $V_i$ and $V_j$ (they are adjacent in the graph) or there is a hidden confounder between them. This leads to Theorem~\ref{thm:confounders_non_adjacent} \parencite{Kocaoglu2017} for non-adjacent variables.
We assume that the causal relations between the observed variables in the graph are known. In other words, the adjacencies between observed variables are known, and the observed parents of all observed variables are also known.



\begin{theorem}[Hidden Confounders between non-adjacent variables]
\label{thm:confounders_non_adjacent}
Two non-adjacent variables $V_i$ and $V_j$ are not confounded by an unobserved variable iff

\begin{align*}
&P(V_j=v_j|do(O)) = P(V_j=v_j|V_i=v_i,do(O))
\end{align*}

for all values $V_i=v_i$ and $V_j=v_j$, where $O$ is the union of the set of observed parents of $V_i$ and the set of observed parents of $V_j$.
\end{theorem}

Theorem~\ref{thm:confounders_non_adjacent} leads to conditional independence tests, where the equality is tested across all values $V_i=v_i$ and $V_j=v_j$. Note that it is not required to test the equality across all values of $O$. The equality is tested with a constant value of $O$.
 
When $V_i$ and $V_j$ are adjacent, the observed parents of $V_i$ and $V_j$ include $V_i$ or $V_j$ so an intervention on $O$ implies intervening $V_i$ or $V_j$. The criteria used to in Theorem~\ref{thm:confounders_non_adjacent} to detect hidden confounders for non-adjacent variables cannot be used for adjacent variables. Theorem~\ref{thm:confounders_adjacent} \parencite{Kocaoglu2017} provides a criterion for detecting hidden confounders when $V_i$ and $V_j$ are adjacent variables in the {\em observable} graph.

\begin{theorem}[Hidden Confounders between adjacent variables]
\label{thm:confounders_adjacent}
Two adjacent variables $V_i$ and $V_j$ are not confounded by an unobserved variable iff
\begin{align*}
&P(V_j=v_j|do(V_i=v_i,O)) = P(V_j=v_j|V_i=v_i,do(O))
\end{align*}

for all values $V_i=v_i$ and $V_j=v_j$, where $O$ is the union of the set of observed parents of $V_i$ and the set of observed parents of $V_j$.
\end{theorem}

Theorem~\ref{thm:confounders_adjacent} also leads to a conditional independence test for detecting hidden confounders, where the equality is tested across all values $V_i=v_i$ and $V_j=v_j$. Note that it is not required to test the equality across all values of $O$. The equality is tested with a constant value of $O$.




\thispagestyle{plain} 
\mbox{}


\chapter{Active Learning of Causal Graphs} 

\label{ch:Chapter4} 


In this chapter, we introduce a generic method for active learning of causal graphs from interventions. 
The main features of the algorithm are as follows:

\begin{itemize}
\item It allows causal models that contain hidden confounders, and identifies both the edges among observable variables
and the pairs of variables affected by a common hidden confounder. 
Most previous research regarding learning of causal graphs did not consider for the presence of hidden confounders, whereas real world causal systems and causal data typically contain hidden confounders. 

\item It accepts any arbitrary set of candidate graphs as representations of previous knowledge, without any requisite on restrictions or forms. 
The number of interventions performed by the algorithm is $O(|G|)$. In most previous work it is difficult to assess how the amount of previous knowledge
about the domain translates to a reduction in the number of interventions. For example, eliminating some of the candidate graphs with expert knowledge or running a causal discovery algorithm from observational data, before running our active learning algorithm, 
may reduce the number of candidates and reduce the execution time of our algorithm, but is not a requirement. Most existing algorithms are not designed to
take advantage of this pre-processing option. 

\item Our algorithm can accommodate various functions describing the cost of an intervention, unlike others that 
are designed for one particular cost (e.g., unit cost where all interventions have the same cost). Our method considers that the cost of intervening variables may differ for every variable, or combination of variables, and for every value assignment of the intervened variables. Also, the cost of observing may differ for every variable, or combination of variables. We aim at minimizing the total cost of the sequence of interventions and observations. To the best of our knowledge, our method is the first that encompasses all these cost dimensions, and is general in that sense.

\item Finally, all existing algorithms for the discrete case perform tests that implicitly require testing all values of the intervened variables
(i.e., $P(Y|do(X))$, where $X$ is implicitly intervened or sampled with all of its domain values). Our algorithm has a pre-processing phase
that eliminates as many candidates as possible using single-valued interventions, i.e., $P(Y|do(X=x))$ for a single assignment of values to 
the variables in $X$. We claim that the number of values tested by an intervention is, in many practical cases, 
a leading factor in the cost of an intervention, and therefore this approach can drastically reduce the cost of the discovery process. We argue that by using single valued interventions it is possible, in the vast majority of cases, to distinguish two candidates with differences in their edge structure, and also candidates that have differences in their hidden confounders.
To the best of our knowledge our method is the first method that uses interventions on a single value of the intervened variables. 

\end{itemize}

Our approach is combinatorial rather than statistical, in that we consider that the answers to 
our interventions are retrieved exactly and we do not study the number of individual samples
required to know the result of an intervention with a given precision. In fact, 
in most cases the primitive operation is simply checking for equality or inequality 
among two probability distributions. 


We start with a number of causal graph candidates that may be the output of other previous analysis using observational data and expert knowledge. This is the most likely general setting, as interventions are usually more expensive than observations, so we consider that everything is done to first use previous expert knowledge and use observational distributions and algorithms based on that data, before starting any experiment design.

We take a two-phase approach. First we use causal effect predictions to discard as many causal graphs candidates as possible. This is done using interventions with a single value of the intervened variables $X$, which avoids the cost of repeating the interventions with multiple values of $X$. Using do-calculus as a predictive mechanism ensures that a causal effect prediction is either correct or not possible, due to the soundness and completeness of do-calculus. The first phase concludes when there are no further single valued interventions capable of discarding any more candidate graphs, which means the remaining graphs are not distinguishable per Definition\ref{def:distinguishability}.

The second phase uses conditional independence tests, in order to discard the remaining causal graph candidates. We provide the graphical conditions that explain which conditional independence tests are required. The advantage of eliminating as many candidates as possible in the first phase, with selected interventions using single values of the intervened variables $X$, before applying the second phase, reduces the overall cost of the process.

\section{Algorithm for Active Learning of Causal Graphs}

This section introduces the {\em ALCAM} algorithm for active learning of causal graphs. The algorithm is based on the iterative application of two main functions: the function {\em SelectIntervention} is given in Figure~\ref{fig:select-int}, and the function {\em SelectGraphs}  is given in Figure~\ref{fig:select-causal}. The algorithm uses a causal effect {\em Predictor} based on do-calculus, given in Figure~\ref{fig:predictor} and an intervention {\em Oracle} given in Figure~\ref{fig:oracle}. Additionally, the function {\em PowerOfIntervention} given in Figure~\ref{fig:countdv} provides the metric required for distinguishing causal graphs. The {\em ALCAM} algorithm for active learning of causal graphs is given in Algorithm~\ref{alg:active-algo}.

\subsection{Power of Intervention function}
Given an intervention and a set of causal graph candidates, the {\em PowerOfIntervention} function (Figure~\ref{fig:countdv}) counts the number of pairs of candidates that are distinguishable (Definition~\ref{def:distinguishability}) with the intervention, based on all case scenarios (see Table ~\ref{tab:intervention_evaluation}).

\begin{figure}[H]
\hrule\medskip
Function \textbf{PowerOfIntervention}($E$,$\cal G$,$P^\star$,$P_{k}$)

INPUT: 
\begin{itemize}
\item $E$: intervention $(X,x,Y)$ where $X, Y \subset V$, $Y\cap X = \emptyset$ and $x$ is a value assignment for $X$

\item $\cal G =$ $\{G_1, G_2,... G_N\}$: set of $N$ causal models over a set $V$ of observed variables and a set $U$ of unobserved variables (hidden confounders)

\item $P^\star$: probability distribution of the observed variables in the true causal model $M^\star$, without interventions
\item $P_{k}$: set of causal effects $P_{k}(Y|do(X=x))$ for all $G_k$ in $\cal G$
\end{itemize}

OUTPUT:
\begin{itemize}
\item {\em PI}: number of pairs of causal graphs in $\cal G$ that are distinguishable with intervention $E$
\end{itemize}

\begin{enumerate}

\item let PI be the number of pairs of graphs $G_k$, $G_l$ in $\cal G$ for which $G_k \not\approx_{P_E} G_l$;

\item return {\em PI};
\end{enumerate}

\caption{The PowerOfIntervention function}
\label{fig:countdv}
\medskip\hrule
\end{figure}


\subsection{Select Interventions}

Intervening and observing variables has a cost. We represent the cost of intervening or observing variables with cost functions $C_I(X=x)$ and $C_O(Y)$ respectively. Our aim is to find the sequence of interventions which eliminates the maximal number of candidate graphs at the lowest possible cost.

Given a set of candidate causal graphs $\mathcal {G}$, some graphs are distinguishable from others in the set, and some graphs are not distinguishable. We provide an example in Figure~\ref{fig:SelectInterventionTables}, which shows a set of candidate causal graphs $\mathcal {G} =\{G_1, G_2, G_3, G_4, G_5\}$ and a method to find interventions that maximize the distinguishability among the graphs at the lowest possible cost. 

In step {\em a} we find all maximal non-distinguishable subsets of graphs $\hat{\mathcal {G}} = \{\mathcal {G}_1, \mathcal {G}_2...\}$ in $\mathcal {G}$ so that for every subset $\mathcal {G}_r \in \hat{\mathcal {G}}$, $PI(E,\mathcal {G}_r,$$P)=0$ for all $E$, and for every pair of different subsets $\mathcal {G}_r, \mathcal {G}_s \in \hat{\mathcal {G}}$, 
there exist interventions $E$ for which $PI(E,\mathcal {G}_r \cup \mathcal {G}_s,P) > 0$. This means no interventions exist that distinguish pairs of graphs within a subset, but there exist interventions that distinguish graphs from different subsets. Note that a graph may belong to several subsets if it is not distinguishable from the other graphs in these subsets. In the example, $G_2$ is not distinguishable from $G_1$ and $G_3$. Interventions $E_3$ and $E_5$ make $G_1$ and $G_3$ distinguishable from each other, but not $G_2$. This is due to $P_1(Y|do(X=x)) \neq P_3(Y|do(X=x))$, however $P_2(Y|do(X=x)) = \emptyset$, $P_1(Y|do(X=x)) \neq P(Y)$ and $P_3(Y|do(X=x)) \neq P(Y)$.


In step {\em b}, we find all minimal sets of interventions $\hat {\mathcal {E}} = \{\mathcal {E}_1,\mathcal {E}_2...\}$ that split all the subsets of non-distinguishable graphs in $\hat{\mathcal {G}}$. In the example, the sets of interventions $\mathcal {E}_1 = \{E_1,E_2,E_5\}$, $\mathcal {E}_2 = \{E_1,E_3,E_4\}$... split all subsets of non-distinguishable graphs. If the set of candidate graphs includes the {\em true} graph, then any of these sets of interventions allows us to eliminate all subsets of graphs other than the subset where the {\em true} graph is. Also, no further intervention will be able to eliminate any additional graphs.

In step {\em c}, we select the set of interventions $\mathcal {E}_t$ with the smallest total cost. This means $\mathcal {E}_t$ has the smallest cost among all sets of interventions that are able to eliminate all candidate graphs except the subset of non-distinguishable graphs where the {\em true} graph is.

At this stage, we can apply a strategy to select the order for the interventions in $\mathcal {E}_t$. In the example, the smallest cost set of interventions is $\mathcal {E}_2=\{E_1, E_3, E_4\}$. If the {\em true} graph is $G_4$ and we choose intervention $E_3$ first, we will need additional interventions as $E_3$ does not split $G_4$ from the other subsets. However, if we choose $E_4$ first we find the {\em true} graph directly. If the {\em true} graph is $G_1$ we need the three interventions $E_1, E_3, E_4$ in order to eliminate all graphs other than the subset of non-distinguishable graphs where $G_1$ is, and there is no other smaller set of interventions able to do that. From these examples we can see that it is better to prioritize interventions that split single subsets of graphs without requiring the entire set of selected interventions. If the {\em true} graph belongs to these subsets we will avoid the cost of the remaining interventions. Our strategy is to select first the subsets of interventions in $\mathcal {E}_t$ that split single subsets of graphs, starting with the smallest cost ones. With each intervention we eliminate subsets of graphs from the candidate set, so we use an active learning process to adjust the strategy after every intervention.

\begin{figure}[H]
\begin{center}
\includegraphics[width=1.2\textwidth]{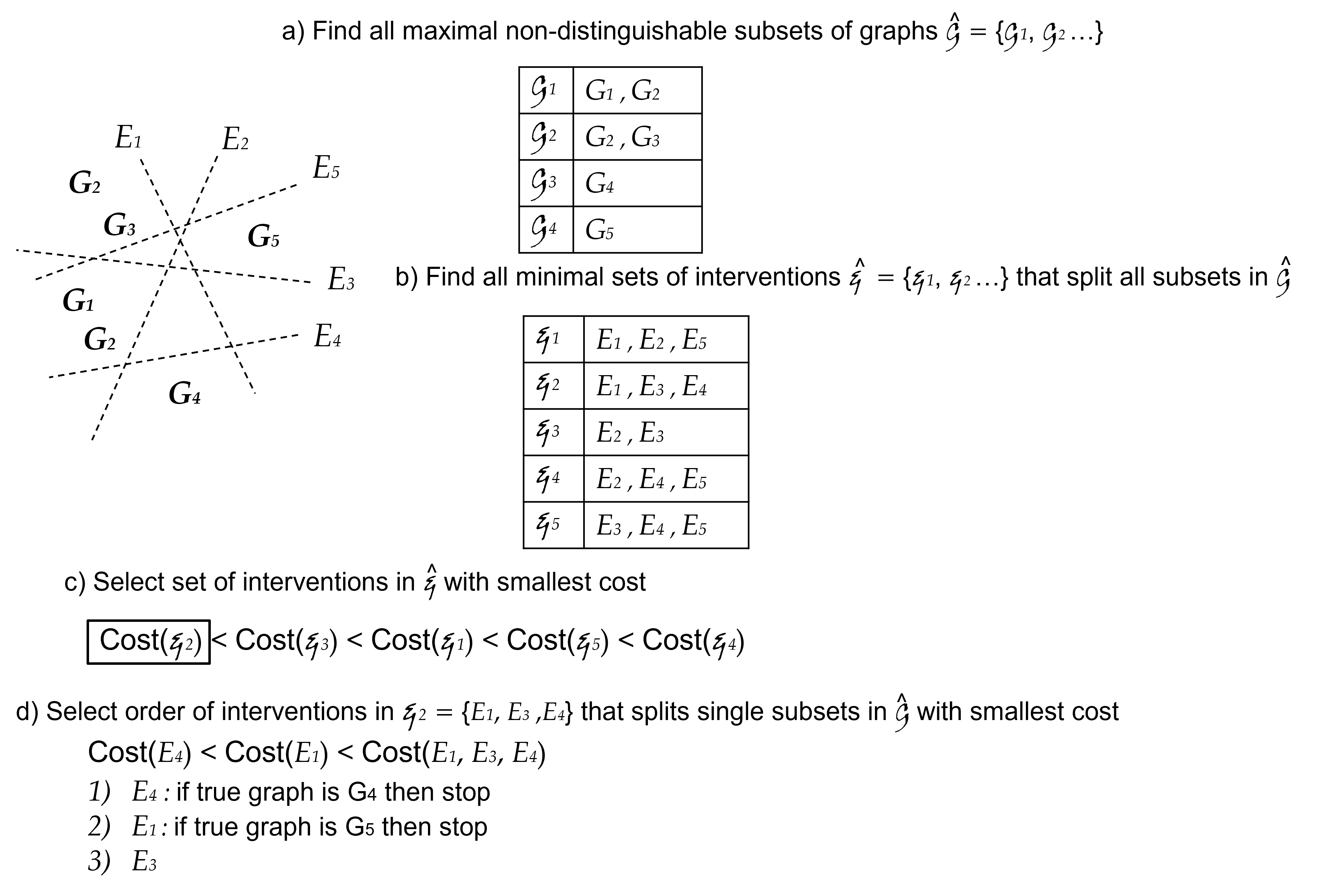}
\end{center}
\caption{Method to find the set of interventions that splits all non-distinguishable subsets of graphs with the smallest cost.}
\label{fig:SelectInterventionTables}
\end{figure}

\begin{figure}[H]
\hrule\medskip
Function \textbf{SelectIntervention}($\mathcal {E}, \hat{\mathcal {G}}, P^\star, P_{k}, C_I, C_O$)

INPUT: 
\begin{itemize}
\item $\mathcal {E}$: set of interventions that splits all subsets of non-distinguishable graphs in $\hat{\cal G}$
\item $\hat{\cal G} = \{\mathcal {G}_1, \mathcal {G}_2...\}$: subsets of non-distinguishable graphs

\item $P^\star$: probability distribution of the observed variables in the true causal model $M^\star$, without interventions
\item $P_{k}$: set of causal effect identifications $P_{k}(Y|do(X=x))$, $\forall G_k \in \cal G$
\item $C_I(X,x)$: cost of intervention $do(X=x)$, $\forall X, x$
\item $C_O(Y)$: cost of observing $Y$, $\forall Y$
\end{itemize}

OUTPUT:
\begin{itemize}
\item $E = (X,x,Y)$: intervention with smallest cost in the subset of interventions $I_{min} \subset \mathcal{E}$: $I_{min}=\{(X_1,x_1,Y_1), (X_2,x_2,Y_2), ... (X_i,x_i,Y_i),..\}$ that splits any subset of graphs in $\hat{\cal G}$ from all other subsets with smallest total cost $\sum_i (C_I(X_i,x_i) + C_O(Y_i))$
\end{itemize}

\begin{enumerate}

\item if {\em PowerOfIntervention}($E,\hat{\cal G},P^\star,P_{k})=0$ for all $E \in \mathcal {E}$ then return NA;

 \item let $\hat{I} = \{I_1, I_2, ...\}$ be all subsets of $\mathcal {E}$ that split any subset of graphs  in $\hat{\cal G}$ from all other subsets, i.e., for every $\mathcal {G}_i \in \hat{\cal G}$ there exists $I_j \in \hat{I}$ that includes interventions $E$ for which {\em PowerOfIntervention}$(E,\mathcal {G}_i \cup \mathcal {G}_s,P^\star,P_{k})>0$ for all ${G}_s \in \hat{\cal G}$, ${G}_s \neq {G}_i$;

\item let  $I_{min}$ be the set in $\hat{I}$ with smallest total cost $\sum_i (C_I(X_i,x_i) + C_O(Y_i))$;

\item let $E = (X,x,Y)$ be the intervention with smallest cost in $I_{min}$;

\item return $E$;

\end{enumerate}

\caption{The SelectIntervention function}
\label{fig:select-int}
\medskip\hrule
\end{figure}

\subsection{Select Graphs}

The {\em SelectGraphs} function selects the subsets of non-distinguishable graphs for which the causal effect from an intervention $E=(X,x,Y)$ is the same as the effect in the {\em true} graph $G^\star$, as well as subsets for which the effect is unknown (empty set) while the effect in $G^\star$ is not $P^{\star}(Y)$.

From the three scenarios of distinguishability from Definition~\ref{def:distinguishability_int}, and the soundness and completeness of do-calculus, we can see that if $P_k(Y|do(X=x)) \neq \emptyset$ then we select $G_k$ only if $P_k(Y|do(X=x))$ is the same joint distribution that the joint distribution $P^{\star}(Y|do(X=x))$ in $M^{\star}$. If $P_k(Y|do(X=x)) = \emptyset$ then we select $G_k$ only if
the effect of the intervention in $M^{\star}$ is not $P^{\star}(Y)$. The graphs for which the effect is unknown have a {\em hedge} for $X$,$Y$ \parencite{shpitser2006identification} so the effect in $G^\star$ is not $P^{\star}(Y)$ and we cannot eliminate these graphs.

\begin{figure}[H]
\hrule\medskip
Function \textbf{SelectGraphs}($\hat{\cal G}$,$E$,$P^\star(Y,do(X=x))$,$P_{k}$)

INPUT: 
\begin{itemize}
\item $\hat{\cal G} = \{\mathcal {G}_1, \mathcal {G}_2...\}$: subsets of non-distinguishable graphs
\item $E$: intervention $(X,x,Y)$ where $X, Y \subset V$, $Y\cap X = \emptyset$ and $x$ is a value assignment for $X$
\item $P^\star(Y|do(X=x))$: joint probability distribution of $Y$ upon performing the intervention $X=x$ on the {\em true} model $M^\star$

\item $P_{k}$: set of causal effect identifications $P_{k}(Y|do(X=x))$, $\forall G_k \in \cal G$

\end{itemize}

OUTPUT:
\begin{itemize}
\item $\hat{\cal G}' \subset \hat{\cal G}$: subsets of non-distinguishable graphs in $\hat{\cal G}$ for which the causal effect from the intervention $E$ is the same as the effect in the {\em true} graph $G^\star$, as well as subsets for which the effect is unknown (empty set) while the effect in $G^\star$ is not $P^{\star}(Y)$. 
\end{itemize}

\begin{enumerate}

\item let $\hat{\cal G}'$ be all subsets of non-distinguishable graphs in $\hat{\cal G}$ for which, $\forall G_k \in \hat{\cal G}'$:
\begin{itemize}
\item $P_{k}(Y|do(X=x)) \neq \emptyset$ and $P_{k}(Y|do(X=x)) = P^\star(Y|do(X=x))$;

or
\item $P_{k}(Y|do(X=x)) = \emptyset$ and $P^\star(Y|do(X=x)) \neq P^\star(Y)$;

\end{itemize}

\item return $\hat{\cal G}'$;

\end{enumerate}

\caption{The SelectGraphs function}
\label{fig:select-causal}
\medskip\hrule
\end{figure}

\subsection{ALCAM algorithm}

The {\em ALCAM} algorithm for active learning of causal graphs is shown in Algorithm~\ref{alg:active-algo}. The algorithm starts with a set $\cal G$ of candidate causal graphs compatible with any previously available information about the $true$ graph. 

Using the predicted causal effects obtained by the {\em Predictor} function given in Figure~\ref{fig:predictor}, and the power of interventions obtained by the {\em PowerOfIntervention} function given in Figure~\ref{fig:countdv},
the algorithm finds all the subsets of non-distinguishable graphs in $\cal G$. Then it finds the set of interventions that splits $\cal G$ into these subsets with the smallest cost. {\em ALCAM} then selects the next intervention by assigning an order of priority to this set, with {\em SelectIntervention}. Subsets of interventions that split one subset of non-distinguishable graphs from the rest of graphs at the lowest cost are done first. {\em ALCAM} calls the {\em Oracle} given in Figure~\ref{fig:oracle} to obtain the interventional probability distribution of the {\em true} model from the selected interventions, and eliminates from $\cal G$ all subsets of non-distinguishable graphs where the causal effect differs from the interventional probability distribution of the {\em true} model, with {\em SelectGraphs}.

This sequence of three steps {\em SelectIntervention}, {\em Oracle} and {\em SelectGraphs} is done iteratively until no interventions exist with positive {\em PowerOfIntervention} in $\cal G$. The {\em SelectIntervention} function returns {\em NA} and at that stage the selected subset of causal graphs is the one for which the causal effects of each intervention, at each iteration of the algorithm, are the same than what the {\em Oracle} returns as effects in the {\em true} graph. At that stage if there is only one graph left in $\cal G$, then we have identified the {\em true} graph. If there are several candidate graphs left, due to the presence in $\cal G$ of graphs that are non-distinguishable from the {\em true} graph, then we test the edges and hidden confounders that differ among the remaining candidates, to confirm or refute their presence, using conditional independence testing (functions {\em idEdges} shown in Figure~\ref{fig:idEdges} and {\em idHidden}, shown in Figure~\ref{fig:idHidden}). Note that the interventions have reduced the conditional independence testing to the minimal possible set of graphs, edges and hidden confounders. Also, in some sub-cases conditional independence testing is not required at all. 

\begin{lemma}
\label{lem:ALCAM_Distinguish}
The subset of graphs in $\hat{\mathcal {G}}$ that ALCAM finds in line 5 consists of all the graphs that are non-distinguishable from the true graph. All other graphs are removed from $\hat{\mathcal {G}}$.
\end{lemma}

\begin{proof}
$\hat{\mathcal {E}}_t=\{E_1, E_2, ..., E_n\}$ is a set of interventions for which
{\em PowerOfIntervention}$(E_i,\mathcal {G}_r \cup \mathcal {G}_s,P^\star,P_{k})>0$ for every pair of subsets of non-distinguishable graphs $\mathcal {G}_r$, $\mathcal {G}_s$ in $\hat{\mathcal {G}}$. Line 5a selects one intervention in $\hat{\mathcal {E}}_t$ and line 5b queries the Oracle for its effect on the {\em true} model. Line 5c removes from $\hat{\mathcal {G}}$ all subsets of graphs that are distinguishable from the {\em true} graph with the intervention and line 5d removes the intervention from the set $\hat{\mathcal {E}}_t$. At each iteration of the loop, line 5 selects another intervention in $\hat{\mathcal {E}}_t$ until only one subset of non-distinguishable graphs remain in  $\hat{\mathcal {G}}$, which is the subset of graphs that are non-distinguishable from the true graph. 
\qed\end{proof}

\begin{lemma}
\label{lem:ALCAM_Distinguish_n_int}
The maximal number of interventions in line 5 is at most $|\cal G|$ - $|{\cal G}_q|$, where $\cal G$ is the set of candidate graphs, and ${\cal G}_q \in \hat{\mathcal {G}}$ is the subset of candidate graphs that are non-distinguishable from the true graph.
\end{lemma}

\begin{proof}
Assume the worst case, where line 5 removes from $\hat{\mathcal {G}}$ only one subset of graphs consisting on one graph only, at each recurrence.
\qed\end{proof}

\begin{lemma}
\label{lem:ALCAM_Graph}
The subset of graphs in $\hat{\mathcal {G}}$ that ALCAM finds in line 5 have:
\begin{enumerate}

\item the same ancestral relations than the true graph;

\item the same edges than the true graph, except graphs that have edges $ed(V_i,V_j)$ that differ among candidates and have a hedge for $P(V_j|block(V_i,V_j))$;

\item the same hidden confounders than the true graph, except graphs that have hidden confounders $hc(V_i,V_j)$ that differ among the candidates and either have a hedge for $P(V_j|do(V_i))$ where $V_i \in An(V_j)$, or have $V_i \notin An(V_j)$ and $V_j \notin An(V_i)$;

\end{enumerate}
\end{lemma}

\begin{proof}
By lemma~\ref{lem:ALCAM_Distinguish} the subset of graphs that ALCAM finds in line 5 consists of all the graphs non-distinguishable from the true graph. Lemma~\ref{lem:Dist_Graphs} provides the three conditions under which two causal graphs are distinguishable.

\qed\end{proof}

\begin{lemma}
\label{lem:ALCAM_Edges}
If there exist edge differences among the graphs in $\hat{\mathcal {G}}$ in line 8, ALCAM uses at most $Q$ conditional independence tests to eliminate from $\hat{\mathcal {G}}$ the graphs that have different edges than the {\em true} graph, where $Q$ is the number of edges $ed(V_i,V_j)$ that differ and for which there is a hedge for $P(V_j|do(block(V_i,V_j)))$ in one or more of the graphs.
\end{lemma}

\begin{proof}
If the true graph is identified in line 6, ALCAM ends at line 7. If there are no edge differences among graphs in $\hat{\mathcal {G}}$, ALCAM does not execute function $idEdges(\hat{\mathcal {G}})$ in line 8 and moves onto line 9. Finally, if there are edge differences among graphs in $\hat{\mathcal {G}}$, ALCAM executes function $idEdges(\hat{\mathcal {G}})$ and performs conditional independence tests on the edges that differ from the {\em true} graph in the subset of graphs found in line 5 as specified in lemma~\ref{lem:DistinguishEdges}.

\qed\end{proof}

\begin{lemma}
\label{lem:ALCAM_HC}
If there exist hidden confounder differences among the graphs in $\hat{\mathcal {G}}$ in line 9, ALCAM uses at most $H$ conditional independence tests to eliminate from $\hat{\mathcal {G}}$ the graphs that have different hidden confounders than the true graph, where $H$ is the number of hidden confounders $hc(V_i,V_j)$ that differ for which there is a hedge for $P(V_j|do(V_i))$ in one or more of the graphs, or for which $V_i \notin An(V_j)$ and $V_j \notin An(V_i)$.
\end{lemma}

\begin{proof}
If there are no hidden confounder differences among graphs $\hat{\mathcal {G}}$ in line 9, ALCAM does not execute the function $idHidden(\hat{\mathcal {G}})$ and ends at line 10. If there are hidden confounder differences among graphs in $\hat{\mathcal {G}}$, ALCAM executes the function $idHidden(\hat{\mathcal {G}})$ and performs conditional independence tests on the hidden confounders that differ from the {\em true} graph in the subset of graphs found in line 5, as specified in lemma~\ref{lem:DistinguishHC}
\qed\end{proof}

\begin{theorem}[ALCAM is Sound and Complete]
\label{thm:alcam}
ALCAM always returns the true causal graph $G^\star$. 
\end{theorem}

\begin{proof} 
Let $G^\star$ be the {\em true} causal graph. Let $\mathcal {G}$ be the set of candidate causal graphs and $\hat{\mathcal {G}}$ the subsets of non-distinguishable graphs in $\mathcal {G}$.
 By lemma~\ref{lem:ALCAM_Distinguish}
the subset of graphs in $\hat{\mathcal {G}}$ that ALCAM finds in line 5 consists of all the graphs that have the same ancestral relations than the {\em true} graph, the same edges than the {\em true} graph, and the same hidden confounders than the {\em true} graph, except graphs that have edges or hidden confounders under the conditions of lemma~\ref{lem:ALCAM_Distinguish}. If the exceptions do not exist, then the subset in $\hat{\mathcal {G}}$ consists of one graph only, ALCAM exits at line 7, and $\hat{\mathcal {G}}=\{G^{\star}\}$. If the exceptions do exist for edges, then $\hat{\mathcal {G}}$ contains graphs with edge differences, and ALCAM line 8 executes. By lemma~\ref{lem:DistinguishEdges} ALCAM line 8 eliminates from $\hat{\mathcal {G}}$ all graphs that have different edges than the {\em true} graph. If the exceptions do exist for hidden confounders, then $\hat{\mathcal {G}}$ contains graphs with hidden confounder differences, and ALCAM line 9 executes. By lemma~\ref{lem:DistinguishHC} ALCAM line 9 eliminates from $\hat{\mathcal {G}}$ all graphs that have different hidden confounders than the {\em true} graph, ALCAM exits at line 10, and $\hat{\mathcal {G}}=\{G^{\star}\}$, which completes the proof.

\qed\end{proof}

\begin{figure}[h]
\hrule\medskip
Function \textbf{Predictor}($E,G_k,P^\star)$

INPUT: 
\begin{itemize}
\item $E$: intervention $(X,x,Y)$ where $X, Y \subset V$, $Y\cap X = \emptyset$ and $x$ is a value assignment for $X$
\item $G_k$: causal graph
\item $P^\star$: joint probability distribution over $V$
\end{itemize}

OUTPUT:
\begin{itemize}
\item $P_{k}(Y|do(X=x))$: set of causal effects in $Y$ for intervention $do(X=x)$ from $G_k$ 
\end{itemize}

\begin{enumerate}
\item if there is a hedge for $E$ in $G_k$ then let $P_k = \emptyset$;

\item else, let $P_{k}(Y|do(X=x))$ be all the different evaluations from $P^\star$ of all do-free expressions that can be obtained by repeated application of do-calculus rules, together with standard probability manipulations from $G_k$;

\item return $P_{k}$;

\end{enumerate}

\caption{Predictor based on do-calculus}
\label{fig:predictor}
\medskip\hrule
\end{figure}

\begin{algorithm}
\caption{\textbf{ALCAM}$(\mathcal{G},P^\star,C_I,C_O)$ Algorithm for active learning of causal graphs}
\label{alg:active-algo}
\begin{algorithmic} 
\STATE 

INPUT: 
\begin{itemize}
\item $\cal G$: set of causal graphs over a set $V$ of observed variables and a set $U$ of unobserved variables (hidden confounders)
\item $P^\star$: probability distribution of the observed variables in the true causal model $M^\star$, without interventions
\item $C_I(X,x)$: cost of intervention $do(X=x), \forall X,x$
\item $C_O(Y)$: cost of observing $Y, \forall Y$
\end{itemize}

OUTPUT:
\begin{itemize}
\item $G_{id}\in \cal G$: causal graph for which all interventional distributions provided by the {\em Oracle} are equal to the distributions from the {\em Predictor}, and all the edges and hidden confounders that are not identifiable with interventions are identified with conditional independence tests
\end{itemize}

\begin{enumerate}
\item let $P_{k}=Predictor(E,G_k,P^\star)$ for all $G_k \in \mathcal {G}$, for all $E=(X,x,Y)$;

\item let $\hat{\mathcal {G}} = \{ \mathcal {G}_1, \mathcal {G}_2...\}$ be all maximal subsets of non-distinguishable graphs in $\mathcal {G}$, i.e:

\begin{itemize} 
\item {\em PowerOfIntervention}$(E,\mathcal {G}_r,P^\star,P_{k})=0$, for all $\mathcal {G}_r \in \hat{\mathcal {G}}$ and for all $E=(X,x,Y)$;
\item there exists $E=(X,x,Y)$ for every pair of different subsets $\mathcal {G}_r $, $\mathcal {G}_s \in \hat{\mathcal {G}}$:
{\em PowerOfIntervention}$(E,\mathcal {G}_r \cup \mathcal {G}_s,P^\star,P_{k})>0$;
\end{itemize} 

\item let $\hat{\mathcal {E}} = \{\mathcal {E}_1, \mathcal {E}_2...\}$ be all minimal sets of interventions that split all subsets of non-distinguishable graphs in $\hat{\mathcal {G}}$, i.e for every pair of different subsets $\mathcal {G}_r, \mathcal {G}_s \in \hat{\mathcal {G}}$, all $\mathcal {E}_t$ include an intervention $E=(X,x,Y)$ for which
{\em PowerOfIntervention}$(E,\mathcal {G}_r \cup \mathcal {G}_s,P^\star,P_{k})>0$;

\item let $\mathcal {E}_t=\{E_1, E_2, ..., E_n\}$ be the set of interventions in $\hat{\mathcal {E}}$ with smallest total cost $\sum\limits_{1\leq i\leq n} (C_I(X_i,x_i) + C_O(Y_i))$;

\item While $|\hat{\mathcal {G}}|>1$

\begin{enumerate}
\item let $E_i$ = {\em SelectIntervention}$(\mathcal {E}_t, \hat{\mathcal {G}}, P^\star, P_{k}, C_I, C_O)$;

\item let $P^\star(Y|do(X=x)) = Oracle(E_i)$; 

\item let $\hat{\mathcal {G}} = SelectGraphs(\hat{\mathcal {G}},E_i,P^\star(Y|do(X=x)),P_{k})$;

\item let $\mathcal {E}_t = \mathcal {E}_t \setminus E_i$;
\end{enumerate}

\item let $\hat{\mathcal {G}}$ be the set of graphs in the single subset in $\hat{\mathcal {G}}$

\item If $\hat{\mathcal {G}}$ contains 1 graph, then return $\hat{\mathcal {G}}$;

\item If there exist edge differences in graphs in $\hat{\mathcal {G}}$, then let $\hat{\mathcal {G}} = idEdges(\hat{\mathcal {G}})$

\item If there exist hidden confounder differences in graphs in $\hat{\mathcal {G}}$, then let $\hat{\mathcal {G}} = idHidden(\hat{\mathcal {G}})$

\item return $\hat{\mathcal {G}}$;

\end{enumerate}

\end{algorithmic}
\end{algorithm}

\begin{figure}[h]
\hrule\medskip
Function \textbf{Oracle}($E$)

INPUT: 
\begin{itemize}
\item $E$: intervention $(X,x,Y)$ where $X, Y \subset V$, $Y\cap X = \emptyset$ and $x$ is a value assignment for $X$

\end{itemize}

OUTPUT:
\begin{itemize}
\item $P^\star(Y|do(X=x))$: joint probability of $Y$ upon performing the intervention $do(X=x)$ on the {\em true} causal model $M^\star$
\end{itemize}

\caption{The Intervention Oracle}
\label{fig:oracle}
\medskip\hrule
\end{figure}

\begin{figure}[h]
\hrule\medskip
Function \textbf{idEdges}($\hat{\mathcal {G}}$)

INPUT: 
\begin{itemize}
\item $\hat{\mathcal {G}}$: set of graphs with the same ancestral relations than the {\em true} graph and with edge differences
\end{itemize}

OUTPUT:
\begin{itemize}
\item $\hat{\mathcal {G}}' \subset \hat{\mathcal {G}}$: set of graphs with the same edges than the {\em true} graph
\end{itemize}

\begin{enumerate}
\item for each edge $ed(V_i,V_j)$ appearing only in a subset $\hat{\mathcal {G}}_{ed}$ of $\hat{\mathcal {G}}$:

\begin{enumerate}
\item 
On the {\em true} causal model $M^\star$, perform intervention on the minimal set of variables $dsep(V_i,V_j)$ whose intervention d-separates $V_i$ and $V_j$ in all graphs in $\hat{\mathcal {G}} \setminus \hat{\mathcal {G}}_{ed}$, and do a conditional independence test between $V_i$ and $V_j$. If $V_i$ and $V_j$ are dependent then $\hat{\mathcal {G}} = \hat{\mathcal {G}}_{ed}$;
\item else, $\hat{\mathcal {G}} = \hat{\mathcal {G}} \setminus \hat{\mathcal {G}}_{ed}$;

\end{enumerate}

\end{enumerate}

\caption{Edge Identification Function}
\label{fig:idEdges}
\medskip\hrule
\end{figure}

\begin{figure}[h]
\hrule\medskip
Function \textbf{idHidden}($\hat{\mathcal {G}}$)

INPUT: 
\begin{itemize}
\item $\hat{\mathcal {G}}$: subset of non-distinguishable graphs with hidden confounder differences
\end{itemize}

OUTPUT:
\begin{itemize}
\item $\hat{\mathcal {G}}' \subset \hat{\mathcal {G}}$: subset of graphs with the same hidden confounders than the {\em true} causal model $M^\star$
\end{itemize}

\begin{enumerate}
\item for each hidden confounder $(V_i,V_j)$ appearing only in a subset $\hat{\mathcal {G}}_{hc}$ of $\hat{\mathcal {G}}$ do the following conditional independence tests between $V_i$ and $V_j$ on the {\em true} causal model $M^\star$:

    \begin{enumerate}
    \item if $V_i$ and $V_j$ are adjacent and $P(V_j|do(V_i,O)) = P(V_j|V_i,do(O))$
    \item if $V_i$ and $V_j$ are non-adjacent and $P(V_j|do(O)) = P(V_j|V_i,do(O))$
    \end{enumerate}

where $O=Parents(V_i) \cup Parents(V_j)$

\item then $\hat{\mathcal {G}} = \hat{\mathcal {G}} \setminus \hat{\mathcal {G}}_{hc}$;
\item else, $\hat{\mathcal {G}} = \hat{\mathcal {G}}_{hc}$;

\item return $\hat{\mathcal {G}}$;

\end{enumerate}

\caption{Hidden Confounder Identification Function}
\label{fig:idHidden}
\medskip\hrule
\end{figure}

\begin{theorem}[ALCAM number of interventions]
\label{thm:alcam_cost}
ALCAM requires at most $|\cal G|$ - $|{\cal G}_q|$ interventions on a single value of the intervened variables, and $Q + H$ conditional independence tests, where

\begin{itemize}

\item $\cal G$ is the set of candidate graphs;

\item ${\cal G}_q$ is the subset of candidate graphs that are non-distinguishable from the true graph;

\item $Q$ is the number of edges $ed(V_i,V_j)$ that differ among the graphs in ${\cal G}_q$ and for which there is a hedge for $P(V_j|do(block(V_i,V_j)))$ in one or more of the graphs;

\item $H$ is the number of hidden confounders $hc(V_i,V_j)$ that differ among the graphs in ${\cal G}_q$ and for which there is a hedge for $P(V_j|do(V_i))$ in one or more of the graphs, or for which $V_i \notin An(V_j)$ and $V_j \notin An(V_i)$;

\end{itemize}

\end{theorem}

\begin{proof}
By Lemma~\ref{lem:ALCAM_Distinguish_n_int}, Lemma~\ref{lem:ALCAM_Edges}, and Lemma~\ref{lem:ALCAM_HC}.

\qed\end{proof}







\chapter{Causality and Time} 

\label{ch:Chapter5} 


\section{Dynamic Causal Networks}

The generic definition of causal models (Definition~\ref{def:causalnetwork}) leaves the functions $f_k$ unspecified. These functions can take any suitable form that best describes the causal dependencies between variables in the model. In natural phenomena some variables may be time independent while others may evolve over time. However rarely does Pearl specifically treat the case of dynamic variables.

The definition of Dynamic Causal Network \parencite{Blondel2017} is an extension of Pearl's causal model definition, by specifying that the variables are sampled over time, as in \parencite{valdes2011effective}.

\begin{definition}[Dynamic Causal Network] 
\label{def:dynamiccausalnetwork}
A dynamic causal network $D$ is a causal model in which the set $F$ of functions is such that $V_{k,t} = f_k(Pa(V_{k,t}),U_{k,t-\alpha})$; where $V_{k,t}$ is the variable associated with the time sampling $t$ of the observed process $V_k$; $U_{k,t-\alpha}$ is the variable associated with the time sampling $t-\alpha$ of the unobserved process $U_k$; $t$ and $\alpha$ are discrete values of time.
\end{definition}

Note that $Pa(V_{k,t})$ may include variables in any time sampling previous to $t$ up to and including $t$, depending on the delays of the direct causal dependencies between processes in comparison with the sampling rate. $U_{k,t-\alpha}$ may be generated by a noise process or by a hidden confounder. In the case of noise, we assume that all noise processes $U_{k}$ are independent of each other, and that their influence to the observed variables happens without delay, so that $\alpha=0$. In the case of hidden confounders, we assume $\alpha\geq 0$ as causes precede their effects.

\begin{figure}
\begin{center}
\includegraphics[width=2cm]{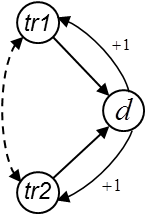}
\end{center}
\caption{Dynamic Causal Network where $tr1$ and $tr2$ have a common unobserved cause, a {\em hidden confounder}. Since both variables are in the same time slice, we call it a {\em static hidden confounder}.}
\label{fig:dcn_confounder_compact}
\end{figure}

To represent hidden confounders in DCN, we extend to the dynamic context the framework developed in \parencite{pearl1991theory} on causal model equivalence and latent structure projections. Let us consider the projection algorithm \parencite{verma1993graphical}, which takes a causal model with unobserved variables and finds an equivalent model (with the same set of causal dependencies), called a "dependency-equivalent projection", but with no links between unobserved variables and where every unobserved variable is a parent of exactly two observed variables. 

The projection algorithm in DCN works as follows. For each pair $(V_{m},V_{n})$ of observed processes, if there is a directed path from $V_{m,t}$ to $V_{n,t+\alpha}$ through unobserved processes then we assign a directed edge from $V_{m,t}$ to $V_{n,t+\alpha}$; however, if there is a divergent path between them through unobserved processes then we assign a bidirected edge, representing a hidden confounder.


\section{Hidden Confounders}

In this thesis, we represent all DCN by their dependency-equivalent projection. Also, we assume the sampling rate to be adjusted to the dynamics of the observed processes. However, both the directed edges and the bidirected edges representing hidden confounders may be crossing several time steps depending on the delay of the causal dependencies in comparison with the sampling rate. We now introduce the concept of static and dynamic hidden confounder.

\subsection{Static Hidden Confounders}

\begin{definition}[Static Hidden Confounder]
\label{def:staticconfounder}
Let $D$ be a DCN. Let $\beta$ be the maximal number of time steps crossed by any of the directed edges in $D$. Let $\alpha$ be the maximal number of time steps crossed by a bidirected edge representing a hidden confounder. If $\alpha\leq\beta$, then the hidden confounder is called "static".
\end{definition}

To give an example of a static confounder, gas consumption and acceleration power at a car are causally influenced by the tire pressure, which may be unmeasured. The gas consumption and acceleration are influenced by the tire pressure at any particular moment in time, and are not impacted by the tire pressure earlier or later.

\subsection{Dynamic Hidden Confounders}

\begin{definition}[Dynamic Hidden Confounder]
\label{def:dynamicconfounder}
Let $D$, $\beta$ and $\alpha$ be as in Definition \ref{def:staticconfounder}. If $\alpha>\beta$, then the hidden confounder is called "dynamic". More specifically, if $\beta<\alpha\leq 2\beta$, we call it "first order" Dynamic Hidden Confounder; if $\alpha>2\beta$, we call it "higher order" Dynamic Hidden Confounder.
\end{definition}

In this thesis, we consider three case scenarios in regards to DCN and their time-invariance properties. If a DCN $D$ contains only static hidden confounders, we can construct a first order Markov process in discrete time, by taking $\beta$ (per Definition \ref{def:staticconfounder}) consecutive time samples of the observed processes $V_k$ in $D$. See Figure \ref{fig:beta_slices1}. This does not mean the DCN generating functions $f_k$ in Definition \ref{def:dynamiccausalnetwork} are time-invariant, but that a first order Markov chain can be built over the observed variables when marginalizing the static confounders over $\beta$ time samples.

\begin{figure}[h]
\begin{center}
\includegraphics[width=0.6\textwidth]{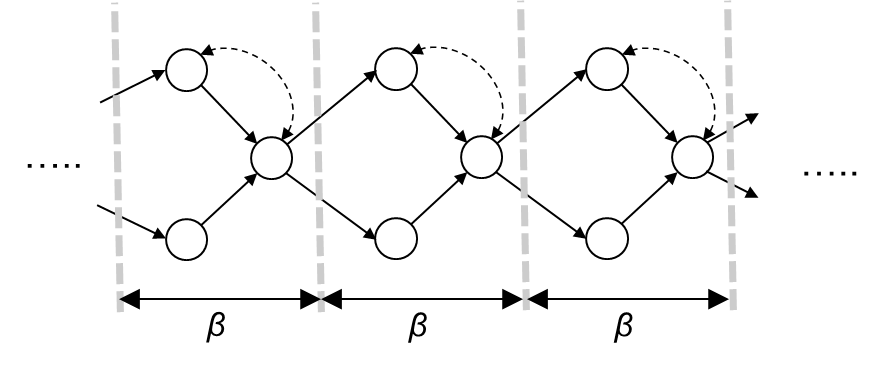}
\end{center}
\caption{Example of static hidden confounder. The time steps crossed by confounders are $\alpha\leq\beta$ and we can construct a first order Markov process in discrete time, using time slices of length $\beta$.}
\label{fig:beta_slices1}
\end{figure}

\begin{figure}[h]
\begin{center}
\includegraphics[width=0.80\textwidth]{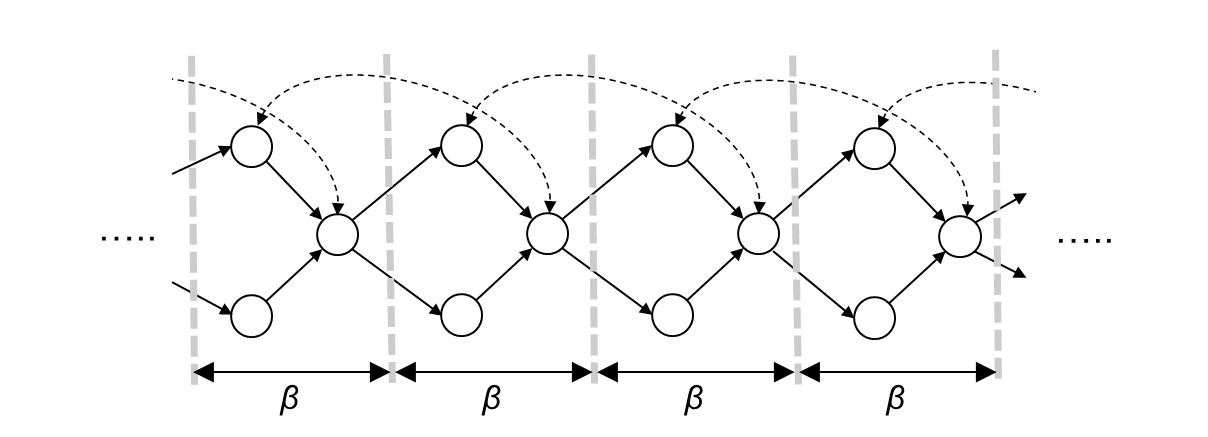}
\end{center}
\caption{Example of first order dynamic hidden confounder. The time steps crossed by confounders are $\beta<\alpha\leq 2\beta$ and we can construct a first order Markov process in discrete time, using time slices of length $\beta$.}
\label{fig:beta_slices2}
\end{figure}

\begin{figure}[h]
\begin{center}
\includegraphics[width=1\textwidth]{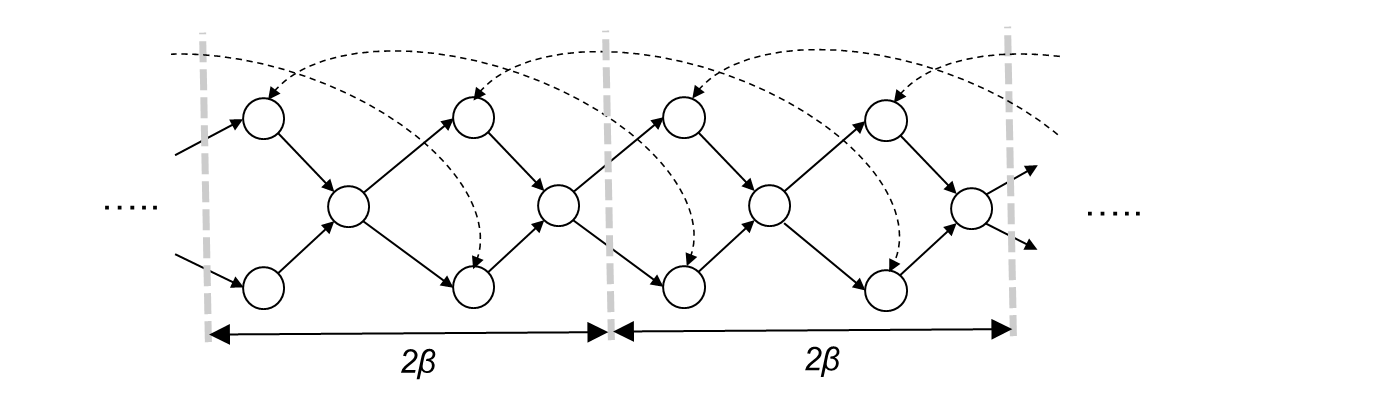}
\end{center}
\caption{Example of higher order dynamic hidden confounder. The time steps crossed by confounders in this example are $2\beta<\alpha\leq 3\beta$ and we can construct a first order Markov process in discrete time, using time slices of length $2\beta$.}
\label{fig:beta_slices3}
\end{figure}

In a second scenario, we consider DCN with first order dynamic hidden confounders. See Figure \ref{fig:beta_slices2}. We can still construct a first order Markov process in discrete time, by taking $\beta$ consecutive time samples. However, we will see in later sections how the effect of interventions on this type of DCN has a different impact than on DCN with static hidden confounders.

Finally, we consider DCN with higher order dynamic hidden confounders, in which case we may construct a first order Markov process in discrete time by taking a multiple of $\beta$ consecutive time samples. See Figure \ref{fig:beta_slices3}.

As we will see in later sections, the difference between these three types of DCN is crucial in the context of identifiability. 
Dynamic hidden confounders cause a time invariant transition matrix to become dynamic after an intervention, e.g., the post-intervention transition matrix will change over time. However, if we perform an intervention on a DCN with static hidden confounders, the network will return to its previous time-invariant behaviour after a transient period. These differences have a great impact on the complexity of the causal identification algorithms that we present. 

Considering that causes precede their effects, the associated graphical representation of a DCN is a DAG. All DCN can be represented as a biinfinite DAG with vertices $V_{k,t}$; edges from $pa(V_{k,t})$ to $V_{k,t}$; and hidden confounders (bi-directed edges). DCN with static hidden confounders and DCN with first order dynamic hidden confounders can be compactly represented as $\beta$ time samples (a multiple of $\beta$ time samples for higher order dynamic hidden confounders) of the observed processes $V_{k,t}$; their corresponding edges and hidden confounders; and some of the directed and bi-directed edges marked with a "+1" label representing the dependencies with the next time slice of the DCN.

Note that a DCN can also be seen as a biinfinite causal model in general, agnostic to whether time is defined or not. As such, removing the time assumption does not make any of the lemmas, theorems or algorithms in this thesis invalid, as they are the result of graphical non-parametric reasoning.


\section{Do-Calculus and Time} 

In the context of DCN, the do-calculus rules, originally identified for the non-temporal context, can be translated under several assumptions. The assumptions rely on the presence of static or dynamic hidden confounders.

DCNs with static hidden confounders contain hidden confounders that impact sets of variables within one time slice only, and contain no hidden confounders between variables at different time slices (see Figure~\ref{fig:dcn_confounder_compact}).

The following two lemmas are based on the application of do-calculus to DCNs with static hidden confounders only. Intuitively, conditioning on the variables that cause time dependent effects d-separates entire parts (future from past) of the DCN (Lemmas \ref{lem:past}, \ref{lem:futureobs}).

\begin{lemma}[Past observations and actions]\label{lem:past}
Let $D$ be a DCN with static hidden confounders. Take any set $X$. Let $C \subseteq V_{t}$ be the set of variables in $G_t$ that are direct causes of variables in $G_{t+1}$. Let $Y \subseteq V_{t+\alpha}$ and $Z \subseteq V_{t-\beta}$, with $\alpha > 0$ and $\beta > 0$ (positive natural numbers). The following distributions are identical:

\begin{enumerate}
\item 
$P(Y | do(X),Z,C)$
\item 
$P(Y | do(X),do(Z),C)$
\item 
$P(Y | do(X),C)$
\end{enumerate}
\end{lemma}

\begin{proof}
By the graphical structure of a DCN with static hidden confounders, conditioning on $C$ d-separates $Y$ from $Z$.
The three rules of do-calculus apply, and (1) equals (3) by rule 1, (1) equals (2) by rule 2, and also (2) equals (3) by rule 3.
\qed\end{proof}

In our traffic example, we want to predict the traffic flow $Y$ in two days caused by traffic control mechanisms applied tomorrow $X$, and conditioned on the traffic delay today $C$. Any traffic controls $Z$ applied before today are irrelevant, because their impact is already accounted for in $C$.

\begin{lemma}[Future observations]\label{lem:futureobs}
Let $D$, $X$ and $C$ be as in Lemma~\ref{lem:past}. Let $Y \subseteq V_{t-\alpha}$ and $Z \subseteq V_{t+\beta}$, with $\alpha > 0$ and $\beta > 0$, then:
$$
P(Y|do(X),Z,C)=P(Y|do(X),C)
$$
\end{lemma}

\begin{proof}
By the graphical structure of a DCN with static hidden confounders, conditioning on $C$ d-separates $Y$ from $Z$ and the expression is valid by rule 1 of do-calculus.
\qed\end{proof}

In our example, observing the travel delay today makes observing the future traffic flow irrelevant to evaluate yesterday’s traffic flow.

The following lemma (Lemma \ref{lem:futureact}) is based on the application of do-calculus to DCN in general, with static or dynamic confounders. Intuitively, future actions have no impact on the past.

\begin{lemma}[Future actions]\label{lem:futureact} 
Let $D$ be a DCN. 
Take any sets $X \subseteq V_{t}$ and $Y \subseteq V_{t-\alpha}$, 
with $\alpha >0$. 
Then for any set $Z$ the following equalities hold:

\begin{enumerate}
\item $P(Y|do(X),do(Z))=P(Y|do(Z))$
\item $P(Y|do(X))=P(Y)$
\item $P(Y|Z,do(X))=P(Y|Z)$ whenever $Z \subseteq V_{t-\beta}$ with $\beta > 0$. 
\end{enumerate}
\end{lemma}

\begin{proof}
The first equality 
derives from rule 3 and the proof in~\parencite{shpitser2006identification} that interventions on variables that are not ancestors of $Y$ in $D$ have no effect on $Y$. The second is the special case $Z=\emptyset$. We can transform the third expression using the equivalence \[P(Y|Z,do(X))= P(Y,Z|do(X))/P(Z|do(X));\] since $Y$ and $Z$ precede $X$ in $D$, by rule 3
$P(Y,Z|do(X)) = P(Y,Z)$ and 
$P(Z|do(X)) = P(Z)$,
and then the above equals
$P(Y,Z)/P(Z) = P(Y|Z)$.
\qed\end{proof}

In words, traffic control mechanisms applied next week have no causal effect on the traffic flow this week.


\chapter{Identification of Causal Effects in Dynamic Causal Networks} 

\label{ch:Chapter6} 

In this chapter, we analyse the identifiability of causal effects in the DCN setting. We first study DCNs with static hidden confounders and propose a method for identification of causal effects in DCNs using transition matrices. Then we extend the analysis and identification method to DCNs with dynamic hidden confounders. 

First we define the identification problem in the DCN context.

\section{Identification in DCN}

\begin{definition}[Dynamic Causal Network identification] Let $D$ be a DCN, and $t$, $t+\alpha$ be two time slices of $D$. Let $X$ be a subset of $V_{t}$ and $Y$ be a subset of $V_{t+\alpha}$. 
The DCN identification problem consists of computing the probability distribution $P(Y|do(X))$ from the observed probability distributions in $D$, i.e., computing an expression for the distribution containing no do() operators.
\end{definition}

In this thesis, we always assume that $X$ and $Y$ are disjoint and we only consider the case in which all intervened variables $X$ are in the same time sample. It is not difficult to extend our algorithms to the general case.

The following lemma shows that it is possible to limit the size of the graph to be used for the identification of DCNs.

\begin{lemma}
\label{lem:graphID}
Let $D$ be a DCN with biinfinite graph $\hat G$. Let $t_x$, $t_y$ be two time points in $\hat G$.

Let $G_{xy}$ be sub-graph of $\hat G$ consisting of all time slices in between (and including) $G_{t_x}$ and $G_{t_y}$. Let $G_{lx}$ be graph consisting of all time slices in between (and including) $G_{t_x}$ and the left-most time slice connected to $G_{t_x}$ by a path of dynamic hidden confounders. Let $G_{dx}$ be the graph consisting of all time slices that are in $G_{lx}$ or $G_{xy}$. Let $G_{dx-}$ be the graph consisting of the time slice preceding $G_{dx}$. Let $G_{id}$ be the graph consisting of all time slices in $G_{dx-}$ and $G_{dx}$. If $P(Y|do(X))$ is identifiable in $\hat G$ then it is identifiable in $G_{id}$ and the identification provides the same result on both graphs.
\end{lemma}
\begin{proof} Let $G_{past}$ be the graph consisting of all time slices preceding $G_{id}$ and $G_{future}$ be the graph consisting of all time slices succeeding $G_{id}$ in $\hat G$. By application of do-calculus rule 3, non-ancestors of $Y$ can be ignored from $\hat G$ for the identification of $P(Y|do(X))$ \parencite{shpitser2006identification}, so $G_{future}$ can be discarded. We will now show that identifying $P(Y|do(X))$ in the graph including all time slices of $G_{past}$ and $G_{id}$ is equal to identifying $P(Y|do(X))$ in $G_{id}$. 

By C-component factorization \parencite{tianphd,shpitser2006identification}, the set $V$ of variables in a causal graph $G$ can be partitioned into disjoint groups called C-components by assigning two variables to the same C-component if and only if they are connected by a path consisting entirely of hidden confounder edges, and 
\begin{align*}
P(Y|do(X))=\sum_{V\setminus (Y\cup X)}\prod_{i} P(S_i|do(V\setminus S_i)) 
\end{align*}

where $S_i$ are the C-components of $G_{An(Y)}\setminus X$ expressed as $C(G_{An(Y)}\setminus X)=\{S_1,...,S_k\}$ and $G_{An(Y)}$ is the sub-graph of $G$ including only the variables that are ancestors of $Y$. If and only if every C-component factor $P(S_i|do(V\setminus S_i))$ is identifiable then $P(Y|do(X))$ is identifiable. 

C-component factorization can be applied to DCN. Let $V_{G_{past}}$, $V_{G_{dx-}}$ and $V_{G_{dx}}$ be the set of variables in $G_{past}$, $G_{dx-}$ and $G_{dx}$ respectively. Then $(V_{G_{past}}\cup V_{G_{dx-}}) \cap (Y\cup X)=\emptyset$ and it follows that
$V\setminus (Y\cup X)=V_{G_{past}}\cup V_{G_{dx-}} \cup (V_{G_{dx}}\setminus (Y\cup X))$.

If $S_i \in C(G_{An(S_i)})$ the C-component factor $P(S_i|do(V\setminus S_i))$ is computed as \parencite{shpitser2006identification}:

\begin{align*}
P(S_i|do(V\setminus S_i))=\prod_{\{j|v_j\in S_i\}} P(v_j|v_{\pi}^{(j-1)}) 
\end{align*}

Therefore, there is a $P(v_j|v_{\pi}^{(j-1)}) $ factor for each variable $v_j$ in the C-component, where $v_{\pi}^{(j-1)}$ is the set of all variables preceding $v_j$ in some topological ordering $\pi$ in $G$. 

Let $v_j$ be any variable $v_j \in V_{G_{past}}\cup V_{G_{dx-}}$. There are no hidden confounder edge paths connecting $v_j$ to $X$, and so $v_j \in S_i \in C(G_{An(S_i)})$. Therefore, the C-component factors $Q_{V_{G_{past}}\cup V_{G_{dx-}}}$ of $V_{G_{past}}\cup V_{G_{dx-}}$ can be computed as (chain rule of probability):

$$
Q_{V_{G_{past}}\cup V_{G_{dx-}}}=\prod_{\{j|v_j\in V_{G_{past}}\cup V_{G_{dx-}}\}} P(v_j|v_{\pi}^{(j-1)})=P(V_{G_{past}}\cup V_{G_{dx-}})
$$

We will now look into the C-component factors of $V_{G_{dx}}$. As the DCN is a first order Markov process, the C-component factors of $V_{G_{dx}}$ can be computed as \parencite{shpitser2006identification}:

$$
Q_{V_{G_{dx}}}=\prod_{i}\sum_{S_i\setminus Y} \prod_{\{j|v_j\in S_i\}} P(v_j|v_{\pi}^{(j-1)})=\prod_{i}\sum_{S_i\setminus Y} \prod_{\{j|v_j\in S_i\}} P(v_j|v_{\pi}^{(j-1)}\cap (V_{G_{dx-}}\cup V_{G_{dx}}))
$$

So, these factors have no dependency on $V_{G_{past}}$ and therefore $P(Y|do(X))$ can be marginalized over $V_{G_{past}}$ and simplified as:

$$
P(Y|do(X))=\sum_{V\setminus (Y\cup X)}\prod_{i} P(S_i|do(V\setminus S_i))=\sum_{V_{G_{past}}\cup V_{G_{dx-}} \cup (V_{G_{dx}}\setminus (Y\cup X))}Q_{V_{G_{past}}\cup V_{G_{dx-}}}Q_{V_{G_{dx}}}
$$

\[=\sum_{V_{G_{dx-}} \cup (V_{G_{dx}}\setminus (Y\cup X))}P(V_{G_{dx-}})Q_{V_{G_{dx}}}\]

We can now replace $V_{G_{dx-}} \cup V_{G_{dx}}$ by $V_{G_{id}}$ and define $S'_i$ as the C-component factors of $V_{G_{id}}$ which leads to

\begin{align*}
P(Y|do(X))=\sum_{V_{G_{id}}\setminus (Y\cup X)}\prod_{i} P(S'_i|do(V\setminus S'_i))
\end{align*}

Therefore, the identification of $P(Y|do(X))$ can be computed in the limited graph $G_{id}$. 
\qed\end{proof}

Note that if a DCN contains no dynamic hidden confounders, then $G_{id}$ consists of $G_{xy}$ and the time slice preceding it. In a DCN with dynamic hidden confounders $G_{id}$ may require additional time slices into the past, depending on the reach of hidden dynamic confounder paths. Note that $G_{id}$ may include infinite time slices to the past, if hidden dynamic confounders connect with each other cyclically in successive time slices. However, in this doctoral thesis we will consider only finite dynamic confounding.

This result is crucial to reduce the complexity of identification algorithms in dynamic settings. In order to describe the evolution of a dynamic system over time, after an intervention, we can run a causal identification algorithm over a limited number of time slices of the DCN, instead of the entire DCN.

As discussed in Chapter \ref{ch:Chapter5}, both the DCNs with static hidden confounders and with dynamic hidden confounders can be represented as a Markov chain. For graphical and notational simplicity, we represent these DCN graphically as recurrent time slices, as opposed to the shorter time samples, on the basis that one time slice contains as many time samples as the maximal delay of any directed edge among the processes. Also, for notational simplicity we assume the transition matrix from one time slice to the next to be time-invariant; however, removing this restriction would not make any of the lemmas, theorems or algorithms invalid, as they are the result of graphical non-parametric reasoning.

Consider a DCN under the above assumptions, and let $T$ be its time invariant transition matrix from any time slice $V_{t}$ to $V_{t+1}$. We assume that there is some time $t_0$ such that the distribution $P(V_{t_0})$ is known. Fix now $t_x > t_0$ and a set $X \subseteq V_{t_x}$. 
We will now see how performing an intervention on $X$ affects the distributions in $D$.

We begin by stating a series of lemmas that apply to DCNs in general.

\begin{lemma}
\label{lem:Tbefore}
Let $t$ be such that $t_0 \le t < t_x$, with $X \subseteq V_{t_x}$. Then $P(V_{t}|do(X)) = T^{t-t_0} P(V_{t_0})$.
Namely, transition probabilities
are not affected by an intervention in the future. 
\end{lemma}
\begin{proof}
By Lemma~\ref{lem:futureact}, (2),
$P(V_{t}|do(X))= P(V_{t})$ for all such $t$. 
By definition of $T$, this equals $T\,P(V_{t-1})$. Then induct on $t$ with $P(V_{t_0}) = T^0 P(V_{t_0})$ as base.
\qed\end{proof}

\begin{lemma}
\label{lem:A}
Assume that an expression $P(V_{t+\alpha}|V_{t},do(X))$ is identifiable for some $\alpha>0$. Let $A$ be the matrix whose entries $A_{ij}$ correspond to the probabilities $P(V_{t+\alpha} = v_j|V_t = v_i, do(X))$. Then $P(V_{t+\alpha}|do(X)) = A\,P(V_t|do(X))$.
\end{lemma}
\begin{proof} 
Case by case evaluation of $A$'s entries. 
\qed\end{proof}


\section{Identification of DCN with Static Hidden Confounders} 

Our analysis of DCN with static hidden confounders requires an additional lemma, and then we will be able to provide a theorem and an algorithm for the causal identification of DCN with static hidden confounders.

\begin{lemma}
\label{lem:Tafter}
If $t > t_x$ then 
$P(V_{t+1}|do(X)) = TP(V_{t}|do(X))$.
Namely, transition probabilities
are not affected by an intervention more than one time unit 
in the past. 
\end{lemma}
\begin{proof}
$P(V_{t+1}|do(X)) = T'\,P(V_{t}|do(X))$ where the elements of $T'$ are $P(V_{t+1}|V_t, do(X))$. As $V_{t}$ includes all variables in $G_{t}$ that are direct causes of variables in $G_{t+1}$, conditioning on $V_{t}$ d-separates $X$ from $V_{t+1}$. By Lemma~\ref{lem:past}, we exchange the action $do(X)$ by the observation $X$ and so $P(V_{t+1}|V_t, do(X)) = P(V_{t+1}|V_t, X)$. 

Moreover, $V_t$ d-separates $X$ from $V_{t+1}$, so they are statistically independent given $V_t$. Therefore, \[P(V_{t+1}|V_t, do(X)) = P(V_{t+1}|V_t, X) = P(V_{t+1}|V_t)\] which are the elements of matrix $T$ as required.
\qed\end{proof}

\begin{theorem}
\label{thm:static}
Let $D$ be a DCN with static hidden confounders, and transition matrix $T$.
Let $X\subseteq V_{t_x}$ and $Y\subseteq V_{t_y}$ for two time points $t_x < t_y$.

If the expression $P(V_{t_x+1}|V_{t_x-1}, do(X))$ is identifiable and its values represented in a transition matrix $A$, then $P(Y|do(X))$ is identifiable and 
$$P(Y|do(X))=\sum_{V_{t_y}\setminus Y} T^{t_y-(t_x+1)}AT^{t_x-1-t_0}P(V_{t_0}).$$
\end{theorem}

\begin{proof} 
Applying Lemma~\ref{lem:Tbefore}, we obtain that
\[P(V_{t_x-1}|do(X)) = T^{t_x-1-t_0}P(V_{t_0}).\]
We  assumed that $P(V_{t_x+1}|V_{t_x-1}, do(X))$ is identifiable and,
therefore, Lemma~\ref{lem:A} guarantees that 
$$P(V_{t_x+1}|do(X)) = A\,P(V_{t_x-1}|do(X)) = A\,T^{t_x-1-t_0}P(V_{t_0}).$$
Finally, $P(V_{t_y}|do(X)) = T^{(t_y-(t_x+1))} P(V_{t_x+1}|do(X))$ by repeatedly applying Lemma~\ref{lem:Tafter}.
$P(Y|do(X))$ is obtained by marginalizing variables in $V_{t_y}\setminus Y$ in the resulting expression $T^{t_y-(t_x+1)}AT^{t_x-1-t_0}P(V_{t_0})$.
\qed\end{proof}

As a consequence of Theorem~\ref{thm:static}, causal identification of $D$ reduces to the problem of identifying
the expression $P(V_{t_x+1}|V_{t_x-1},do(X))$. 
The ID algorithm can be used to check whether this expression is identifiable and, if it is, to compute its joint probability from observed data.

Note that Theorem~\ref{thm:static} holds without the assumption of transition matrix time-invariance by replacing powers of $T$ with products of matrices $T_t$. Also, note the assumption on identifiability holds when there is no \textit{hedge} structure (Definition~\ref{def:hedge}) in the DCN for the expression $P(V_{t_x+1}|V_{t_x-1},do(X))$. See Section~\ref{non_id_sec} for the full analysis on non-identifiability.

\subsection{DCN-ID algorithm for DCNs with Static Hidden Confounders}
\label{sec:static}
The DCN-ID algorithm for DCNs with static hidden confounders is given in Figure~\ref{fig:static-algo}. 
Its soundness is immediate from Theorem~\ref{thm:static}, the soundness of the ID algorithm \parencite{shpitser2006identification}, and Lemma~\ref{lem:graphID}.

\begin{theorem}[Soundness] Whenever DCN-ID returns a distribution for  $P(Y|do(X))$, it is correct.\ \qed
\end{theorem}

Observe that line 2 of the algorithm calls ID with a graph of size $4|G|$. Formally, this would require two calls to ID, but notice that in this case we can spare the call for the ``denominator'' $P(V_{t_x-1}|do(X))$ because Lemma~\ref{lem:futureact} guarantees $P(V_{t_x-1}|do(X)) = P(V_{t_x-1})$. Computing transition matrix A in line 3 has complexity $O((4k)^{(b+2)})$, where $k$ is the number of variables in one time slice and $b$ the number of bits encoding each variable.
The formula in line 4 is the multiplication of $P(V_{t_0})$ by $n=(t_y-t_0)$ matrices, which has complexity $O(n.b^2)$. To solve the same problem with the ID algorithm would require running it on the entire graph of size $n|G|$ and evaluating the resulting joint probability with complexity $O((n.k)^{(b+2)})$ compared to $O((4k)^{(b+2)}+n.b^2)$ with DCN-ID.

If the problem that we want to solve is evaluating the trajectory of the system over time,
$$(P(V_{t_x+1}),P(V_{t_x+2}),P(V_{t_x+3}),...P(V_{t_x+n}))$$
after an intervention at time slice $t_x$, with ID we would need to run ID $n$ times and evaluate the $n$ outputs with overall complexity $O((k)^{(b+2)}+(2k)^{(b+2)}+(3k)^{(b+2)}+...+(n.k)^{(b+2)})$. Doing the same with DCN-ID requires running ID one time to identify $P(V_{t_x+1})$, evaluating the output and applying successive transition matrix multiplications to obtain the joint probability of the time slices thereafter, with resulting complexity $O((4k)^{(b+2)}+n.b^2)$.

\begin{figure}[ht]
\hrule\medskip
Function \textbf{DCN-ID}($Y$,$t_y$, $X$,$t_x$, $G$,$C$,$T$,$P(V_{t_0})$)

INPUT: 
\begin{itemize}
\item DCN defined by a causal graph $G$ 
on a set of variables $V$ and a set $C \subseteq V \times V$ describing causal relations from $V_t$ to $V_{t+1}$ for every $t$
\item transition matrix $T$ for $G$ derived
from observational data
\item a set $Y$ included in $V_{t_y}$
\item a set $X$ included in $V_{t_x}$
\item distribution $P(V_{t_0})$ at the initial state, 
\end{itemize}

OUTPUT: The distribution  $P(Y|do(X))$, or else FAIL

\begin{enumerate}
\item let $G'$ be the acyclic graph formed by joining $G_{t_x-2}$, $G_{t_x-1}$, $G_{t_x}$, and $G_{t_x+1}$
by the causal relations given by $C$;
\item run the standard ID algorithm for expression $P(V_{t_x+1}|V_{t_x-1},do(X))$ on $G'$; if it returns FAIL, return FAIL;
\item else, use the resulting distribution to compute the transition matrix $A$, where $A_{ij} = P(V_{t_x+1}=v_i|V_{t_x-1}=v_j, do(X))$;
\item return $\sum_{V_{t_y}\setminus Y} T^{t_y-(t_x+1)}\,A\,T^{t_x-1-t_0}\,P(V_{t_0})$;
\end{enumerate}

\caption{The DCN-ID algorithm for DCNs with static hidden confounders}
\label{fig:static-algo}
\medskip\hrule
\end{figure}


\section{Identification of DCN with Dynamic Hidden Confounders}

We now discuss the case of DCNs with dynamic hidden confounders, that is, with hidden confounders that influence variables in consecutive time slices.

The presence of dynamic hidden confounders d-connects time slices, and we will see in the following lemmas how this may be an obstacle for the identifiability of the DCN.

If dynamic hidden confounders are present, Lemma~\ref{lem:Tafter} no longer holds, since d-separation is no longer guaranteed. As a consequence, we cannot guarantee the DCN will recover its ``natural'' (non-interventional) transition probabilities from one cycle to the next after the intervention is performed.

Our statement of the identifiability theorem for DCNs with dynamic hidden confounders is weaker and includes in its assumptions those conditions that can no longer be guaranteed.

\begin{theorem}
\label{thm:dynamic}
Let $D$ be a DCN with dynamic hidden confounders. Let $T$ be its transition matrix under no interventions. We further assume that:
\begin{enumerate}
\item $P(V_{t_x+1}|V_{t_x-1}, do(X))$ is identifiable and its values represented in a transition matrix $A$ 
\item For all
$t > t_x+1$, $P(V_{t}|V_{t-1}, do(X))$ is identifiable and its values represented in a transition matrix $M_t$
\end{enumerate}
Then $P(Y|do(X))$ is identifiable and computed by
\[P(Y|do(X))=\sum_{V_{t_y}\setminus Y} \left[\prod\limits_{t=t_x+2}^{t_y} M_t\right]\,A\,T^{t_x-1-t_0}P(V_{t_0}).\]
\end{theorem}
\begin{proof} 
Similar to the proof of Theorem~\ref{thm:static}. By Lemma~\ref{lem:Tbefore}, we can compute the distribution up to time $t_x-1$ as 
$$P(V_{t_x-1}|do(X)) = T^{t_x-1-t_0}P(V_{t_0}).$$ 
Using the first assumption in the statement of the theorem, by Lemma~\ref{lem:A} we obtain
$$P(V_{t_x+1}|do(X)) = A\,T^{t_x-1-t_0}P(V_{t_0}).$$
Then, we compute the final $P(V_{t_y}|do(X))$ using the matrices $M_t$ from the statement of the theorem that allows us to compute probabilities for subsequent time-slices. Namely, 
\begin{align*}
P(V_{t_x+2}|do(X)) &= M_{t_x+2}\,A\,T^{t_x-1-t_0}P(V_{t_0}), \\
P(V_{t_x+3}|do(X)) &= M_{t_x+3}\,M_{t_x+2}\,A\,T^{t_x-1-t_0}P(V_{t_0}),
\end{align*}
and so on until we find
\[P(V_{t_y}|do(X)) = \left[\prod\limits_{t=t_x+2}^{t_y} M_t\right]\,A\,T^{t_x-1-t_0}P(V_{t_0}).\]
Finally, the do-free expression of $P(Y|do(X))$ is obtained by marginalization over variables of $V_{t_y}$ not in $Y$.
\qed\end{proof}

Again, note that Theorem~\ref{thm:dynamic} holds without the assumption of transition matrix time-invariance by replacing powers of $T$ with products of matrices $T_t$. Also, the assumptions of Theorem~\ref{thm:dynamic} on identifiability hold when there is no \textit{hedge} structure (Definition~\ref{def:hedge}) in the DCN for expressions $P(V_{t_x+1}|V_{t_x-1}, do(X))$ and $P(V_{t}|V_{t-1}, do(X))$. See Section~\ref{non_id_sec} for the full analysis on non-identifiability.

\subsection{DCN-ID algorithm for DCNs with Dynamic Hidden Confounders}
\label{sec:dynamic}

The DCN-ID algorithm for DCNs with dynamic hidden confounders  is given in Figure~\ref{fig:dynamic-algo}.

Its soundness is immediate from Theorem~\ref{thm:dynamic}, the soundness of the ID algorithm \parencite{shpitser2006identification}, and Lemma~\ref{lem:graphID}.

\begin{theorem}[Soundness] Whenever DCN-ID returns a distribution for  $P(Y|do(X))$, it is correct.\ \qed
\end{theorem}

Notice that this algorithm is more expensive than the DCN-ID algorithm for DCNs with static hidden confounders. In particular, it requires $(t_y - t_x)$ calls to the ID algorithm with increasingly larger chunks of the DCN. To identify a single future effect $P(Y|do(X))$ it may be simpler to invoke Lemma~\ref{lem:graphID} and do a unique call to the ID algorithm for the expression $P(Y|do(X))$ restricted to the causal graph $G_{id}$. However, to predict the trajectory of the system over time after an intervention, the DCN-ID algorithm for dynamic hidden confounders directly identifies the post-intervention transition matrix and its evolution. A system characterized by a time-invariant transition matrix before the intervention may be characterized by a time dependent transition matrix, given by the DCN-ID algorithm, after the intervention. This dynamic view offers opportunities for the analysis of the time evolution of the system, and conditions for convergence to a steady state.

To give an intuitive example of a DCN with dynamic hidden confounders, let us consider three roads in which the traffic conditions are linked by hidden confounders from $tr1$ to $tr2$ the following day, and from $tr2$ to $tr3$ the day after. After applying control mechanisms to $tr1$, the traffic transition matrix to the next day is different than the transition matrix several days later, because it is not possible to d-separate the future from the controlling action by just conditioning on a given day. As a consequence, the identification algorithm must calculate every successive transition matrix in the future.

\begin{figure}[]
\hrule\medskip
Function \textbf{DCN-ID}($Y$,$t_y$, $X$,$t_x$, $G$,$C$,$C'$,$T$,$P(V_{t_0})$)

INPUT: 
\begin{itemize}
\item DCN defined by a causal graph $G$ 
on a set of variables $V$ and a set $C \subseteq V \times V$ describing causal relations from $V_t$ to $V_{t+1}$ for every $t$, and a set $C' \subseteq V \times V$ describing hidden confounder relations from $V_t$ to $V_{t+1}$ for every $t$
\item transition matrix $T$ for $G$ derived
from observational data
\item a set $Y$ included in $V_{t_y}$
\item a set $X$ included in $V_{t_x}$
\item distribution $P(V_{t_0})$ at the initial state, 
\end{itemize}

OUTPUT: The distribution  $P(Y|do(X))$, or else FAIL

\begin{enumerate}
\item let $G'$ be the graph consisting of all time slices in between (and including) $G_{t_x+1}$ and the time slice preceding the left-most time slice connected to $X$ by a hidden confounder path or, if there is no hidden confounder path to X, $G_{t_x-2}$;
\item run the standard ID algorithm for expression $P(V_{t_x+1}|V_{t_x-1},do(X))$ on $G'$; if it returns FAIL, return FAIL;
\item else, use the resulting distribution to compute the transition matrix $A$, where $A_{ij} = P(V_{t_x+1}=v_i|V_{t_x-1}=v_j, do(X))$;
\item for each $t$ from $t_x+2$ up to $t_y$:
	\begin{enumerate}
    \item let $G''$ be the graph consisting of all time slices in between (and including) $G_{t}$ and the time slice preceding the left-most time slice connected to $X$ by a hidden confounder path or, if there is no hidden confounder path to X, $G_{t_x-1}$;
	\item run the standard ID algorithm on $G''$ for the expression $P(V_t|V_{t-1},do(X))$; if it returns FAIL, return FAIL;
    \item else, use the resulting distribution to compute the transition matrix $M_t$, where $(M_t)_{ij} = P(V_{t}=v_i|V_{t-1}=v_j, do(X))$;
	\end{enumerate}
\item return $\sum_{V_{t_y}\setminus Y} \left[\prod\limits_{t=t_x+2}^{t_y} M_t\right]\,A\,T^{t_x-1-t_0}P(V_{t_0})$;
\end{enumerate}

\caption{The DCN-ID algorithm for DCNs with dynamic hidden confounders}
\label{fig:dynamic-algo}
\medskip\hrule
\end{figure}


\section{Non-Identifiability}
\label{non_id_sec}

In this section we show that the identification algorithms, as formulated in previous sections, are not complete, and we develop the algorithms for complete identification of DCNs. To prove completeness, we use previous results \parencite{shpitser2006identification}. It is shown there that the absence of a \textit{hedge} structure (Definition~\ref{def:hedge}) is a sufficient and necessary condition for identifiability. The same applies in the context of DCNs.

\begin{lemma}[DCN complete identification]
\label{lem:dcn_complete}
Let $D$ be a DCN with hidden confounders. Let $X$ and $Y$ be sets of variables in $D$. $P(Y|do(X))$ is identifiable iff there is no hedge in $D$ for $P(Y|do(X))$.
\end{lemma}

\begin{proof}
If a $hedge$ exists in $D$ for $P(Y|do(X))$ then the conditions for the existence of a $hedge$ in every time slice of $D$ are true. By \parencite{shpitser2006identification}, $P(Y|do(X))$ is identifiable iff there is no $hedge$ for $X$ and $Y$ in the expanded causal graph of $D$.
\qed\end{proof}

We can show that the algorithms presented in the previous section, in some cases introduce hedges in the sub-networks they analyse, even if no hedges existed in the original expanded network.

\begin{lemma} The DCN-ID algorithms for DCNs with static hidden confounders (Section~\ref{sec:static}) and dynamic hidden confounders (Section~\ref{sec:dynamic}) are not complete.
\label{lem:dcn_id_not_complete}
\end{lemma}

\begin{proof}
Let $D$ be an DCN. Let $X$ be such that $D$ contains two $R$-rooted C-forests $F$ and $F'$, $F'\subseteq F$, $F\cap X \neq 0$, $F'\cap X = 0$. Let $Y$ be such that $R\not\subset An(Y)_{D_{\bar{X}}}$. The condition for $Y$ implies that $D$ does not contain a hedge, and is therefore identifiable by Lemma~\ref{lem:dcn_complete}. Let the set of variables at time slice $t_x+1$ of $D$,  $V_{t_x+1}$, be such that $R\subset An(V_{t_x+1})_{D_{\bar{X}}}$. By Definition~\ref{def:hedge}, $D$ contains a hedge for $P(V_{t_x+1}|V_{t_x-1},do(X))$. The identification of $P(Y|do(X))$ requires DCN-ID to identify $P(V_{t_x+1}|V_{t_x-1},do(X))$ which fails.
\qed\end{proof}

The proof of Lemma~\ref{lem:dcn_id_not_complete} provides the framework to build a complete algorithm for identification of DCNs.

\begin{figure}[H]
\begin{center}
\includegraphics[width=8cm,height=4cm]{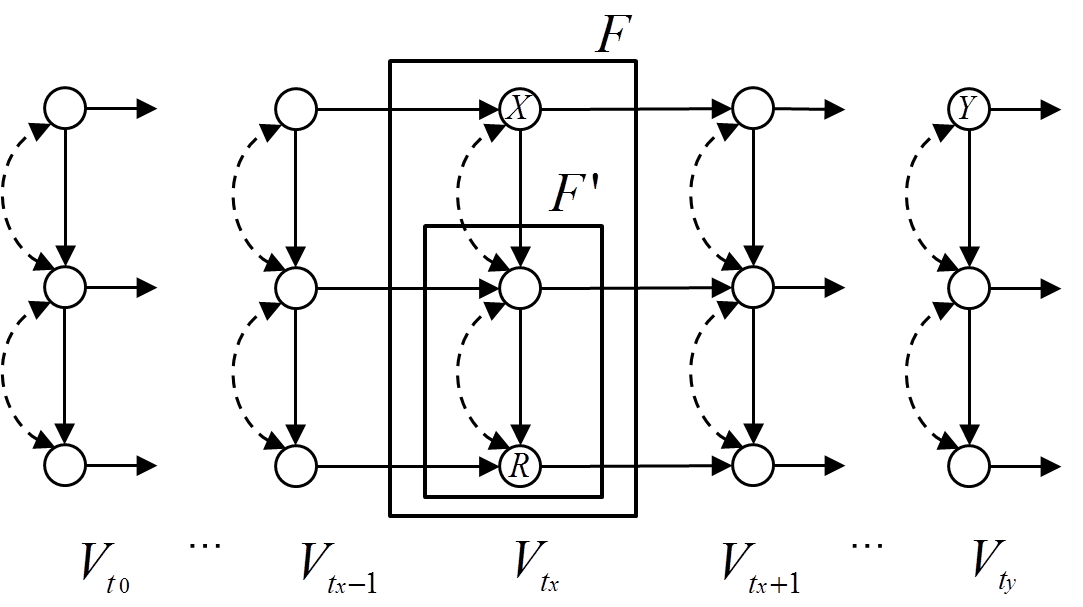}
\end{center}
\caption{Identifiable Dynamic Causal Network which the DCN-ID algorithm fails to identify. $F$ and $F'$ are $R$-rooted C-forests, but since $R$ is not an ancestor of $Y$ there is no hedge for $P(Y|do(X))$. However, $R$ is an ancestor of $V_{t_x+1}$ and DCN-ID fails when finding the hedge for $P(V_{t_x+1}|V_{t_x-1}, do(X))$.}
\label{fig:dcn_no_hedge}
\end{figure}

Figure~\ref{fig:dcn_no_hedge} shows an identifiable DCN that DCN-ID fails to identify.

\subsection{Complete DCN identification with Static Hidden Confounders}
\label{sec:completestatic}

The DCN-ID algorithm can be modified so that no hedges are introduced if none existed in the original network. This is done at the cost of more complicated notation, because the fragments of network to be analysed do no longer correspond to natural time slices. More delicate surgery is needed.

\begin{lemma}
\label{lem:hedge_static}
Let $D$ be a DCN with static hidden confounders. Let $X\subseteq V_{t_x}$ and $Y\subseteq V_{t_y}$ for two time slices $t_x < t_y$. If there is a hedge $H$ for $P(Y|do(X))$ in $D$ then $H\subseteq V_{t_x}$.
\end{lemma}

\begin{proof}
By definition of hedge, $F$ and $F'$ are connected by hidden confounders to $X$. As $D$ has only static hidden confounders $F$, $F'$ and $X$ must be within $t_x$.
\qed\end{proof}

\begin{lemma}
\label{lem:static_complete}
Let $D$ be a DCN with static hidden confounders. Let $X\subseteq V_{t_x}$ and $Y\subseteq V_{t_y}$ for two time slices $t_x < t_y$.
Then, $P(Y|do(X))$ is identifiable if and only if the expression $P(V_{t_x+1}\cap An(Y)|V_{t_x-1}, do(X))$ is identifiable.
\end{lemma}

\begin{proof} 
(\textbf{if}) By Lemma~\ref{lem:dcn_complete}, if 
\begin{align*}
P(V_{t_x+1}\cap An(Y)&|V_{t_x-1},do(X))=\frac{P(V_{t_x+1}\cap An(Y),V_{t_x-1}|do(X))}{P(V_{t_x-1})}
\end{align*} 
is identifiable, then there is no hedge for this expression in $D$. By Lemma~\ref{lem:hedge_static}, if $D$ has static hidden confounders, a hedge must be within time slice $t_x$. If time slice $t_x$ does not contain two $R$-rooted C-forests $F$ and $F'$ such that $F'\subseteq F$, $F\cap X \neq 0$, $F'\cap X = 0$, then there is no hedge for any set $Y$ so there is no hedge for the expression $P(Y|do(X))$, which makes it identifiable. Now let us assume time slice $t_x$ contains two $R$-rooted C-forests $F$ and $F'$ such that $F'\subseteq F$, $F\cap X \neq 0$, $F'\cap X = 0$, then $R\not\subset An(V_{t_x+1}\cap An(Y),V_{t_x-1})_{D_{\bar{X}}}$. As $R$ is in time slice $t_x$, this implies $R\not\subset An(Y)_{D_{\bar{X}}}$ and so there is no hedge for the expression $P(Y|do(X))$ which makes it identifiable.

(\textbf{only if}) By Lemma~\ref{lem:dcn_complete}, if $P(Y|do(X))$ is identifiable then there is no hedge for $P(Y|do(X))$ in $D$. By Lemma~\ref{lem:hedge_static} if $D$ has static hidden confounders, a hedge must be within time slice $t_x$. If time slice $t_x$ does not contain two $R$-rooted C-forests $F$ and $F'$ such that $F'\subseteq F$, $F\cap X \neq 0$, $F'\cap X = 0$, then there is no hedge for any set $Y$ so there is no hedge for the expression 
\begin{align*}
P(V_{t_x+1}\cap An(Y)&|V_{t_x-1},do(X))=\frac{P(V_{t_x+1}\cap An(Y),V_{t_x-1}|do(X))}{P(V_{t_x-1})}
\end{align*} 
which makes it identifiable. Now let us assume time slice $t_x$ contains two $R$-rooted C-forests $F$ and $F'$ such that $F'\subseteq F$, $F\cap X \neq 0$, $F'\cap X = 0$, then $R\not\subset An(Y)_{D_{\bar{X}}}$ (if $R\subset An(Y)_{D_{\bar{X}}}$ $D$ would contain a hedge by definition). As $R$ is in time slice $t_x$, $R\not\subset An(Y)_{D_{\bar{X}}}$ implies $R\not\subset An(V_{t_x+1}\cap An(Y))_{D_{\bar{X}}}$ and $R\not\subset An(V_{t_x+1}\cap An(Y),V_{t_x-1})_{D_{\bar{X}}}$ so there is no hedge for $P(V_{t_x+1}\cap An(Y)|V_{t_x-1}, do(X))$ which makes this expression identifiable.
\qed\end{proof}

\begin{lemma}
\label{lem:A_generic}
Assume that an expression $P(V'_{t+\alpha}|V_{t},do(X))$ is identifiable for some $\alpha>0$ and $V'_{t+\alpha}\subseteq V_{t+\alpha}$. Let $A$ be the matrix whose entries $A_{ij}$ correspond to the probabilities $P(V'_{t+\alpha} = v_j|V_t = v_i, do(X))$. Then $P(V'_{t+\alpha}|do(X)) = A\,P(V_t|do(X))$.
\end{lemma}
\begin{proof} 
Case by case evaluation of $A$'s entries. 
\qed\end{proof}

\begin{lemma}
\label{lem:Tafter_marginal}
Let $D$ be a DCN with static hidden confounders. 
Let $X\subseteq V_{t_x}$ and $Y\subseteq V_{t_y}$ for two time slices $t_x < t_y$. Then 
$P(Y|do(X))=\left[\prod\limits_{t=t_x+2}^{t_y} M_t\right]P(V_{t_x+1}\cap An(Y)|do(X))$ where $M_t$ is the matrix whose entries correspond to the probabilities $P(V_{t}\cap An(Y) = v_j|V_{t-1}\cap An(Y) = v_i)$.
\end{lemma}
\begin{proof}
For the identification of $P(Y|do(X))$ we can restrict our attention to the subset of variables in $D$ that are ancestors of Y. Then we repeatedly apply Lemma~\ref{lem:Tafter} on this subset from $t=t_x+2$ to $t=t_y$ until we find $P(V_{t_y}\cap An(Y)|do(X))=P(Y|do(X))$.
\qed\end{proof}

\begin{figure}[H]
\hrule\medskip
Function \textbf{cDCN-ID}($Y$,$t_y$, $X$,$t_x$, $G$,$C$,$T$,$P(V_{t_0})$)

INPUT: 
\begin{itemize}
\item DCN defined by a causal graph $G$ 
on a set of variables $V$ and a set $C \subseteq V \times V$ describing causal relations from $V_t$ to $V_{t+1}$ for every $t$
\item transition matrix $T$ representing the probabilities $P(V_{t+1}|V_{t})$ derived from observational data
\item a set $Y$ included in $V_{t_y}$
\item a set $X$ included in $V_{t_x}$
\item distribution $P(V_{t_0})$ at the initial state, 
\end{itemize}

OUTPUT: The distribution  $P(Y|do(X))$ if it is identifiable, or else FAIL

\begin{enumerate}
\item let $G'$ be the acyclic graph formed by joining $G_{t_x-2}$, $G_{t_x-1}$, $G_{t_x}$, and $G_{t_x+1}$
by the causal relations given by $C$;
\item run the standard ID algorithm for expression $P(V_{t_x+1}\cap An(Y)|V_{t_x-1},do(X))$ on $G'$; if it returns FAIL, return FAIL;
\item else, use the resulting distribution to compute the transition matrix $A$, where $A_{ij} = P(V_{t_x+1}\cap An(Y)=v_i|V_{t_x-1}=v_j, do(X))$;
\item let $M_t$ be the matrix $T$ marginalized as $P(V_{t}\cap An(Y) = v_j|V_{t-1}\cap An(Y) = v_i)$
\item return $\left[\prod\limits_{t=t_x+2}^{t_y} M_t\right]A\,T^{t_x-1-t_0}\,P(V_{t_0})$;
\end{enumerate}

\caption{The cDCN algorithm for DCNs with static hidden confounders}
\label{fig:static-algo-complete}
\medskip\hrule
\end{figure}

\begin{theorem}
\label{thm:calc_static_complete}
Let $D$ be a DCN with static hidden confounders and transition matrix $T$. 
Let $X\subseteq V_{t_x}$ and $Y\subseteq V_{t_y}$ for two time slices $t_x < t_y$.
If $P(Y|do(X))$ is identifiable then $P(Y|do(X))=\left[\prod\limits_{t=t_x+2}^{t_y} M_t\right]AT^{t_x-1-t_0}P(V_{t_0})$ where $A$ is the matrix whose entries $A_{ij}$ correspond to $P(V_{t_x+1}\cap An(Y)|V_{t_x-1}, do(X))$ and $M_t$ is the matrix whose entries correspond to the probabilities $P(V_{t}\cap An(Y) = v_j|V_{t-1}\cap An(Y) = v_i)$.
\end{theorem}

\begin{proof} 
Applying Lemma~\ref{lem:Tbefore}, we obtain that
$$P(V_{t_x-1}|do(X)) = T^{t_x-1-t_0}P(V_{t_0}).$$
By Lemma~\ref{lem:static_complete} $P(V_{t_x+1}\cap An(Y)|V_{t_x-1}, do(X))$ is identifiable. Lemma~\ref{lem:A_generic} guarantees that $P(V_{t_x+1}\cap An(Y)|do(X)) = A\,P(V_{t_x-1}|do(X)) = A\,T^{t_x-1-t_0}P(V_{t_0})$. Then we apply Lemma~\ref{lem:Tafter_marginal} and obtain the resulting expression \[P(Y|do(X))=\left[\prod\limits_{t=t_x+2}^{t_y} M_t\right]AT^{t_x-1-t_0}P(V_{t_0}).\]
\qed\end{proof}

The cDCN-ID algorithm for identification of DCNs with static hidden confounders is given in Figure~\ref{fig:static-algo-complete}. 

\begin{theorem}[Soundness and completeness]
\label{thm:completeness}
The cDCN-ID algorithm for DCNs with static hidden confounders is sound and complete.
\end{theorem}

\begin{proof} 
The completeness derives from Lemma~\ref{lem:static_complete} and the soundness from Theorem~\ref{thm:calc_static_complete}.
\qed\end{proof}

\subsection{Complete DCN identification with Dynamic Hidden Confounders}
\label{sec:completesdynamic}

We now discuss the complete identification of DCNs with dynamic hidden confounders. First we introduce the concept of dynamic time span from which we derive two lemmas.

\begin{definition}[Dynamic time span]
Let $D$ be a DCN with dynamic hidden confounders and $X\subseteq V_{t_x}$. Let $t_m$ be the maximal time slice d-connected by confounders to $X$; $t_m-t_x$ is called the dynamic time span of $X$ in $D$.
\end{definition}

Note that the dynamic time span of $X$ in $D$ can be in some cases infinite, the simplest case being when $X$ is connected by a hidden confounder to itself at $V_{t_x+1}$. In this thesis we consider finite dynamic time spans only. We will label the dynamic time span of $X$ as $t_{dx}$.

\begin{lemma}
\label{lem:hedge_dynamic}
Let $D$ be a DCN with dynamic hidden confounders. Let $X$, $Y$ be sets of variables in $D$. Let $t_{dx}$ be the dynamic time span of $X$ in $D$. If there is a hedge for $P(Y|do(X))$ in $D$ then the hedge does not include variables at $t>t_x+t_{dx}$.
\end{lemma}

\begin{proof}
By definition of hedge, $F$ and $F'$ are connected by hidden confounders to $X$. The maximal time point connected by hidden confounders to $X$ is $t_x+t_{dx}$.
\qed\end{proof}

\begin{lemma}
\label{lem:dynamic_complete}
Let $D$ be a DCN with dynamic hidden confounders. Let $X\subseteq V_{t_x}$ and $Y\subseteq V_{t_y}$ for two time slices $t_x, t_y$. Let $t_{dx}$ be the dynamic time span of $X$ in $D$ and $t_x + t_{dx} < t_y$.
$P(Y|do(X))$ is identifiable if and only if $P(V_{t_x+t_{dx}+1}\cap An(Y)|V_{t_x-1}, do(X))$ is identifiable.
\end{lemma}

\begin{proof} 
Same as the proof of Lemma~\ref{lem:static_complete}, but replacing "static" by "dynamic", $V_{t_x+1}$ by $V_{t_x+t_{dx}+1}$, Lemma~\ref{lem:hedge_static} by Lemma~\ref{lem:hedge_dynamic}, and "time slice $t_x$" by "time slices $t_x$ to $t_x+t_{dx}$".
\qed\end{proof}

\begin{theorem}
\label{thm:calc_dynamic_complete}
Let $D$ be a DCN with dynamic hidden confounders and $T$ be its transition matrix under no interventions. Let $X\subseteq V_{t_x}$ and $Y\subseteq V_{t_y}$ for two time slices $t_x, t_y$. Let $t_{dx}$ be the dynamic time span of $X$ in $D$ and $t_x + t_{dx} < t_y$.
If $P(Y|do(X))$ is identifiable then:
\begin{enumerate}
\item $P(V_{t_x+t_{dx}+1}\cap An(Y)|V_{t_x-1}, do(X))$ is identifiable by matrix $A$ 
\item For $t > t_x+t_{dx}+1$, $P(V_{t}\cap An(Y)|V_{t-1}\cap An(Y), do(X))$ is identifiable by matrix $M_t$
\item $P(Y|do(X))=\left[\prod\limits_{t=t_x+t_{dx}+2}^{t_y} M_t\right]\,A\,T^{t_x-1-t_0}P(V_{t_0})$
\end{enumerate}
\end{theorem}
\begin{proof} 
We obtain the first statement from Lemma~\ref{lem:dynamic_complete} and Lemma~\ref{lem:A_generic}. Then if $t > t_x+t_{dx}+1$, then the set $(V_{t}\cap An(Y),V_{t-1}\cap An(Y))$ has the same ancestors than $Y$ within time slices $t_x$ to $t_x+t_{dx}+1$, so if $P(Y|do(X))$ is identifiable then $P(V_{t}\cap An(Y)|V_{t-1}\cap An(Y), do(X))$ is identifiable, which proves the second statement. Finally, we obtain the third statement similarly to the proof of Theorem~\ref{thm:dynamic} but using statements 1 and 2 as proved instead of assumed.
\qed\end{proof}

\begin{figure}[H]
\hrule\medskip
Function \textbf{cDCN-ID}($Y$,$t_y$, $X$,$t_x$, $G$,$C$,$C'$,$T$,$P(V_{t_0})$)

INPUT: 
\begin{itemize}
\item DCN defined by a causal graph $G$ 
on a set of variables $V$ and a set $C \subseteq V \times V$ describing causal relations from $V_t$ to $V_{t+1}$ for every $t$, and a set $C' \subseteq V \times V$ describing hidden confounders from $V_t$ to $V_{t+1}$ for every $t$
\item transition matrix $T$ for $G$ derived
from observational data
\item a set $Y$ included in $V_{t_y}$
\item a set $X$ included in $V_{t_x}$
\item distribution $P(V_{t_0})$ at the initial state, 
\end{itemize}

OUTPUT: The distribution $P(Y|do(X))$ if it is identifiable or else FAIL

\begin{enumerate}
\item let $G'$ be the graph consisting of all time slices in between (and including) $G_{t_x+1}$ and the time slice preceding the left-most time slice connected to $X$ by a hidden confounder path or, if there is no hidden confounder path to X, $G_{t_x-2}$;
\item run the standard ID algorithm for expression $P(V_{t_x+t_{dx}+1}\cap An(Y)|V_{t_x-1}, do(X))$ on $G'$; if it returns FAIL, return FAIL;
\item else, use the resulting distribution to compute the transition matrix $A$, where $A_{ij} = P(V_{t_x+t_{dx}+1}\cap An(Y)=v_i|V_{t_x-1}=v_j, do(X))$;
\item for each $t$ from $t_x+t_{dx}+2$ up to $t_y$:
	\begin{enumerate}
    \item let $G''$ be the graph consisting of all time slices in between (and including) $G_{t}$ and the time slice preceding the left-most time slice connected to $X$ by a hidden confounder path or, if there is no hidden confounder path to X, $G_{t_x-1}$;
	\item run the standard ID algorithm on $G''$ for the expression $P(V_t\cap An(Y)|V_{t-1}\cap An(Y),do(X))$; if it returns FAIL, return FAIL;
    \item else, use the resulting distribution to compute the transition matrix $M_t$, where $(M_t)_{ij} = P(V_{t}\cap An(Y)=v_i|V_{t-1}\cap An(Y)=v_j, do(X))$;
	\end{enumerate}
\item return $ \left[\prod\limits_{t=t_x+t_{dx}+2}^{t_y} M_t\right]\,A\,T^{t_x-1-t_0}P(V_{t_0})$;
\end{enumerate}

\caption{The cDCN algorithm for DCNs with dynamic hidden confounders}
\label{fig:dynamic-algo-complete}
\medskip\hrule
\end{figure}

The cDCN-ID algorithm for DCNs with dynamic hidden confounders is given in Figure~\ref{fig:dynamic-algo-complete}.

\begin{theorem}[Soundness and completeness]
\label{thm:completeness_dynamic}
The cDCN-ID algorithm for DCNs with dynamic hidden confounders is sound and complete.
\end{theorem}

\begin{proof} 
The completeness derives from the first and second statements of Theorem~\ref{thm:calc_dynamic_complete}. The soundness derives from the third statement of Theorem~\ref{thm:calc_dynamic_complete}.
\qed\end{proof}


\section{Transportability}

\parencite{pearl2011transportability} introduced the sID algorithm, based on do-calculus, to identify a transport formula between two domains, where the effect in a target domain can be estimated from experimental results in a source domain and some observations on the target domain, thus avoiding the need to perform an experiment on the target domain.

Let us consider a country with a number of alternative roads linking city pairs in different provinces. Suppose that the alternative roads are all consistent with the same causal model (such as the one in Figure~\ref{fig:dcn_confounder_compact}, for example) but have different traffic patterns (proportion of cars/trucks, toll prices, traffic light durations...).
Traffic authorities in one of the provinces may have experimented with policies and observed the impact on, say, traffic delay. This information may be usable to predict the average travel delay in another province for a given traffic policy. The source domain (province where the impact of traffic policy has already been monitored) and target domain (new province) share the same causal relations among variables, represented by a single DCN (see Figure~\ref{fig:dcntransport}). 

\begin{figure}[H]
\begin{center}
\includegraphics[width=2.8cm]{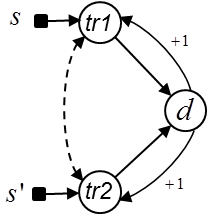}
\end{center}
\caption{A DCN with selection variables $s$ and $s'$, representing the differences in the distribution of variables $tr1$ and $tr1$ in two domains $M_1$ and $M_2$ (two provinces in the same country). This model can be used to evaluate the causal impacts of traffic policy in the target domain $M_2$ based on the impacts observed in the source domain $M_1$.}
\label{fig:dcntransport}
\end{figure}

The target domain may have specific distributions of the toll price and traffic signs, which are accounted for in the model by adding a set of selection variables to the DCN, pointing at variables whose distribution differs among the two domains. If the DCN with the selection variables is identifiable for the traffic delay upon increasing the toll price, then the DCN identification algorithm provides a transport formula which combines experimental probabilities from the source domain and observed distributions from the target domain. Thus, the traffic authorities in the new province can evaluate the impacts before effectively changing traffic policies. This amounts to relational knowledge transfer learning between the two domains \parencite{pan2010survey}.

Consider a DCN with static hidden confounders only. We have demonstrated already that for identification of the effects
of an intervention at time $t_x$ we can restrict our attention to four time slices of the DCN, $t_x-2$, $t_x-1$, $t_x$, and $t_x+1$. Let $M_1$ and $M_2$ be two domains based on this same DCN, 
though the distributions of some variables in $M_1$ and $M_2$ may differ. Then we have 
\begin{align*}
P_{M_2}(Y|do(X))=T_{M_2}^{t_y-(t_x+1)}A_{M_2}T_{M_2}^{t_x-1-t_0}P(V_{t_0}),
\end{align*}
where the entry $ij$ of matrix $A_{M_2}$ corresponds to the transition probability $P_{M_2}(V_{t_x+1}=v_i|V_{t_x-1}=v_j, do(X))$.

By applying the identification algorithm sID, with selection variables, to the elements of matrix $A$ we then obtain a transport formula, which combines experimental distributions in $M_1$ with observational distributions in $M_2$. The algorithm for transportability of causal effects with static hidden confounders is given in Figure~\ref{fig:transport-static-algo}.

\begin{figure}[H]
\hrule\medskip
Function \textbf{DCN-sID}($Y$,$t_y$, $X$,$t_x$, $G$,$C$,$T_{M_2}$,$P_{M_2}(V_{t_0})$,$I_{M_1}$)

INPUT: 
\begin{itemize}
\item DCN defined by a causal graph $G$ (common to both source and target domains $M_1$ and $M_2$) over a set of variables $V$ and a set $C \subseteq V \times V$ describing causal relations from $V_t$ to $V_{t+1}$ for every $t$
\item transition matrix $T_{M_2}$ for $G$ derived
from observational data in $M_2$
\item a set $Y$ included in $V_{t_y}$
\item a set $X$ included in $V_{t_x}$
\item distribution $P_{M_2}(V_{t_0})$ at the initial state in $M_2$
\item set of interventional distributions $I_{M_1}$ in $M_1$
\item set S of selection variables

\end{itemize}

OUTPUT: The distribution  $P_{M_2}(Y|do(X))$ in $M_2$ in terms of $T_{M_2}$, $P_{M_2}(V_{t_0})$ and $I_{M_1}$, or else FAIL

\begin{enumerate}
\item let $G'$ be the acyclic graph formed by joining $G_{t_x-2}$, $G_{t_x-1}$, $G_{t_x}$, and $G_{t_x+1}$
by the causal relations given by $C$;
\item run the standard sID algorithm for expression $P(V_{t_x+1}|V_{t_x-1},do(X))$ on $G'$; if it returns FAIL, return FAIL;
\item else, use the resulting transport formula to compute the transition matrix $A$, where $A_{ij} = P(V_{t_x+1}=v_i|V_{t_x-1}=v_j, do(X))$;
\item return $\sum_{V_{t_y}\setminus Y} T^{t_y-(t_x+1)}\,A\,T^{t_x-1-t_0}\,P(V_{t_0})$;
\end{enumerate}

\caption{The DCN-sID algorithm for the transportability in DCNs with static hidden confounders}
\label{fig:transport-static-algo}
\medskip\hrule
\end{figure}

For brevity, we omit the algorithm extension to dynamic hidden confounders, and the completeness results, which follow the same caveats already explained in the previous sections.


\section{Experiments}

In this section, we provide some numerical examples of causal effect identifiability in DCN, using the algorithms proposed in this thesis.

In our first example, the DCN in Figure~\ref{fig:dcn_confounder_compact} represents how  the traffic between two cities evolves. There are two roads and drivers choose every day to use one or the other road. Traffic conditions on either road on a given day ($tr1$, $tr2$) affect the travel delay between the cities on that same day ($d$). Driver experience influences the road choice next day, impacting $tr1$ and $tr2$. For simplicity we assume variables $tr1$, $tr2$ and $d$ to be binary. Let us assume that from Monday to Friday the joint distribution of the variables follow transition matrix $T_1$ while on Saturday and Sunday they follow transition matrix $T_2$. These transition matrices indicate the traffic distribution change from the previous day to the current day. This system is a DCN with static hidden confounders, and has a Markov chain representation as in Figure~\ref{fig:dcn_confounder_compact}.

\[ T_1 = \left( \begin{matrix}
0.0&0.4&0.0&0.3&0.0&0.2&0.0&0.1 \\
0.0&0.4&0.0&0.3&0.0&0.2&0.0&0.1 \\
0.0&0.4&0.0&0.3&0.0&0.2&0.0&0.1 \\
0.0&0.4&0.0&0.3&0.0&0.2&0.0&0.1 \\
0.2&0.0&0.0&0.1&0.4&0.0&0.0&0.3 \\
0.2&0.0&0.0&0.1&0.4&0.0&0.0&0.3 \\
0.2&0.0&0.0&0.1&0.4&0.0&0.0&0.3 \\
0.2&0.0&0.0&0.1&0.4&0.0&0.0&0.3 \\
\end{matrix} \right)\] 

\[ T_2 = \left( \begin{matrix}
0.1&0.0&0.3&0.1&0.2&0.2&0.0&0.1 \\
0.1&0.0&0.3&0.1&0.2&0.2&0.0&0.1 \\
0.1&0.0&0.3&0.1&0.2&0.2&0.0&0.1 \\
0.1&0.0&0.3&0.1&0.2&0.2&0.0&0.1 \\
0.0&0.2&0.1&0.0&0.1&0.3&0.3&0.0 \\
0.0&0.2&0.1&0.0&0.1&0.3&0.3&0.0 \\
0.0&0.2&0.1&0.0&0.1&0.3&0.3&0.0 \\
0.0&0.2&0.1&0.0&0.1&0.3&0.3&0.0 \\
\end{matrix} \right)\] 

The average travel delay $d$ during a two-week period is shown in Figure~\ref{fig:chart_exp0}. 

\begin{figure}[H]
\begin{center}
\includegraphics[width=7cm]{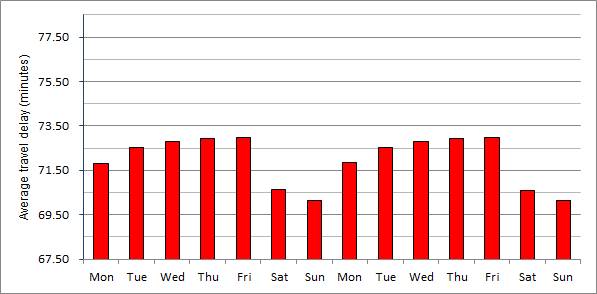}
\end{center}
\caption{Average travel delay of the DCN without intervention.}
\label{fig:chart_exp0}
\end{figure}

Now let us perform an intervention by altering the traffic on the first road $tr1$ and evaluate the subsequent evolution of the average travel delay $d$. We use the algorithm for DCNs with static hidden confounders. We trigger line 1 of the DCN-ID algorithm in Figure~\ref{fig:static-algo-complete} and build a graph consisting of four time slices $G'=(G_{t_x-2},G_{t_x-1},G_{t_x},G_{t_x+1})$ as shown in Figure~\ref{fig:fig_example}.

\begin{figure}[H]
\begin{center}
\includegraphics[width=7cm]{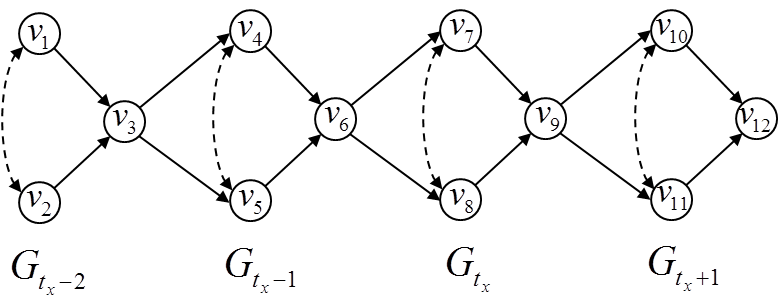}
\end{center}
\caption{Causal graph $G'$ consisting of four time slices of the DCN, from $t_x-2$ to $t_x+1$}
\label{fig:fig_example}
\end{figure}

The ancestors of any future delay at
$t=t_y$ are all the variables in the DCN up to $t_y$, so in line 2 we run the standard ID algorithm for $\alpha=P(v_{10},v_{11},v_{12}|v_4,v_5,v_6, do(v_7))$ on $G'$, which returns the expression $\alpha$:
\begin{align*}
\sum_{v_1,v_2,v_3,v_8,v_9}\frac{P(v_1,v_2,...v_{12})\sum_{v_7,v9}P(v_7,v_8,v_9|v_4,v_5,v_6)}{P(v_4,v_5,v_6)\sum_{v_9}P(v_7,v_8,v_9|v_4,v_5,v_6)}
\end{align*}

Using this expression, line 3 of the algorithm computes the elements of matrix $A$. If we perform the intervention on a Thursday the matrices $A$ for $v_7=0$ and $v_7=1$ can be evaluated from $T_1$.

\[ A_{v_7=0} = \left( \begin{matrix}
0.0&0.4&0.0&0.3&0.0&0.2&0.0&0.1 \\
0.0&0.4&0.0&0.3&0.0&0.2&0.0&0.1 \\
0.0&0.4&0.0&0.3&0.0&0.2&0.0&0.1 \\
0.0&0.4&0.0&0.3&0.0&0.2&0.0&0.1 \\
0.0&0.4&0.0&0.3&0.0&0.2&0.0&0.1 \\
0.0&0.4&0.0&0.3&0.0&0.2&0.0&0.1 \\
0.0&0.4&0.0&0.3&0.0&0.2&0.0&0.1 \\
0.0&0.4&0.0&0.3&0.0&0.2&0.0&0.1 \\
\end{matrix} \right)\]
\[ A_{v_7=1} = \left( \begin{matrix}
0.2&0.0&0.0&0.1&0.4&0.0&0.0&0.3 \\
0.2&0.0&0.0&0.1&0.4&0.0&0.0&0.3 \\
0.2&0.0&0.0&0.1&0.4&0.0&0.0&0.3 \\
0.2&0.0&0.0&0.1&0.4&0.0&0.0&0.3 \\
0.2&0.0&0.0&0.1&0.4&0.0&0.0&0.3 \\
0.2&0.0&0.0&0.1&0.4&0.0&0.0&0.3 \\
0.2&0.0&0.0&0.1&0.4&0.0&0.0&0.3 \\
0.2&0.0&0.0&0.1&0.4&0.0&0.0&0.3 \\
\end{matrix} \right)\] 

In line 4, we find that transition matrices $M_t$ are the same than for the DCN without intervention. Figure~\ref{fig:chart_exp1} shows the average travel delay without intervention, and with intervention on the traffic conditions of the first road.
 
\begin{figure}[H]
\begin{center}
\includegraphics[width=7cm]{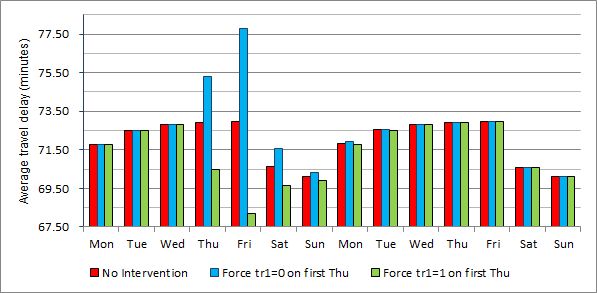}
\end{center}
\caption{Average travel delay of the DCN without intervention, and with interventions $tr1=0$ and $tr1=1$ on the first Thursday}
\label{fig:chart_exp1}
\end{figure}

In a second numerical example, we consider that the system is characterized by a unique transition matrix $T$ and the delay $d$ tends to a steady state. We measure $d$ without intervention and with intervention on $tr1$ at $t=15$. The system's transition matrix $T$ is shown below:

\[ T = \left( \begin{matrix}
0.02&0&0.03&0&0.26&0.13&0.34&0.22 \\
0.02&0&0.03&0&0.26&0.13&0.34&0.22 \\
0.02&0&0.03&0&0.26&0.13&0.34&0.22 \\
0.02&0&0.03&0&0.26&0.13&0.34&0.22 \\
0.34&0.1&0.24&0.21&0&0.02&0.09&0 \\
0.34&0.1&0.24&0.21&0&0.02&0.09&0 \\
0.34&0.1&0.24&0.21&0&0.02&0.09&0 \\
0.34&0.1&0.24&0.21&0&0.02&0.09&0 \\
\end{matrix} \right)\] 

Figure~\ref{fig:chart_exp2} shows the evolution of $d$ with no intervention and with intervention.

\begin{figure}[H]
\begin{center}
\includegraphics[width=7cm]{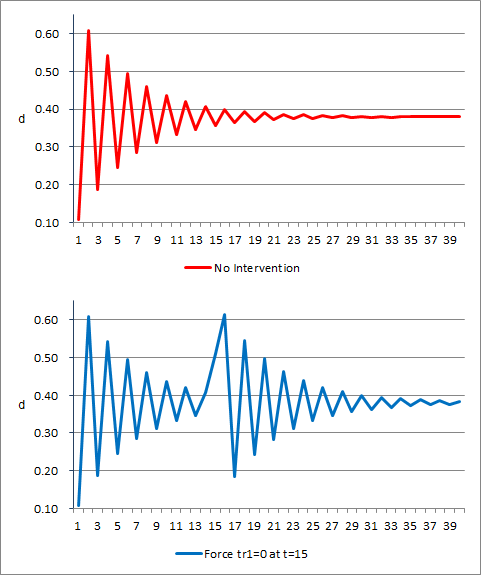}
\end{center}
\caption{Average $d$ of the DCN without intervention and with intervention on $tr1$ at $t=15$.}
\label{fig:chart_exp2}
\end{figure}

As shown in the examples, the DCN-ID algorithm calls ID only once with a graph of size $4|G|$ and evaluates the elements of matrix A with complexity $O((4k)^{(b+2)}$, where $k=3$ is the number of variables per slice and $b=1$ is the number of bits used to encode the variables. The rest is the computation of transition matrix multiplications, which can be done with complexity $O(n.b^2)$, with $n=40-15$ in example 2. To obtain the same result with the ID algorithm by brute force, we would require processing $n$ times the identifiability of a graph of size $40|G|$, with overall complexity $O((k)^{(b+2)}+(2k)^{(b+2)}+(3k)^{(b+2)}+...+(n.k)^{(b+2)})$.

\thispagestyle{plain} 
\mbox{}


\chapter{Conclusions} 

\label{ch:Chapter7} 


\section{Conclusions} 

This doctoral thesis introduces the {\em ALCAM} algorithm for the discovery of causal models with hidden confounders. It uses active learning and chooses a sequence of interventions in order to minimize the overall cost of the discovery process across a set of comprehensive cost dimensions.

Also, this doctoral thesis introduces dynamic causal networks and their analysis with do-calculus, so far studied thoroughly only in static causal graphs. We extend the ID algorithm to the identification of DCNs, and remark the difference between static vs.\ dynamic hidden confounders.  We also provide an algorithm for the transportability of causal effects from one domain to another with the same dynamic causal structure.


\section{Future Work}

The {\em ALCAM} algorithm learns causal graphs with hidden confounders in $O(|{\cal G}|)$ interventions. Using combinatorial  optimisations, we may improve this bound and learn the causal graph with the same fundamental concepts and methodologies but with $O(\log |{\cal G}|)$ interventions instead of  $O(|{\cal G}|)$.

Another line of future research could be to reduce the complexity of the {\em ALCAM} algorithm, by performing some pre-analysis of the causal graph structure and discarding some interventions upfront, instead of calculating the predicted effects for every single possible intervention.

Generalizing the method to a mix of observational and interventional data is also an interesting direction. Defining a comprehensive and integrated set of distinguishability conditions across all available data, both observational and interventional, would close an important chapter of causal research.

For future work on DCN identifiability, note that in the present thesis we have assumed all intervened variables to be in the same time slice; removing this restriction flows rather naturally from our work, and may be an interesting application for the dynamic treatment of time evolving models. Dynamic control of pandemics could be an application for this setting. 

Also, we would like to extend the introduction of causal analysis to a number of dynamic settings, including Hidden Markov Models, and study properties of DCNs in terms of Markov chains (conditions for ergodicity, for example). 

Finally, evaluating the distributions returned by the causal identification algorithms is in general very complex (exponential in the number of variables and domain size); identifying tractable sub-cases and simplifying heuristics is a general question in the field of causality.







\printbibliography[heading=bibintoc]


\end{document}